\documentclass[letterpaper, journal]{IEEEtran}

\usepackage{amsmath,amsthm,graphicx,amssymb,fullpage,wrapfig}

\usepackage{setspace}
\usepackage{cite}

\usepackage{amsmath,amsthm,amssymb,algorithmic,algorithm,hyperref}
\usepackage[margin=1in]{geometry}
\usepackage{setspace}
\usepackage{color}
\usepackage{soul}
\usepackage{epstopdf}

\usepackage{adjustbox}
\usepackage{bm}

\theoremstyle{remark}

\hypersetup{colorlinks=true}

\newcommand{\ben}{\begin{eqnarray}}

\newcommand{\een}{\end{eqnarray}}

\newcommand{\transpose}{^{\intercal}}

\newtheorem{lemma}{Lemma}

\newcommand{\ab}{{a}}
\newcommand{\bb}{{b}}

\newcommand{\cb}{{c}}
\newcommand{\eb}{{e}}
\newcommand{\gb}{{g}}
\newcommand{\xb}{{x}}
\newcommand{\rb}{{r}}

\newcommand{\zb}{{z}}

\newcommand{\Pb}{{P}}

\newcommand{\yb}{{y}}

\newcommand{\vb}{{v}}
\newcommand{\wb}{{w}}

\newcommand{\iid}{\emph{i.i.d.}}

\newtheorem{theorem}{Theorem}

\newtheorem{remark}{Remark}

\title{Detecting changes in dynamic events over networks}
\author{Shuang Li, \thanks{Shuang Li
    (Email: sli370@gatech.edu) and Yao Xie (Email: yao.xie@isye.gatech.edu)
   are with the H. Milton Stewart School of
    Industrial and Systems Engineering, Georgia Institute of
    Technology, Atlanta, GA, USA. Mehrdad Farajtabar (Email: mehrdad@gatech.edu), Apurv Verma (Email: apurvverma@gatech.edu) and Le Song (Email: lsong@cc.gatech.edu) are with College of Computing, Georgia Institute of
    Technology, Atlanta, GA, USA. This research was supported in part by CMMI-1538746 and CCF-1442635 to Y.X.; NSF/NIH BIGDATA
1R01GM108341, ONR N00014-15-1-2340, NSF IIS-1218749, NSF IIS-1639792, NSF CAREER IIS-1350983, grant from Intel and NVIDIA to L.S..} Yao Xie, Mehrdad Farajtabar, Apurv Verma, and Le Song
}

\begin{document}
%\ninept
%
\maketitle
\begin{abstract}
Large volume of networked streaming event data are becoming increasingly available in a wide variety of applications, such as social network analysis, Internet traffic monitoring and healthcare analytics. Streaming event data are discrete observation occurred in continuous time, and the precise time interval between two events carries a great deal of information about the dynamics of the underlying systems. How to promptly detect changes in these dynamic systems using these streaming event data? In this paper, we propose a novel change-point detection framework for multi-dimensional event data over networks. We cast the problem into sequential hypothesis test, and derive the likelihood ratios for point processes, which are computed efficiently via an EM-like algorithm that is parameter-free and can be computed in a distributed fashion. We derive a highly accurate theoretical characterization of the false-alarm-rate, and show that it can achieve weak signal detection by aggregating local statistics over time and networks. Finally, we demonstrate the good performance of our algorithm on numerical examples and real-world datasets from twitter and Memetracker.

%\begin{keywords}

\end{abstract}

%\end{keywords}

\section{Introduction}

Networks have become a convenient tool for people to efficiently disseminate, exchange and search for information. Recent attacks on very popular web sites such as Yahoo and eBay \cite{laptev2015generic}, leading to a disruption of services to users, have triggered an increasing interest in network anomaly detection. 
%Once such an anomaly appears, it will be harder to control as time goes by if not detect and take actions in time. 
In the positive side, surge of hot topics and breaking news can provide business opportunities. Therefore, {\it early detection} of changes, such as anomalies, epidemic outbreaks, hot topics, or new trends among streams of data from networked entities is a very important task and has been attracting significant interests \cite{laptev2015generic,levy2009detection,christakis2010social}. 

All types of the above-mentioned changes can be more concretely formulated as the changes of time interval distributions between events, combined with the alteration of interaction structures across components in networks. %Detecting the above defined changes as quickly as possible from network streaming data is the focus of our paper. 
%
%We want to online detect the status changes in networks (can be conceptual) using the information of event times and locations.  
%
However, change-point detection based on event data occurring over the network topology is nontrivial. Apart from the possible temporal dependency of the event data as well as the complex cross-dimensional dependence among components in network, event data from networked entities are usually not synchronized in time. Dynamic in nature, many of the collected data are discrete events observed irregularly in continuous time~\cite{rodriguez2011uncovering,myers2014bursty}. The precise time interval between two events is random and carries a great deal of information about the dynamics of the underlying systems. These characteristics make such event data fundamentally different from independently and identically distributed (\iid) data, and time-series data where time and space is treated as an index rather than random variables (see Figure~\ref{fig:asyn_data} for further illustrations of the distinctive nature of event data vs. \iid~and time series data). 
Clearly, \iid~assumption can not capture temporal dependency between data points, while time-series models require us to discretize the time axis and aggregate the observed events into bins (such as the approach in \cite{spikeTrain2003} for neural spike train change detection). If this approach is taken, it is not clear how one can choose the size of the bin and how to best deal with the case when there is no event within a bin. %Given such distinctive nature of the event data, the question is how to detect changes and take into account the temporal and cross-dimension dependency between these events? 

%Nowadays, data from networked entities are dynamic in nature, and many of them are discrete events observed irregularly in continuous time~\citep{rodriguez2011uncovering,myers2014bursty}. The precise time interval between two events carries a great deal of information about the dynamics of the underlying systems. These characteristics make such event data fundamentally different from independently and identically distributed (\iid) data, and time-series data where time and space is treated as an index rather than random variables (see Figure~\ref{fig:asyn_data} for further illustrations of the distinctive nature of event data vs. \iid~and time series data.) 
%Obviously, \iid~assumption can not capture temporal dependency between data points, while time-series models require us to discretize the time axis and aggregate the observed events into bins. It is not clear how one can choose the size of the bin and how to best deal with the case when there is no event within a bin. 
%Given such distinctive nature of the event data, then how to detect changes and take into account the temporal dependency between these events? 

\begin{figure*}[t]
%   \vspace{-3mm}
  \centering
  \includegraphics[width=1.00\linewidth]{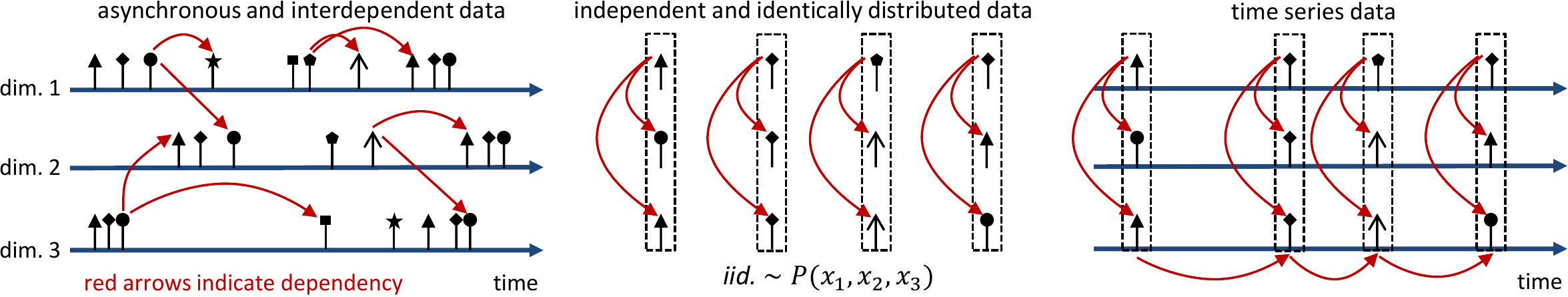}
    \vspace{-4mm}
  \caption{\small Asynchronously and interdependently generated high dimensional event data are fundamentally different from \iid~and time-series data. First, observations for each dimension can be collected at different time points; Second, there can be temporal dependence as well as cross-dimensional dependence. In contrast, the dimensions of \iid~and time-series data are sampled at the same time point, and in the figure, different marks indicate potentially different values or features of an observation.}
    \label{fig:asyn_data}
    \vspace{-4mm}
\end{figure*} 

Besides the distinctive temporal and spatial aspect, there are three additional challenges using event data over network: ({\it i}) how to detect weak changes; ({\it ii}) how to update the statistics efficiently online; and ({\it iii}) how to provide theoretical characterization of the false-alarm-rate for the statistics. 
For the first challenge, many existing approaches usually use random or ad-hoc aggregations which may not pool data efficiently or lose statistical power to detect weak signals. Occurrence of change-points (e.g., epidemic outbreaks, hot topics, etc.) over networks usually evince a certain clustering behavior over dimensions and tend to synchronize in time. %This is because some changes are triggered by major events in actual scenarios and thus tend to roughly synchronize in time. %Changes would further propagate due to the information transmissions across networks. In this case, for single dimension, such change might be initially weak and hidden in networks; however, 
Smart aggregation over dimensions and time horizon would manifest the strength of signals and detect the change quicker~\cite{milling2015distinguishing}. 
For the second challenge, many existing change-point detection methods based on likelihood ratio statistics do not take into account computational complexity nor can be computed in a distributed fashion and, hence, are not scalable to large networks. Temporal events can arrive at social platforms in very high volume and velocity. For instance, every day, on average, around 500 million tweets are tweeted on Twitter \cite{twitter}. %\footnote{\url{http://www.internetlivestats.com/twitter-statistics/}}. 
There is a great need for developing efficient algorithms for updating the detection statistics online. 
For the third challenge, it is usually very hard to control false-alarms for change-point detection statistics over a large network. When applied to real network data, traditional detection approaches usually have a high false alarms~\cite{laptev2015generic}. This would lead to a huge waste of resources since every time a change-point is declared, subsequent diagnoses are needed. Lacking accurate theoretical characterization of false-alarms, existing approaches usually have to incur expensive Monte Carlo simulations to determine the false-alarms and are prohibitive for large networks.

{\bf Our contributions.} In this paper, we present a novel online change-point detection framework tailored to multi-dimensional intertwined event data streams over networks (or conceptual networks) tackling the above challenges. We formulate the problem by leveraging the mathematical framework of sequential hypothesis testing and point processes modeling, where before the change the event stream follows one point process, and after the change the event stream becomes a different point process. Our goal is to detect such changes {\it as quickly as possible} after the occurrences. We derive generalized likelihood ratio statistics, and present an efficient EM-like algorithm to compute the statistic online with streaming data. The EM-like algorithm is parameter-free and can be implemented in a distributed fashion and, hence, it is suitable for large networks.

Specifically, our contributions include the following: 

%{\color{red}Note: discuss the pros and cons of using Poisson process over Hawkes process, and when should choose one over the other.}

% \begin{itemize}
  ({\it i}) We present a new sequential hypothesis test and likelihood ratio approach for detecting changes for the event data streams {\it over networks}.  We will either use the Poisson process as the null distribution to detect the appearance of temporal independence, or use the Hawkes process as the null distribution to detect the possible alteration of the dependency structure.  %  We mainly use Poisson process and Hawkes process to model the event stream data. 
 For (inhomogeneous) Poisson process, time intervals between events are assumed to be independent and exponentially distributed. For Hawkes process, the occurrence intensity of events depends on the events that have occurred, which implies that the time intervals between events would be correlated. Therefore, Hawkes process can be thought of as a special autoregressive process in time, and multivariate Hawkes process also provides a flexible model to capture cross-dimension dependency in addition to temporal dependency. Our model explicitly captures the information diffusion (and dependencies) both over networks and time, and allows us to aggregate information for weak signal detection. Our proposed detection framework is quite general and can be easily adapted to other point processes. 
  
In contrast, existing work on change-point detection for point processes has also been focused on a single stream rather than the multidimensional case with networks. These work including detecting change in the intensity of a Poisson process \cite{shen2012change,zhang2014scanning,herberts2004optimal} and the coefficient of continuous diffusion process \cite{StimbergRuttor2011}; detecting change using the self-exciting Hawkes processes include 
trend detection in social networks \cite{pinto2015trend}; %inferring leadership in e-mail networks \cite{fox2014modeling}; 
detecting for Poisson processes using a score statistic \cite{solo2012test}. %In single stream neural spike chain change-detection \cite{spikeTrain2003}, data is preprocessed by be counting the number of events over discrete time intervals. %Such an ad-hoc approach can not handle the continuous data stream directly, which may distort the distribution of the data and can not detect quickly. 

%  ({\it ii}) Our method allows for a general transitions such as Poisson to Hawkes processes, modeling the appearance of temporal or/and cross-dimension dependency; or Hawkes to Hawkes process, describing the increases or decreases in temporal or/and cross-dimension dependency. Which model to choose relies on the nature of the real dataset. 
  
  ({\it ii}) We present an efficient expectation-maximization (EM) like algorithm for updating the likelihood-ratio detection statistic online. The algorithm can be implemented in a {\it distributed} fashion due to is structure: only neighboring nodes need to exchange information for the E-step and M-step. 

  ({\it iii}) We also present accurate theoretical approximation to the false-alarm-rate (formally the average-run-length or ARL) of the detection algorithm, via the recently developed change-of-measure approach to handle highly correlated statistics. Our theoretical approximation can be used to determine the threshold in the algorithm accurately. 
% \end{itemize}
%we show that the analytical approximation for the false-alarm-rate can be evaluated directly for simple network topologies such as the star networks or the chain networks; for general topology network, our false-alarm-rate expression provides a viable interface for simulation approaches.

  ({\it iv}) Finally, we demonstrate the performance gain of our algorithm over two baseline algorithms (which ignore the temporal correlation and correlation between nodes), using synthetic experiments and real-world data. These two baseline algorithms representing the current approaches for processing event stream data. We also show that our  algorithm is very sensitive to true changes, and the theoretical false-alarm-rates are very accurate compared to the experimental results. 
  
%\note{Talk about likelihood for point processes given a sequence of events. Talk about about how is it going to be convex optimization problem. But typical algorithm needs parameter.} 

{\bf Related work.} 
Recently, there has been a surge of interests in using multidimensional point processes for modeling dynamic event data over networks. However, most of these works focus on modeling and inference of the point processes over networks. 
%networks of point processes and further estimating model parameters. 
Related works include modeling and learning bursty dynamics~\cite{myers2014bursty}; 
 shaping social activity by incentivization~\cite{FarDuGomValZhaSon14};
 learning information diffusion networks~\cite{rodriguez2011uncovering};
inferring causality~\cite{xu2016learning};
 %and source identification from cascades of events~\citep{FarGomDuZamZhaSon15}; 
 learning mutually exciting processes for viral diffusion \cite{yang2013mixture}; learning triggering kernels for multi-dimensional Hawkes processes
\cite{zhou2013learning}; 
in networks where each dimension is a Poisson process \cite{rajaram2005poisson}; 
learning latent network structure for general counting processes \cite{linderman2014discovering}; tracking parameters of dynamic point process networks \cite{hall2014tracking}; 
%estimating point-process models of social network
%interactions \citep{zipkin2015point}; 
and estimating point process models for the co-evolution of network structure an information diffusion~\cite{FarWanGomLiZhaSon15}, just to name a few.
% Another related work is \cite{yilmaz2013channel}, which developed active sampling methods for continuously monitoring wireless channels. 
These existing works provide a wealth of tools through which we can, to some extent, keep track of the network dynamics if the model parameters can be sequentially updated. However, only given the values of the up-to-date model parameters, especially in high dimensional networks, it is still not clear how to perform change detection based on these models in a principled fashion. 
%
% straightforward to aggregate all the information and make a decision whether the network dynamic has changed or not. 
%%Furthermore, it is even harder to have some statistical guarantees in the process of decision making. 
%Therefore, there is a great need for one holistic metric that can aggregate all the information in models and meanwhile have statistical interpretation in the process of decision-making. 

Classical statistical sequential analysis (see, e.g., \cite{detectAbruptChange93,Tartakovsky2014}),  where one monitors {\it i.i.d.} univariate and low-dimensional multivariate observations observations from a single data stream is a well-developed area. Outstanding contributions include Shewhart's control chart \cite{Shewhart31}, the minimax approach Page's CUSUM procedure \cite{Page1954,Page1955}, the Bayesian approach 
Shiryaev-Roberts procedure \cite{Shiryaev1963,Roberts66}, and window-limited procedures \cite{lai1995sequential}. 
However, there is limited research in monitoring large-scale data streams over a network, or even event streams over networks. Detection of change-points in point processes has so far mostly focused on the simple Poisson process models without considering temporal dependency, and most of the detection statistics are computed in a discrete-time fashion, that is, one needs to aggregate the observed events into bins and then apply the traditional detection approaches to time-series of count data. Examples include \cite{ihler2006adaptive,mei2011early,levy2009detection} .

%Some exceptions are 

The notations are standard. The remaining sections are organized as follows. Section \ref{sec:model} presents the point process model and derives the likelihood functions. Section \ref{sec:change-det} presents our sequential likelihood ratio procedure. Section \ref{sec:alg} presents the EM-like algorithm. Section \ref{sec:theory} presents our theoretical approximation to false-alarm-rate. Section \ref{sec:numerical_eg} contains the numerical examples. Section \ref{sec:numerical_eg} presents our results for real-data. Finally, Section \ref{sec:summary} summarizes the paper. All proofs are delegated to the Appendix.

\section{Model and Formulation}\label{sec:model}

Consider a sequence of events over a network with $d$ nodes, represented as a double sequence
\begin{equation}
(t_1, u_1), (t_2, u_2), \dots, (t_n, u_n), \ldots
\label{event_stream}
\end{equation}
 where $t_i \in \mathbb{R}^{+}$ denotes the real-valued time when the $i$th event happens, and $i \in \mathbb{Z}^{+}$ and $u_{i} \in \{1,2, \dots, d\}$ indicating the node index where the event happens. We use temporal point processes \cite{daley2007introduction} to model the discrete event streams, since they provide convenient tool in directly modeling the time intervals between events, and avoid the need of picking a time window to aggregate events and allow temporal events to be modeled in a fine grained fashion. 

\subsection{Temporal point processes}

A temporal point process is a random process whose realization consists of a list of discrete events localized in time, $\{t_i\}$, with $t_i \in \mathbb{R}^{+}$ and $i \in \mathbb{Z}^{+}$.   We start by considering one-dimensional point processes. Let the list of times of events up to but not including time $t$ be the history
\[\mathcal{H}_{t} = \{t_1, \ldots, t_n: t_n< t\}.\] Let  $N_t$ represent the total number of events till time $t$.  %Then \[N_t=\sum_{ t_i \in \mathcal{H}_{t}} \mathbb{I}_{[t_i, \infty) } (t) =\int_0^t dN_s,\] where the step function is defined as
%\begin{align*}
%\mathbb{I}_{[t_i, \infty) } (t) = \begin{cases}
%1, & t \geq t_i, \\
%0, & 0 \leq t < t_i. 
%\end{cases}
%\end{align*}
Then the counting measure can be defined as
\begin{equation}
dN_t= \sum_{t_i \in \mathcal{H}_{t}} \delta(t-t_i) dt,
\label{cnt_measure}
\end{equation} 
where $\delta(t)$ is the Dirac function.

To define the likelihood ratio for point processes, we first introduce the notion of {\it conditional intensity function} \cite{temporalPointProcess_2011}. The conditional intensity function is a convenient and intuitive way of specifying how the present depends on the past in a temporal point process. Let $F^*(t) $ be the conditional probability that the next event $t_{n+1}$ happens before $t$ given the history of previous events
\[F^*(t) = \mathbb{P}\{t_{n+t}<t|\mathcal{H}_t\},\] and let $f^*(t)$ be the corresponding conditional density function. 
%Let $f^*(t) = f(t|\mathcal{H}_t)$ be the conditional density function of the time of the next event $t_{n+1}$ given the history of previous events. Let $F^*(t)$ be the corresponding cumulative distribution function. 
%Define $F(t)$ \textcolor{blue}{(definition)} as the conditional probability that a new event has happened before time $t$ conditioned on all past events $\mathcal{H}_{t}$. Define $f(t)$ as the derivative of $F(t)$, and it corresponds to a conditional probability density function. 
 The conditional intensity function (or the hazard function) \cite{temporalPointProcess_2011} is defined by 
\begin{equation}
\lambda_t = \frac{f^*(t)}{1-F^*(t)}, \label{def_lambda}
\end{equation}
and it can be interpreted as the probability that an event occurs in an infinitesimal interval  
\begin{equation}
\lambda_t dt = \mathbb{P}\{\mbox{event in } [t, t+ dt) | \mathcal{H}_t\}.
\end{equation}
This general model includes Poisson process and Hawkes process as special cases.

 ({\it i}) For (inhomogeneous) Poisson processes, each event is stochastically independent to all the other events in the process, and the time intervals between consecutive events are independent with each other and are exponentially distributed. As a result, the conditional intensity function is independent of the past, which is simply deterministic
$
\lambda_t = \mu_t.
$

  \vspace{0.05in}
 ({\it ii}) For one dimensional Hawkes processes, the intensity function is history dependent and models a mutual excitation between events
\begin{equation} 
\lambda_t= \mu_t+\alpha  \int_{0}^t  \varphi (t-\tau) dN_{\tau}, \label{intensity1d}
\end{equation}
where $\mu_t$ is the base intensity (deterministic), $\alpha \in (0,1)$ (due to the requirement of stationary condition) is the influence parameter, and $\varphi(t)$ is a normalized kernel function $\int \varphi(t)dt=1$. Together, they characterize how the history influences the current intensity. Fixing the kernel function, a higher value of $\alpha$ means a stronger temporal dependency between events.  A commonly used kernel function is the exponential kernel $\varphi (t)= \beta e^{-\beta t}$, which we will use through the paper.  

\vspace{0.05in}
 ({\it iii}) The multi-dimensional Hawkes process is defined similarly, with each dimension being a one-dimensional counting process. It can be used to model the sequence of events over network such as (\ref{event_stream}). We may convert a multi-dimensional process into a double sequence, using the first coordinate to represent time of the event, and the second coordinate to represent the index of the corresponding node. 
 
Define a multivariate counting process $(N^1_t, N^2_t, \dots, N^d_t)$, $t \geqslant 0$, with each component $N^i_t$ recording the number of events of the $i$-th component (node) of the network during $[0,t]$. 
% or equivalently the time stamps of the observed events. 
The intensity function is 
\begin{equation}
\lambda_t^i=\mu^i_t+ \sum_{j=1}^d \int_{0}^{t}\alpha_{ij} \varphi (t-\tau) dN^j_{\tau},
\label{m-hawkes}
\end{equation}
% \begin{align} \label{intensityDD}
% \lambda_t^i=\mu^i_t+ \sum_{j=1}^d \int_{-\infty}^{t-}\alpha_{ij} \varphi (t-\tau) dN^j_{\tau},
% \end{align}
%where $\phi(t-\tau)$ is the kernel that models delay of the influence, and a commonly used kernel is the exponential kernel $\phi(t) = \beta e^{- \beta t}$. 
where $\alpha_{ij}, \, j,i \in \{1,\dots,d\}$ represents the strength of influence of the $j$-th  node on the $i$-th node by affecting its intensity process $\lambda^i$. If $\alpha_{ij}=0$, then it means that $N^j$ is not influencing $N^i$. Written in matrix form, the intensity can be expressed as
\begin{equation} 
\bm{ \lambda}_t=\bm{\mu}_t +\bm{A} \int_{0}^{t} \varphi(t-\tau)  d\bm{N}_{\tau}, \end{equation} 
where 
\[\bm{\mu}_t=[\mu^1_t, \mu^2_t, \dots, \mu^d_t]\transpose, d\bm{N}_\tau=[dN_\tau^1, dN_\tau^2, \dots, dN_\tau^d]\transpose,\] and $\bm{A}=[\alpha_{ij}]_{1\leqslant i, j \leqslant d}$ is the {\it influence matrix}, which is our main quantity-of-interest when detect a change. The diagonal entries 
characterize the self-excitation and the off-diagonal entries capture the mutual-excitation among nodes in the network. The influence matrix can be asymmetric since influence can be bidirectional. 

\begin{figure}[t!]
\centering
\begin{tabular}{cc}
\includegraphics[width=0.45 \linewidth]{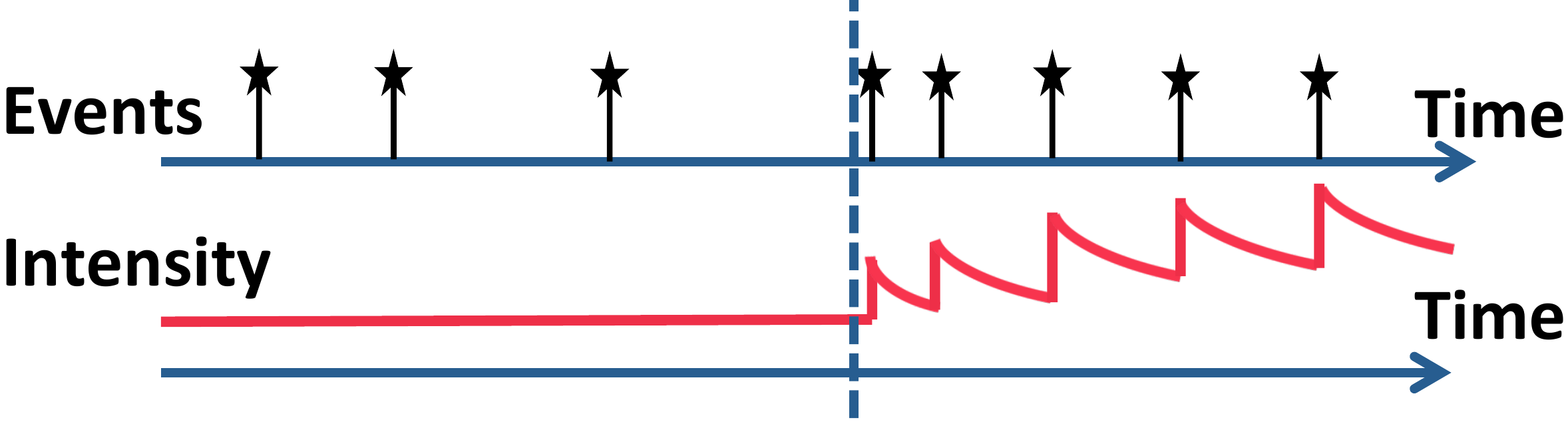} 
& \includegraphics[width=0.45\linewidth]{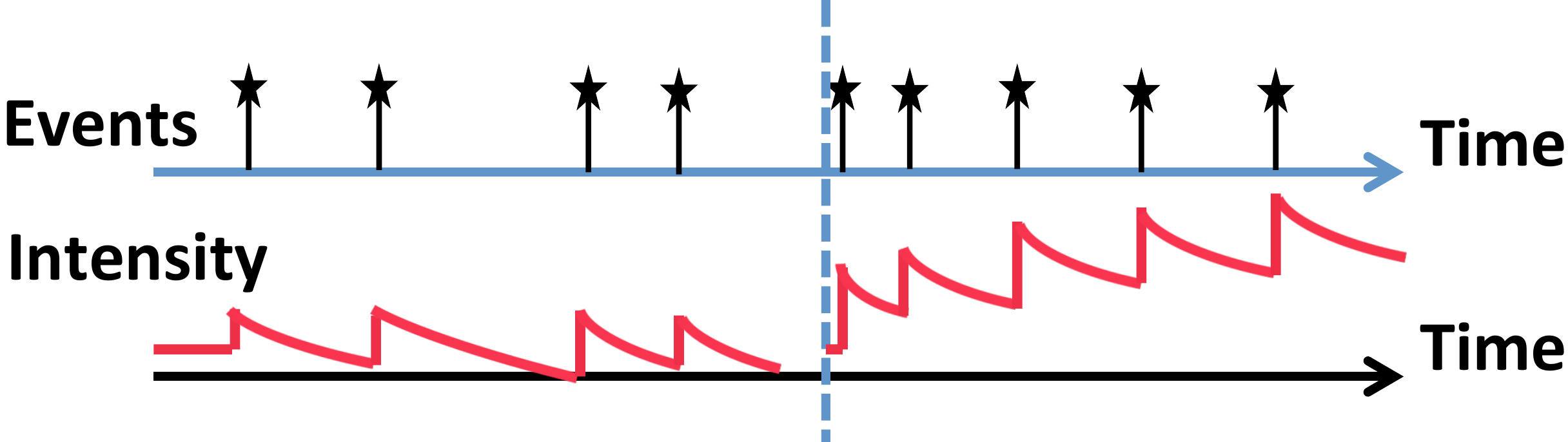} \\
(a) {\small Poisson to Hawkes} & 
(b) {\small Hawkes to Hawkes}
\end{tabular}
%\vspace{-3mm}
\caption{Illustration of scenarios for one-dimensional examples: (a) Poisson to Hawkes; (b) Hawkes to Hawkes.}
\label{fig:illustration}
\vspace{-2mm}
\end{figure}

\subsection{Likelihood function}\label{section:likelihood} 
In the following, we will explicitly denote the dependence of the likelihood function on the parameters in each setting. The following three cases are useful for our subsequent derivations. Let $f(t)$ denote the probability density function. For the one-dimensional setting, given a sequence of $n$ events (event times) $\{t_1, t_2, \ldots, t_n\}$ before time $t$. Using the conditional probability formula, we obtain
\begin{equation}
\begin{split}
& \mathcal{L} = f(t_1, \ldots, t_n)  = (1-F^*(t))\prod_{i=1}^n f(t_i|t_1, \ldots, t_{i-1}) \\
& = (1-F^*(t))\prod_{i=1}^n f^*(t_i) = \left(  \prod_{i=1}^n \lambda_{t_i} \right) \mbox{exp} \left \{ - \int_{0}^t  \lambda_s ds  \right\}.
\end{split}
\label{gen_likelihood}
\end{equation}
The last equation is from the following argument. From the definition of the conditional density function, we have
\[
\lambda_t =  \frac{d}{dt} F^*(t) / (1-F^*(t)) = - \frac{d}{dt} \mbox{log} (1-F^*(t)).\]
Hence,  $\int_{t_n}^t \lambda_s ds = -\mbox{log} (1-F^*(t)) $, where $F^*(t_n)=0$, since event $n+1$ cannot happen at time $t_n$. Therefore,
\[
F^*(t)=1-\mbox{exp} \left\{  - \int_{t_n}^t \lambda_s ds \right\},  ~
f^*(t) = \lambda_t \mbox{exp} \left\{  - \int_{t_n}^t \lambda_s ds \right\}.
\]
The likelihood function for multi-dimensional Hawkes process can be derived similarly, by redefining $f^*(t)$ and $F^*(t)$ according to the intensity functions of the multi-dimensional processes. %Given a double sequence consists of event times and the  node index that the event occurs before time $t$, $(t_1, u_1), \ldots, (t_n, u_n)$, the likelihood function is given by
%\begin{equation}
%\mathcal{L} = \left(  \prod_{i=1}^n \lambda_{t_i}^{u_i} \right) \mbox{exp} \left \{ - \sum_{i=1}^n  \lambda_s^i ds  \right\}.
%\label{mHawkes-likelihood}
%\end{equation}

Based on the above principle, we can derive the following likelihood functions.

\subsubsection{Homogeneous Poisson process}
For homogeneous Poisson process, 
$
\lambda_t = \mu.
$
Given constant intensity, the log-likelihood function for a list of events $\{t_1, t_2, \ldots, t_n\}$ in the time interval $[0, t] $ can be written as
\begin{equation}
\mbox{log} \,\mathcal{L}(\mu)= n \mbox{log} \,\mu - \mu t.\label{likepoi}
\end{equation}

\subsubsection{One dimensional Hawkes process}
For one-dimensional Hawkes process with constant baseline intensity $\mu_t =\mu$ and exponential kernel, we may obtain its log-likelihood function based on the above calculation. By substituting the conditional intensity function  (\ref{intensity1d}) into  (\ref{gen_likelihood}), the log-likelihood function for events in the time interval $[0,t] $ is given by
\begin{equation}
\begin{split}
& \mbox{log} \,\mathcal{L} (\alpha, \beta, \mu) = \sum_{i=1}^n \mbox{log}\, \left( \mu + \alpha \sum_{t_j <t_i}\beta e^{-\beta(t_i-t_j)}    \right) \\
& ~~ - \mu t - \sum_{t_i < t}\alpha \left[ 1- e^{-\beta (t-t_i)} \right]. \label{likeHaw1}
 \end{split}
 \end{equation}
To obtain the above expression, we have used the following two simple results for exponential kernels, due to the property of counting measure defined in (\ref{cnt_measure}):
 \begin{equation}
% \begin{split}
 \lambda_t = \mu +\alpha \int_{-\infty}^{t}  \varphi (t-\tau) dN_{\tau} 
%& = \mu + \alpha \int_{-\infty}^t \beta e^{-\beta (t-\tau)} \sum_{t_i \in \mathcal{H}_{t}} \delta(\tau-t_i) d\tau \\
 = \mu + \alpha \sum_{t_i < t} \beta e^{-\beta(t-t_i)},
%\end{split}
 \end{equation}
and %it can be shown that
\begin{equation}
%\begin{split}
\int_0^t \lambda_s ds %&=  \int_0^t  \left(  \mu + \alpha \sum_{t_j <s}\beta e^{-\beta(s-t_j)}  \right) ds \\
= \mu t + \sum_{t_i < t}\alpha \left[ 1- e^{-\beta (t-t_i)} \right].
%\end{split}
\end{equation}

\subsubsection{Multi-dimensional Hawkes process}
For multi-dimensional point process, we consider the event stream such as (\ref{event_stream}).
%the sequence of events become $\{(t_1, u_1), \ldots, (t_n, u_n)\}$, where we know the time and also the node index that the event occurs. 
Assume base intensities are constants with $\mu_t^i \triangleq  \mu_i$. Using similar calculations as above, we obtain the log-likelihood function for events in the time interval $[0, t]$ as
\begin{equation}
\begin{split}
&\mbox{log}\,\mathcal{L}\,(\bm{A}, \beta, \bm{\mu} ) 
= \sum_{i=1}^n  \mbox{log} \left[  \mu_{u_i}+  \sum_{t_j< t_i}  \alpha_{u_i, u_j} \beta e^{-\beta (t_i-t_j)} \right] \\
&-\sum_{j=1}^{d} \mu_j t
 - \sum_{j=1}^{d} \sum_{t_i < t} \alpha_{u_i,j} \left[ 1- e^{-\beta(t-t_i)}\right]. 
 \end{split}
 \label{likeHawD}
\end{equation}
%\begin{remark}[Distributed computation]
%Note that the likelihood function only depends on events that occurs at its neighboring nodes. 
%\end{remark}

\section{Sequential change-point detection}\label{sec:change-det}

We are interested in detecting two types of changes sequentially from event streams, which capture two general scenarios in real applications (Fig. \ref{fig:illustration} illustrates these two scenarios for the one dimensional setting):
 ({\it i}) The sequence before change is a Poisson process and after the change is a Hawkes process. This can be useful for applications where we are interested in detecting an emergence of self- or mutual-excitation between nodes.  
 ({\it ii}) The sequence before change is a Hawkes process and after the change the magnitude of influence matrix increases. This can be a more realistic scenario, since often nodes in a network will influence each initially. This can be useful for applications where a triggering event changes the behavior or structure of the network. For instance, detecting emergence of a community in network \cite{influence_community}.

In the following, we cast the change-point detection problems as sequential hypothesis test \cite{Siegmund1985}, and derive generalized likelihood ratio (GLR) statistic for each case. Suppose there may exist an unknown change-point $\kappa$ such that after that time, the distribution of the point process changes. % Since the likelihood ratio test is known to be the most powerful test when no unknown parameters in reference and alternative models (justified by the Neyman-Pearson lemma~\cite{neyman1992problem}). 

\subsection{Change from Poisson to Hawkes}

First, we are interested in detecting the events over network changing from $d$-dimensional independent Poisson processes to an intertwined multivariate Hawkes process. This models the effect that the change affects the spatial dependency structure over the network. Below, we first consider one-dimensional setting, and then generalize them to multi-dimensional case. 

\subsubsection{One-dimensional case}
The data consists of a sequence of events occurring at time $\{t_1, t_2, \ldots, t_n\}$. Under the hypothesis of no change (i.e. $\textsf{H}_0$), the event time is a one-dimensional Poisson process with intensity $\lambda$. Under the alternative hypothesis (i.e. $ \textsf{H}_1$), there exists a change-point $\kappa$. The sequence is a Poisson process with intensity $\lambda$ initially, and changes to a one-dimensional Hawkes process with parameter $\alpha$ after the change. Formally, the hypothesis test can be stated as
\begin{equation} \label{test_poi_haw_1d}
\left\{
\begin{array}{ll}
 \textsf{H}_0: &  \lambda_s =\mu, \quad 0 < s <t;   \\
 \textsf{H}_1: &  \lambda _s=\mu, \quad 0 < s <\kappa,\\
&   \lambda_s^*=\mu+\alpha  \int_{\kappa}^{s} \varphi(s-\tau)  dN_{\tau}, \quad \kappa < s<t. 
\end{array}
\right.
\end{equation}
Assume intensity $\mu$ can be estimated from reference data and $\beta$ is given as a priori. We treat the post-change influence parameter $\alpha$ as unknown parameter since it represents an anomaly. 

Using the likelihood functions derived in Section \ref{section:likelihood}, equations (\ref{likepoi}) and (\ref{likeHaw1}), for a hypothetical change-point location $\tau$, the log-likelihood ratio as a function of $\alpha$, $\beta$ and $\mu$, is given by
\begin{equation}
\begin{split}
&\ell_{t, \tau, \alpha} = \mbox{log} \,\mathcal{L} (\alpha, \beta, \mu)  - \mbox{log} \,\mathcal{L}(\mu)  \\
&=\sum_{t_i \in (\tau, t)}  \mbox{log} \left[ \mu +  \alpha \sum_{t_j \in (\tau, t_i)} \beta e^{-\beta (t_i-t_j)} \right]  \\
& \quad- \mu(t-\tau)  - \alpha \sum_{\tau_i \in (\tau, t)} \left[1 -e^{-\beta (t-t_i)}  \right].
\end{split}
\label{loglikelihood_poi}
\end{equation}
Note that log-likelihood ratio only depends on the events in the interval $(\tau, t)$ and $\alpha$. We maximize the statistic with respect to the unknown parameters $\alpha$ and $\tau$ to obtain the log GLR statistic. Finally, the sequential change-point detection procedure is a stopping rule (related to the non-Bayesian minimax type of detection rule, see \cite{changepoint_new_book2014}):
\begin{equation}
T_{\rm one-dim} = \inf\{t: \max_{\tau<t} \max_\alpha ~\ell_{t, \tau, \alpha} > x\},
\label{one-d-T}
\end{equation} 
where $x$ is a pre-scribed threshold, whose choice will be discussed later. Even though there does not exist a closed-form expression for the estimator of $\alpha$, we can estimate $\alpha$ via an EM-like algorithm, which will be discussed in Section \ref{em}.

\begin{remark}[Offline detection] We can adapt the procedure for offline change-point detection by considering the fixed-sample hypothesis test. For instance, for the one-dimensional setting, given a sequence of $n$ events with $t_{\rm max} \triangleq t_n$, we may detect the existence of change when the detection statistic, $\max_{\tau<t_{\rm max}} \max_\alpha ~\ell_{t_{\rm max}, \tau, \alpha}$, exceeds a threshold. The change-point location can be estimated as $\tau^*$ that obtains the maximum. However, the algorithm consideration for online and offline detection are very different, as discussed in Section \ref{sec:alg}. 
\end{remark}

%Then we can write out the log-likelihood function based on the data stored in between the time interval $(0,t)$. Under the null, the log-likelihood function $\mathcal{L}_{H_0}$ is
%\begin{eqnarray}
% \int_{0}^{t} \mbox{log} \lambda_s dN_s-\int_{0}^{t} \lambda_s ds.
%\end{eqnarray}
%Under the alternative, the log-likelihood function $\mathcal{L}_{H_1}$ is
%\begin{align}
%\int_{0}^{\tau} \mbox{log} \lambda_sdN_s+\int_{\tau}^{t} \mbox{log} \lambda^*_s dN_s  -\int_{0}^{\tau} \lambda_sds-\int_{\tau}^{t} \lambda^*_s ds.
%\end{align}
%Hence, the log-likelihood ratio is:
%\begin{align} \label{likelihoodratio}
%\hspace{-3mm}
%\ell =   \mathcal{L}_{H_1} - \mathcal{L}_{H_0} 
%= \int_{\tau}^t \mbox{log} \left( \frac{\lambda^*_s}{\lambda_s} \right) dN_s - \int_{\tau}^t \left( \lambda^*_s-\lambda_s \right) ds .
%\end{align}
%Note that (\ref{likelihoodratio}) is the general log-likelihood ratio for any two point processes with intensity $\lambda_s$ and $\lambda_s^*$ respectively.
%In particular, if we plug in the intensity defined by (\ref{test_poi_haw_1d}), choose $\varphi (t)= \beta e^{-\beta t}$, where $\beta>0$, and define $L=t-\tau$, we can compute the log-likelihood ratio as
%\begin{align} \label{loglikelihood_poi}
%\ell =& \sum_{i=1}^n  \mbox{log} \left[ \mu +  \alpha \sum_{\tau< t_j < t_i} \beta e^{-\beta (t_i-t_j)} \right]
%-n\mbox{log} \mu  \nonumber \\
%& \quad - \alpha \sum_{i=1}^n \left[  1 -e^{-\beta (t-t_i)}  \right] ,
%\end{align}
%where $n$ is the number of events within time interval $(\tau, t)$.

%\subsection{Multidimensional change points}

\subsubsection{Multi-dimensional case} 
For the multi-dimensional case, the event stream data can be represented as a double sequence defined in (\ref{event_stream}). We may construct a similar hypothesis test as above. Under the hypothesis of no change, the event times is multi-dimensional Poisson process with a vector intensity function $\bm{\lambda}_s=\bm{\mu}$. Under the alternative hypothesis, there exists a change-point $\kappa$. The sequence is a multi-dimensional Poisson process initially, and changes to a multi-dimensional Hawkes process with influence matrix $\bm{A}$ afterwards. We omit the formal statement of the hypothesis test as it is similar to (\ref{test_poi_haw_1d}). 

%Similarly, we obtain likelihood ratio for multi-dimensional case. Consider the event stream such as (\ref{event_stream}). Suppose the total number of events occurred in between $(\tau, t)$ is $n$. %, and we record all the events as $\{ (t_1, u_1), (t_2, u_2), \ldots\},$ where $u_{i} \in \{1,2, \dots, d\}$ indicates the location of the event. 
%Consider the hypothesis test:
%\begin{equation}
%\left\{
%\begin{array}{ll}
%\textsf{H}_0: &  \bm{\lambda}_s=\bm{\mu} , \quad 0 < s <t; \nonumber \\
%\textsf{H}_1: &  \bm{\lambda}_s=\bm{\mu} , \quad 0 < s <\kappa,\label{test_ddpoi}\\
%&  \bm{ \lambda}_s^*=\bm{\mu} +\bm{A} \int_{\kappa}^{s} \varphi(s-\tau)  d\bm{N}_{\tau}, \quad \kappa < s<t. \nonumber
%\end{array}\right.
%\end{equation}
%Such a hypothesis test can be used to detect the emergence of a community that level of self- and mutual-excitation among network increase after the change point.

Again, using the likelihood functions derived in \ref{section:likelihood}, we obtain the likelihood ratio. 
%Define $\bm{e}=[1,1, \dots, 1]_{d \times 1}\transpose$. 
The log-likelihood ratio for data up to time $t$, given a hypothetical change-point location $\tau$ and parameter $\bm{A}$, is given by
\begin{equation} \label{loglikelihood_ddpoi}
\begin{split}
&\ell_{t, \tau, \bm{A}}=\mbox{log} \,\mathcal{L} (\bm{A}, \beta, \mu)  - \mbox{log} \,\mathcal{L}(\mu) \\
&= \sum_{ t_i \in (\tau, t) } \mbox{log} \left[ 1+  \frac{ 1}{   \mu_{u_i} }  \sum_{ t_j \in (\tau, t_i)  }  \alpha_{u_i, u_j} \beta e^{-\beta (t_i-t_j)} 
\right] \\
& \quad - \sum_{j=1}^{d} \sum_{ t_i \in (\tau, t)} \alpha_{j,u_i}\left[ 1- e^{-\beta(t-t_i)}\right].  
\end{split}    
\end{equation} 
%\begin{equation} \label{loglikelihood_ddpoi}
%\begin{split}
%&\ell_{t, \tau, \bm{A}} =\mbox{log} \,\mathcal{L} (\bm{A}, \beta, \mu)  - \mbox{log} \,\mathcal{L}(\mu) \\
%&= \bm{e}\transpose \left ( \int_{\tau}^t \mbox{log} \left(   \bm{ \lambda}^*_s ./\bm{ \lambda}_s \right) d\bm{N}_s - \int_{\tau}^t \left( \bm{\lambda}^*_s-\bm{\lambda}_s \right) d \bm{s} \right),
%\end{split}
%\end{equation}
%where $./$ denotes elementwise division. 
%\vspace{-0.1in}
The sequential change-point detection procedure is a stopping rule:
\begin{equation}
T_{\rm multi-dim} = \inf\{t: \max_{\tau<t} \max_{\bm{A}} ~ \ell_{t, \tau, \bm{A}} > x\},
\label{mHawkT}
\end{equation} 
where $x$ is a pre-determined threshold. The multi-dimensional maximization can be computed efficiently via an EM algorithm described in Section \ref{em} . 

\begin{remark}[Topology of network]
The topology of the network has been embedded in the sparsity pattern of the influence matrix $A$, which are given as a priori. The dependency between different nodes in the network and the temporal dependence over events can be captured in updating (or tracking) the influence matrix $A$ with events stream. This can be achieved as an EM-like algorithm, which is resulted from solving a sequential optimization problem with warm start (i.e., we always initialize the parameters using the optimal solutions of the last step).
\end{remark}

\begin{figure}[h!]
  \centering
    %    \small
    \includegraphics[
      width=0.8 \linewidth]%{methodgraph}
      {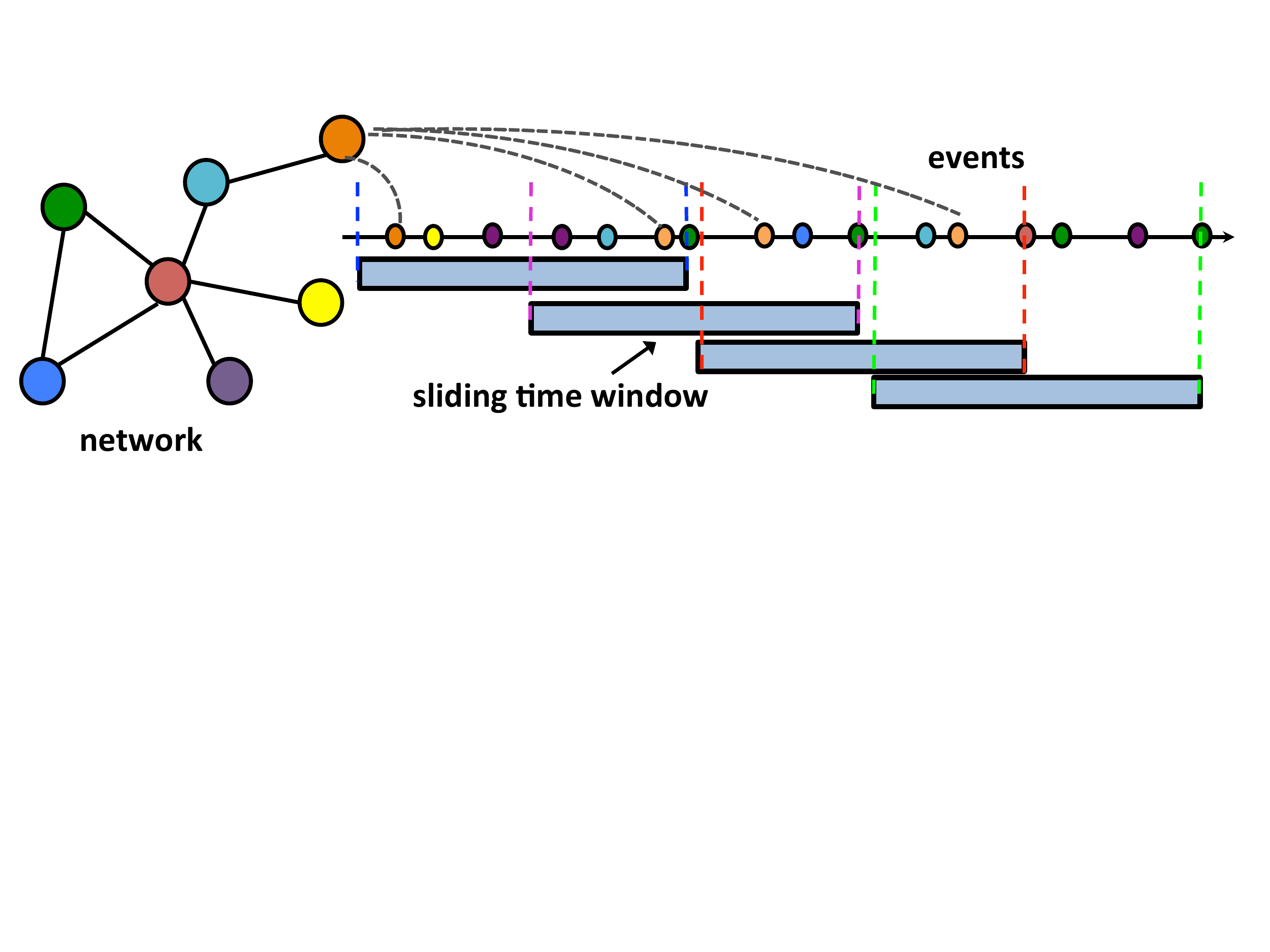}
    \vspace{-2mm}
    \caption{Illustration of the sliding window approach for online detection.}
    \label{model}    
 \vspace{-5mm}
\end{figure}  

\subsection{Changes from Hawkes to Hawkes} 

Next, consider the scenario where the process prior to change is a Hawkes process, and the change happens in the influence parameter $\alpha$ or the influence matrix $\bm{A}$. 

\subsubsection{One-dimensional case}
%Start by considering the one-dimensional case. 
Under the hypothesis of no change, the event stream is a one-dimensional Hawkes process with parameter $\alpha$. Under the alternative hypothesis, there exists a change-point $\kappa$. The sequence is a Hawkes process with intensity $\alpha$, and after the change, the intensity changes to $\alpha^*$. 
%\begin{equation} \label{test_1dhaw}
%\left\{
%\begin{array}{ll}
%\textsf{H}_0: &  \lambda_s=\mu + \alpha  \int_{0}^{s} \varphi(s-\tau)  dN_\tau, \quad 0 < s <t; \nonumber \\
%\textsf{H}_1: &  \lambda_s=\mu + \alpha  \int_{0}^{s} \varphi(s-\tau)  dN_\tau, \quad 0 < s <\kappa, \nonumber  \\
%&     \lambda_s^*=\mu + \alpha^*  \int_{\kappa}^{s} \varphi(s-\tau)  dN_\tau, \quad  \kappa <  
%s < t.
%\end{array}\right.
%\end{equation}
Assume the parameter $\alpha$ prior to change is known. 

Using the likelihood functions derived in \ref{section:likelihood}, we obtain the log-likelihood ratio %as (\ref{likelihoodratio}). %Plug in the expressions for $\lambda_s$ and $\lambda_s^*$, adopt exponential kernel $\varphi(t)=\beta e^{-\beta t}$, and define $ L = t-\tau $, we finally obtain
\begin{equation} \label{ratio_1dhaw}
\begin{split}
& \ell_{t, \tau, \alpha^*} =  \mbox{log} \,\mathcal{L} (\alpha^*, \beta, \mu)  - \mbox{log} \,\mathcal{L}(\mu) \\
&= \sum_{t_i \in (\tau, t)} \mbox{log} \left[    \frac{\mu + \alpha^* \sum_{t_j \in (\tau, t_i)} \beta e^{-\beta (t_i-t_j)}      }{\mu +\alpha \sum_{t_j \in (\tau, t_i)} \beta e^{-\beta (t_i-t_j)}   }  \right]   \\
& \quad -\left( \alpha^* - \alpha \right)  \sum_{t_i \in (\tau, t)}  \left[   1 -  e^{-\beta (t-t_i)} \right],
\end{split}
\end{equation}
and the change-point detection is through a procedure in the form of (\ref{one-d-T}) by maximizing with respect to $\tau$ and $\alpha$. 
%Still, $\alpha^*$ can be estimated from data between time $(\tau, t)$. And we detect a change-point whenever 
%$\sup_{t, \tau, \alpha^*}\ell_{t, \tau, \alpha^*} >x $.

\vspace{0.05in}
\subsubsection{Multi-dimensional case}
For the multi-dimensional setting, we assume the change will alter the influence parameters of the multi-dimensional Hawkes process over network. This captures the effect that, after the change, the influence between nodes becomes different. Assume that under the hypothesis of no change, the event stream is a multi-dimensional Hawkes process with parameter $\bm{A}$. Alternatively, there exists a change-point $\kappa$. The sequence is a multi-dimensional Hawkes process with influence matrix $\bm{A}$ before the change, and after the change, the influence matrix becomes $\bm{A}^*$. Assume the influence matrix $\bm{A}$ prior to change  is known. 
%\begin{equation} \label{test_ddhaw}
%\left\{
%\begin{array}{ll}
% \textsf{H}_0:  & \bm{\lambda}_s=\bm{\mu} +\bm{A} \int_{0}^{s} \varphi(s-v)  d\bm{N}_\tau, ~0 < s <t;  \nonumber \\
% \textsf{H}_1:  & \bm{\lambda}_s=\bm{\mu} +\bm{A} \int_{0}^{s} \varphi(s-v)  d\bm{N}_\tau, ~0 < s <\kappa,  \nonumber\\ 
% & \bm{\lambda}_s^*=\bm{\mu} +\bm{A}^* \int_{\kappa}^{s} \varphi(s-\tau)  d\bm{N}_\tau,~ \kappa < s <t.
%\end{array}  
%\right.
%\end{equation}

Using the likelihood functions derived in \ref{section:likelihood}, the log-likelihood ratio at time $t$ for a hypothetical change-point location $\tau$ and post-change parameter value $\bm{A}^*$ is given by
\begin{equation} \label{ratio_ddhaw}
\begin{split}
& \ell_{t, \tau, \bm{A}^*} = \mbox{log} \,\mathcal{L} (\bm{A}^*, \beta, \mu)  - \mbox{log} \,\mathcal{L}(\mu) \\
& = \sum_{t_i \in (\tau, t)}  \mbox{log} \left[ \frac{  \mu_{u_i}+  \sum_{t_j \in (\tau,  t_i)}  \alpha_{u_i, u_j}^* \beta e^{-\beta (t_i-t_j)} }
{   \mu_{u_i}+  \sum_{t_j \in (\tau, t_i)}  \alpha_{u_i, u_j} \beta e^{-\beta (t_i-t_j)} }\right]  \\
& \quad - \sum_{j=1}^{d} \sum_{t_i \in (\tau, t)} \left(\alpha^*_{j,u_i} - \alpha_{j,u_i}     \right) \left[1- e^{-\beta(t-t_i)}\right],     
\end{split} 
\end{equation}
and the change-point detection is through a procedure in the form of (\ref{mHawkT}) by maximizing with respect to $\tau$ and $\bm{A}^*$.

\begin{algorithm}[h!]
  \caption{Online Detection Algorithm}
  \label{alg:onlinealg}
  \begin{algorithmic}[1]
    \REQUIRE   Data $\{(t_i, u_i)\}$.
                Scanning window length $L$;
             Update frequency $\gamma$ (per events);
             Initialization for parameters $\alpha$ (one-dimension) or $\bm{A}$ (multi-dimension);
             Pre-defined threshold: $x$;
             Estimation accuracy: $\epsilon$.
                    %\(\mu = \mu\), \(\Sigma_x = \Sigma_x\)
    \REPEAT 
     \IF{$\mod(i, \gamma) =0$}
   %   \STATE {Estimate Parameters}
     \STATE Initialize $\alpha^{(0)}=\hat{\alpha}$ or $\bm{A}^{(0)}=\hat{\bm{A}}$ \COMMENT{warm start}
        \REPEAT 
           \STATE Perform \{E-step\} and \{M-step\} from Section \ref{em} %: (\ref{estep_1d})  or  (\ref{estep_dd}) 
         %  \STATE \{M-step\}:  (\ref{mstep_1d}) or (\ref{mstep_dd}) 
      \UNTIL{$\|  \alpha^{(k+1)} -\alpha^{(k)}  \|  < \epsilon $ or $\|  \bm{A}^{(k+1)} -\bm{A}^{(k)}  \|  < \epsilon $} 
      \STATE Let $\hat{\alpha}=\alpha^{(k+1)}$ and $\hat{\bm{A}}=\bm{A}^{(k+1)}$.
    \STATE  Use $\hat{\alpha}$ or $\hat{\bm{A}}$ to compute log likelihood using (\ref{loglikelihood_poi}), (\ref{loglikelihood_ddpoi}), (\ref{ratio_1dhaw}) or (\ref{ratio_ddhaw}).    
       \ENDIF
          \UNTIL{$ \ell_{t,\tau, \hat{\alpha}} >x\) or \( \ell_{t,\tau, \hat{\bm{ A}}}>x\)} 
         and announce a change. % \COMMENT{detect a change point}
  \end{algorithmic}
  \label{alg}
\end{algorithm}

\section{Algorithm for computing likelihood online}\label{sec:alg}

%Sequential detection of changes require computing the likelihood function online. 
In the online setting, we obtain new data continuously. Hence, in order to perform online detection, we need to update the likelihood efficiently to incorporate the new data. To reduce computational cost, update of the likelihood function can be computed recursively and the update algorithm should have low cost. To reduce memory requirement, the algorithm should only store the minimum amount of data necessary for detection rather than the complete history. These requirements make online detection drastically different from offline detection. Since in the offline setting, we can afford more computational complexity. 

\subsection{Sliding window procedure}

The basic idea of online detection procedure is illustrated in Fig.~\ref{model}. We adopt a {\it sliding window} approach to reduce computational complexity as well the memory requirement. When evaluating the likelihood function, instead of maximizing over possible change-point location $\tau < t$, we pick a window-size $L$ and set $\tau$ to be a fixed-value
\[\tau = t - L.\] This is equivalent to constantly testing whether a change-point occurs $L$ samples before. By fixing the window-size, we reduce the computational complexity, since we eliminate the maximization over the change-point location. This also reduces the memory requirement as  we only need to store events that fall into the sliding window. The drawback is that, by doing this, some statistical detection power is lost, since we do not use the most likely change-point location, and it may increase detection delay. When implementing the algorithm, we choose $L$ to achieve a good balance in these two aspect. We have to choose $L$ large enough so that there is enough events stored for us to make a consistent inference. In practice, a proper length of window relies on the nature of the data. If the data are noisy, usually a longer time window is needed to have a better estimation of the parameter and reduce the false alarm. 
%\note{Does our theory say anything about the effects of the sliding window  size? \textcolor{blue}{No we do not have that theory - we did not analyze detection delay.}

%More specifically, we will choose a slide window of size $L = t-\tau$. We move the window forward every new event or every $\gamma$ events. Every time we move the sliding window we will estimate the influence parameters $\hat{\alpha}$ based on the events within the sliding window and construct the log-likelihood ratio as discussed beforehand. 
%\note{How do you declare change point, and where is the change point relative to the slide window starting point? Please mimic the offline case to explain.} 

%This is a key difference of the online detection versus offline detection, 
%\note{In order to perform online detection, we need to do this and that, and how this address the special requirement of online case.} \note{Why is the online case different from the offline case? explain...} 

%Although there does not exist a closed-form estimator for influence parameter $\alpha$ or influence matrix $\bm{A}$, we develop an efficient EM algorithm to update the likelihood, exploiting the structure of the likelihood function. One possibility is to use online gradient descent to estimate the parameters. However, we need to turn step-sizes of the gradient descent. An advantage of the EM algorithm is that it does not have tuning parameter.  

\subsection{Parameter Free EM-like Algorithm} \label{em}

We consider one-dimensional point process to illustrate the derivation of the EM-like algorithm. It can be shown that the likelihood function (\ref{loglikelihood_poi}) is a concave function with respect to the parameter $\alpha$. One can use gradient descent to optimize this objective, where the algorithm will typically involves some additional tuning parameters such as the learning rate. Although there does not exist a closed-form estimator for influence parameter $\alpha$ or influence matrix $\bm{A}$, we develop an efficient EM algorithm to update the likelihood, exploiting the structure of the likelihood function \cite{simma2012modeling}. The overall algorithm is summarized in Algorithm~\ref{alg:onlinealg}. 
 %However, we want to obtain an algorithm free of tuning parameters, and design an EM-like algorithm to solve the concave optimization problem. 

First, we obtain a concave lower bound of the likelihood function using Jensen's inequality. Consider all events fall into a sliding window $t_i \in (\tau, t)$ at time $t$. Introduce auxiliary variables $p_{ij}$ for all pair of events $(i, j)$ within the window and such that $t_j < t_i$. The variables are subject to the constraint
\begin{equation} 
\forall i, \sum_{t_j < t_i} p_{ij} = 1, \quad p_{ij} \geqslant 0. \label{constr}
\end{equation}
These $p_{ij}$ can be interpreted as the probability that $j$-th event influence the $i$-th event in the sequence. It can be shown that the likelihood function defined in (\ref{likeHaw1}) can be lower-bounded
\begin{align}
& \ell_{t, \tau, \alpha}  \geqslant  \sum_{t_i \in (\tau, t)} \left( p_{ii} \mbox{log}  (\mu) + \sum_{t_j \in (\tau, t_i)}  p_{ij} \mbox{log} \left[  \alpha \beta e^{-\beta (t_i-t_j)}   \right]  \right. \nonumber \\
&  \left.  -\sum_{t_j\in (\tau, t)} p_{ij} \mbox{log} p_{ij} \right) - \mu (t-\tau) - \alpha  \sum_{t_i \in (\tau, t)} \left[1 -e^{-\beta (t-t_i)}  \right]  ,\nonumber 
\end{align}
Note that the lower-bound is valid for every choice of $\{p_{ij}\}$ which satisfies (\ref{constr}).

To make the lower bound tight and ensure improvement in each iteration, we will maximize it with respect to $p_{ij}$ and obtain (\ref{estep_1d}) (assuming we have $\alpha^{(k)}$ from previous iteration or initialization). Once we have the tight lower bound, we will take gradient of this lower-bound with respect to $\alpha$. When updating from the $k$-th iteration to the $(k+1)$-th iteration, we obtain (\ref{mstep_1d})
\begin{align}
\hspace{-4mm}
p_{ij}^{(k)}&=\frac{\alpha^{(k)} \beta e^{-\beta (t_j-t_i)}}{\mu + \alpha^{(k)} \beta \sum_{t_m \in (\tau, t_j)} e^{-\beta (t_j-t_m)}} \quad \{\mbox{E-step}\}  \label{estep_1d}  
\\
 %1 \leqslant j < i\leqslant n. 
\hspace{-4mm}
\alpha^{(k+1)}&=\frac{\sum_{i<j} p_{ij}^{(k)}}{\sum_{t_i \in (\tau, t)} [1-e^{-\beta (t-t_i)}]} 
\quad \{\mbox{M-step}\}
\label{mstep_1d}
\end{align}
where the superscript denotes the number of iterations. 
The algorithm iterates these two steps until the algorithm converges and obtains the estimated $\alpha$. In practice, we find that we only need 3 or 4 iterations to converge if using warm start. 

Similarly, online estimate for the influence matrix for multi-dimensional case  can be estimated by iterating the following two steps:
\begin{align*}
\hspace{-5mm}
p_{ij}^{(k)} &= \frac{\alpha^{(k)}_{u_i, u_j} \beta e^{-\beta (t_i-t_j)}}{\mu_{u_i} + \beta \sum_{t_m \in (\tau, t_i)} \alpha^{(k)}_{u_i, u_m}  e^{-\beta (t_i-t_m)}} \quad \{\mbox{E-step}\}  %\label{estep_dd} 
\\
\hspace{-5mm}
\alpha^{(k+1)}_{u,v} &= \frac{ \sum_{i:\,u_i=u} \sum_{j <i:\,u_j=v} p_{ij}^{(k)}   }{ \sum_{j:\,t_j \in (\tau, t), u_j=v}  \left[1-e^{-\beta(t-t_j)} \right]}.  \quad \{\mbox{M-step}\} %\label{mstep_dd}
%\begin{align} \label{estep_dd}
%& p_{ij}^k=\frac{\alpha^k_{u_i, u_j} \beta e^{-\beta (t_i-t_j)}}{\mu_{u_i}+\sum_{j=1}^{i-1}  \alpha^k_{u_i, u_j}  \beta e^{-\beta (t_i-t_j)}},\quad 1 \leqslant j < i\leqslant n 
\end{align*}

The overall detection procedure is summarized in Fig.~\ref{model} and Algorithm~\ref{alg}. 

\begin{remark}[Computational complexity]
The key computation is to compute pairwise inter-event times for pairs of event $t_i - t_j$, $i<j$. It is related to the window size (since we have adopted a sliding window approach), the size of the network, and the number of EM steps. However, note that in the EM algorithm, we only need to compute the inter-event times for nodes that are connected by an edge, since the summation is weighted by $\alpha_{ij}$ and the term only counts if $\alpha_{ij}$ is non-zero. Hence, the updates only involve neighboring nodes and the complexity is proportional to the number of edges in the network. Since most social networks are sparse, the will significantly lower the complexity. We may reduce the number of EM iterations for each update, by  leveraging a warm-start for initializing the parameter values: since typically for two adjacent sliding window, the corresponding optimal parameter values should be very close to the previous one. 
\end{remark}

%\note{Explain in detail using mathematical notations, and use concrete notation to explain the process. Itemize if possible.} 

\begin{remark}[Distributed implementation]
Our EM-like algorithm in the network setting can be implemented in a distributed fashion. This has embedded in the form of the algorithm already.  Hence, the algorithm can be used for process large networks. In the E-step, when updating the $p_{ij}$, we need to evaluate a sum in the denominator, and this is the only place where different nodes need to exchange information, i.e., the event times happened at that node. Since we only need to sum over all events such that the corresponding $\alpha_{u_i, u_j}$ is non-zero, this means that each node only needs to consider the events happened at the neighboring nodes. Similarly, in the M-step, only neighboring nodes need to exchange their values of $p_{ij}$ and event times to update the influence parameter values. 
\end{remark}

\vspace{-0.1in}
\section{Theoretical threshold}\label{sec:theory}

A key step in implementing the detection algorithm is to set the threshold. The choice of threshold involves a trade-off between two standard performance metrics for sequential change-point detection: the false-alarm rate and how fast we can detect the change. Formally, these two performance metrics are: (i) the expected stopping time when there is no change-points, or named average run length (ARL); and (ii) the expected detection delay when there exists a change-point. 

Typically, a higher threshold $x$ results in a larger ARL (hence smaller false-alarm rate) but larger detection delay. A usual practice is to set the false-alarm-rate (or ARL) to a pre-determined value, and find the corresponding threshold $x$. The pre-determined ARL depends on how frequent we can tolerate false detection (once a month or once a year). 
Usually, the threshold is estimated via direct Monte Carlo by relating threshold to ARL assuming the data follow the null distribution. However, Monte Carlo is not only computationally expensive, in some practical problems, repeated experiments would be prohibitive. Therefore it is important to find a cheaper way to accurately estimate the threshold. 

We develop an analytical function which relates the threshold to ARL. That is, given a prescribed ARL, we can solve for the corresponding threshold $x$ analytically. We first characterize the property of the likelihood ratio statistic in the following lemma, which states that the mean and variance of the log-likelihood ratios both scale roughly linearly with the post-change time duration. This property of the likelihood ratio statistics is key to developing our main result. 
\begin{lemma}[Mean and variance of log-likelihood ratios]
When the number of post-change samples $(t-\tau)$ is large, the mean and variance of log-likelihood ratio for the single-dimensional and the multi-dimensional cases, denoted as $\ell_{t, \tau, \cdot}$, for our cases converges to simple linear form. Under the null hypothesis, 
$
\mathbb{E}[\ell_{t, \tau, \cdot}] \approx (t-\tau) I_0
$
and 
$
\mathbb{E}[\ell_{t, \tau, \cdot}] \approx (t-\tau) \sigma_0^2
$. Under the alternative hypothesis, 
$
\mathbb{E}[\ell_{t, \tau, \cdot}] \approx (t-\tau) I
$
and 
$
\mathbb{E}[\ell_{t, \tau, \cdot}] \approx (t-\tau) \sigma^2
$.
Above, $I$, $I_0$, $\sigma^2$, and $\sigma_0^2$ are defined in Table \ref{table1} for various settings we considered. \end{lemma}

Our main theoretical result is the following general theorem that can be applied for all hypothesis test we consider. Denote the probability and the expectation under the hypothesis of no change by $\mathbb{P}^\infty$ and $\mathbb{E}^\infty$, respectively. 

\begin{theorem} [ARL under the null distribution] \label{thm:ARL}
When $x\rightarrow \infty$ and $x/\sqrt{L} \rightarrow c'$ for some constant $c'$, the average run length (ARL) of the stopping time $T$ defined in (\ref{one-d-T}) for one-dimensional case, is given by
\begin{equation}
\mathbb{E}^{\infty} [T_{\rm one-dim}] = e^x \left[    \int_{\alpha}   \nu \left( \frac{2 \xi}{\eta^2}\right) \frac{\phi\left( \frac{LI-x}{\sqrt{L\sigma^2}} \right) }{\sqrt{L\sigma^2}} d\alpha    \right]^{-1}  \cdot (1 +o(1) ).
\label{ARL_expr}
\end{equation}
For multi-dimensional case, the same expression holds for $\mathbb{E}^{\infty} [T_{\rm multi-dim}]$ except that $\int_{\alpha}$ is replaced by $\int_{\bm{A}}$, which means taking integral with respect to all \it{nonzero entries} of the matrix
$
\int_{\bm{A} } = \int \cdots \int \int_{\{ \alpha_{ij}, \alpha_{ij} \neq 0\}}.$
Above, the special function % $\nu(\mu)$ is defined as
\[ \nu(\mu) \approx \frac{(2/\mu) \left( \Phi(\mu/2) -0.5      \right)}{(\mu/2) \Phi(\mu/2)+\phi(\mu/2)}.\]
%& \Phi(x) \mbox{ is the CDF of the standard normal;} \nonumber \\
%& \phi(x) \mbox{ is the PDF of the standard normal; } \nonumber  \\
The specific expressions for $I$, $I_0$, $\sigma^2$, and $\sigma_0^2$ for various settings are summarized in Table \ref{table1}, and %other related quantities are defined as 
\begin{equation}
\xi = -(I_0 - I), \quad \eta^2 = \sigma_0^2 + \sigma^2. \label{def_xi_eta}
\end{equation}
Above,  $\Phi(x)$ and $\phi(x)$ are the cumulative distribution function (CDF) and the probability density function (PDF) of the standard normal, respectively.

%If we further define $I_0 = \mathbb{E}[\ell_{t, \tau, \alpha}] / {L}$ and $\sigma_0^2 = \mbox{Var}[\ell_{t, \tau, \alpha}]/{L}$, all the calculations will be finally boiled down to the evaluations of $I$, $I_0$, $\sigma^2$, and $\sigma_0^2$, whose detailed expressions are summarized in Table \ref{table1} for various settings. 
\end{theorem}
\begin{remark}[Evaluating integral] The multi-dimensional integral can be evaluated using Monte Carlo method \cite{MCCasella04}. We use this approach for our numerical examples as well.
\end{remark}

\begin{remark}[Interpretation]
The parameters $I_0$, $I$, $\sigma_0^2$ and $\sigma^2$ have the following interpretation
\begin{align}
 & I_0 = \mathbb{E}[\ell_{t, \tau, \alpha}]/L, \quad
 \sigma_0^2 = \mbox{Var}[\ell_{t, \tau, \alpha}]/L, \nonumber \\
& I =  {\mathbb{E}_{t, \tau, \alpha} [\ell_{t, \tau, \alpha}]}/{L}, \quad
 \sigma^2 =  {\mbox{Var}_{t, \tau, \alpha} [\ell_{t, \tau, \alpha}]}/{L}, 
\end{align}
which are the mean and the variance of the log-likelihood ratio under the null and the alternative distributions, {\it per unit time}, respectively. 
Moreover, $I$ can be interpreted roughly as the Kullback-Leibler information per time for each of the hypothesis test we consider.
\end{remark}

The proof of the Theorem \ref{thm:ARL} combines the recently developed change-of-measure techniques for sequential analysis, with properties the likelihood ratios for the point processes, mean field approximation for point processes, and Delta method \cite{StatsInference01}.

\begin{table*}[h!]
\caption{Expressions for $I$, $I_0$, $\sigma^2$ and $\sigma_0^2$ under different settings.}
% \vspace{3mm}
\setlength{\tabcolsep}{0pt}
\newcommand{\tabincell}[2]{\begin{tabular}{@{}#1@{}}#2\end{tabular}}
\begin{center}
% \begin{adjustbox}{width=1.05\textwidth}
  \begin{tabular}{c|c|c|c|c}
    \hline \hline
 Setting &$I$&$I_0$&$\sigma^2$&$\sigma_0^2$  \\ \hline \hline
  \tabincell{l} {Poi. $\to$ Haw. \\ (one dim.) 
   %as shown in (\ref{test_poi_haw_1d}) and (\ref{likelihoodratio})
   }
 & \small{$ \frac{\mu}{1-\alpha} \mbox{log} \left( \frac{1}{1-\alpha}\right) -\frac{\alpha}{1-\alpha}\mu $ }&\small{$   \mu \mbox{log} \left( \frac{1}{1-\alpha}\right) -\frac{\alpha}{1-\alpha}\mu $}
 &\small{  \tabincell{l}{ $\left[ \mbox{log} \left( \frac{1}{1-\alpha} \right) \right]^2 \cdot $\\$
~\left[ \frac{\mu}{1-\alpha}+ \frac{\alpha (2-\alpha)\mu}{(1-\alpha)^3} \right]$} } 
&\small{ $ \mu \left[ \mbox{log}\left(  \frac{1}{1-\alpha} \right) \right]^2$ } \\ \hline 
\tabincell{l} {Poi. $\to$ Haw. \\ (high dim.)}
 % as shown in (\ref{test_ddpoi}) and (\ref{loglikelihood_ddpoi})}
  &
\small{  \tabincell{l} {   $\bar{ \bm{ \lambda}} ^{*\intercal} \left( \mbox{log} (\bar{\bm{\lambda}} ^* )-   \mbox{log} (\bm{\mu}) \right) $ \\ $-\bm{e}^{\intercal}(\bar{\bm{\lambda}} ^*-\bm{\mu}) $    }  }
   & 
 \small{   \tabincell{l} { $ \bm{ \mu}^{\intercal} \left( \mbox{log} (\bar{\bm{\lambda}}^*)-   \mbox{log} (\bm{\mu}) \right)$ \\ $ -\bm{e}\transpose(\bar{ \bm{\lambda}}^*-\bm{\mu}) $  } } 
    &
  \small{ $\bm{e}^{\intercal} \left( \bm{H} \circ \bm{C}\right) \bm{e}$ }
  &
 \small{ $ \bm{ \mu}^{\intercal} \left( \mbox{log} (\bar{ \bm{\lambda}} ^*)-   \mbox{log} (\bm{\mu} )\right)^{(2)}$ }\\ \hline
  \tabincell{l} {Haw. $\to$ Haw. \\ (one dim.) %\\
  %as shown in (\ref{test_1dhaw}) and (\ref{ratio_1dhaw})
  }   
   &
\small{    \tabincell{l}{  $ \frac{\mu}{1-\alpha^*} \mbox{log} \left( \frac{1-\alpha}{1-\alpha^*} \right) $\\
 $-\frac{\mu}{1-\alpha^*} +\frac{\mu}{1-\alpha}    $}   }
   &
   \small{ \tabincell{l}{ $ \frac{\mu}{1-\alpha} \mbox{log} \left(  \frac{1-\alpha}{1-\alpha^*} \right) $ \\
  $ - \frac{\mu}{1-\alpha^*} +\frac{\mu}{1-\alpha}$ } }
  &
\small{ \tabincell{l} { $  \left[ \mbox{log} \left( \frac{1-\alpha}{1-\alpha^*} \right) \right]^2  \cdot $\\
$~~\left[ \frac{\mu}{1-\alpha^*}+ \frac{\alpha^* (2-\alpha^*)\mu}{(1-\alpha^*)^3}    \right]  $ \\
$+\left(  1-\frac{1-\alpha}{1-\alpha^*}\right)^2 \cdot $\\
$\left[ \frac{\mu}{1-\alpha}+ \frac{\alpha(2-\alpha)\mu}{(1-\alpha)^3}    \right]  $ }}
  &
\small{  \tabincell{l} { $   \left[  1-\frac{1-\alpha^*}{1-\alpha}\right]^2  \cdot $ \\
$~~ \left[ \frac{\mu}{1-\alpha^*}+ \frac{\alpha^* (2-\alpha^*)\mu}{(1-\alpha^*)^3}    \right]  $\\
$ +\left[ \mbox{log} \left( \frac{1-\alpha}{1-\alpha^*} \right) \right]^2 \cdot$ \\
$\left[ \frac{\mu}{1-\alpha}+ \frac{\alpha(2-\alpha)\mu}{(1-\alpha)^3}    \right] $ }}\\  \hline
 \tabincell{l} {Haw. $\to$ Haw. \\ (multi dim.) %\\
  % as shown in (\ref{test_ddhaw}) and (\ref{ratio_ddhaw})
}  
& 
\small{ \tabincell{l} {  $\bar{\bm{\lambda}}^{*\intercal} \left[ \mbox{log} \bar{\bm{\lambda}}^*-\mbox{log} \bar{\bm{\lambda}}  \right] $\\
$ -\bm{e}^{\intercal} [\bar{\bm{\lambda}}^* -\bar{\bm{\lambda}}]$} }
& 
\small{ \tabincell{l} { $\bar{ \bm{\lambda}}^{\intercal} \left[ \mbox{log}\bar{ \bm{\lambda}}^*-\mbox{log} \bar{\bm{\lambda}}  \right]  $  \\
$ -\bm{e}^{\intercal} [\bar{\bm{\lambda}}^* -\bar{\bm{\lambda}}]$ } }
& 
\small{ $\bm{e}^{\intercal} \left(  \bm{G} \circ \bm{C}^* +\bm{F} \circ \bm{C}\right) \bm{e} ~$}
& \small{ $\bm{e}^{\intercal} \left(  \bm{R} \circ \bm{C}^* + \bm{G} \circ \bm{C}\right) \bm{e}$} \\
    \hline\hline
  \end{tabular}
  \label{table1}
\end{center}
In the table above, $\bm{M}^{(2)}=\bm{M}\circ \bm{M}$ denote the Hadamard product, and related quantities are defined as
 \begin{align*}  
& \bar{ \bm{\lambda}}^* = (\bm{I}-\bm{A}^*)^{-1} \bm{\mu}, \quad
\bar{\bm{\lambda}} = (\bm{I}-\bm{A})^{-1} \bm{\mu}, \quad  
 \bm{H} = \left[  \mbox{log} \left(  (\bm{I}-\bm{A})^{-1} \bm{\mu}  \right) - \mbox{log} \left(  \bm{\mu} \right)  \right] \cdot \left[  \mbox{log} \left( (\bm{I}-\bm{A})^{-1} \bm{\mu}   \right) - \mbox{log} \left(  \bm{\mu} \right)  \right]^{\intercal} ,\nonumber \\
&\bm{C} =   (\bm{I}-\bm{A})^{-1}  \bm{A} \left( 2 \bm{I} +(\bm{I}-\bm{A})^{-1}
\bm{A}  \right)  \mbox{diag} \left(  (\bm{I}-\bm{A})^{-1} \bm{\mu}\right) + \mbox{diag} \left(  (\bm{I}-\bm{A})^{-1} \bm{\mu}\right), \nonumber  \\
& \bm{C}^* =  (\bm{I}-\bm{A}^*)^{-1}  \bm{A}^* \left( 2 \bm{I} +(\bm{I}-\bm{A}^*)^{-1}
\bm{A}^*  \right)  \cdot \mbox{diag} \left(  (\bm{I}-\bm{A}^*)^{-1} \bm{\mu}\right) + \mbox{diag} \left(  (\bm{I}-\bm{A}^*)^{-1} \bm{\mu}\right), \nonumber \\
& \bm{G}_{ij} =  [\mbox{log}  \left(\bar{\lambda}^*_i/\bar{\lambda}_i \right)] \cdot [\mbox{log} \left(\bar{\lambda}^*_j/\bar{\lambda}_j \right)], \quad
 \bm{F}_{ij} = \left(1-\bar{\lambda}^*_i/\bar{\lambda}_i \right) \left(1-\bar{\lambda}^*_j/\bar{\lambda}_j \right), \quad \bm{R}_{ij} = \left(\bar{\lambda}_i/\bar{\lambda}^*_i -1 \right)  \left(\bar{\lambda}_j/\bar{\lambda}^*_j -1 \right),  \quad 1 \leqslant i \leqslant j \leqslant d. \nonumber  
\end{align*} 
\vspace{-0.3in}
\end{table*}

% We will evaluate the accuracy of Theorem \ref{thm:ARL} in section \ref{eva_ARL}. 

\begin{figure}[]
\begin{center}
\includegraphics[width=0.7 \linewidth]{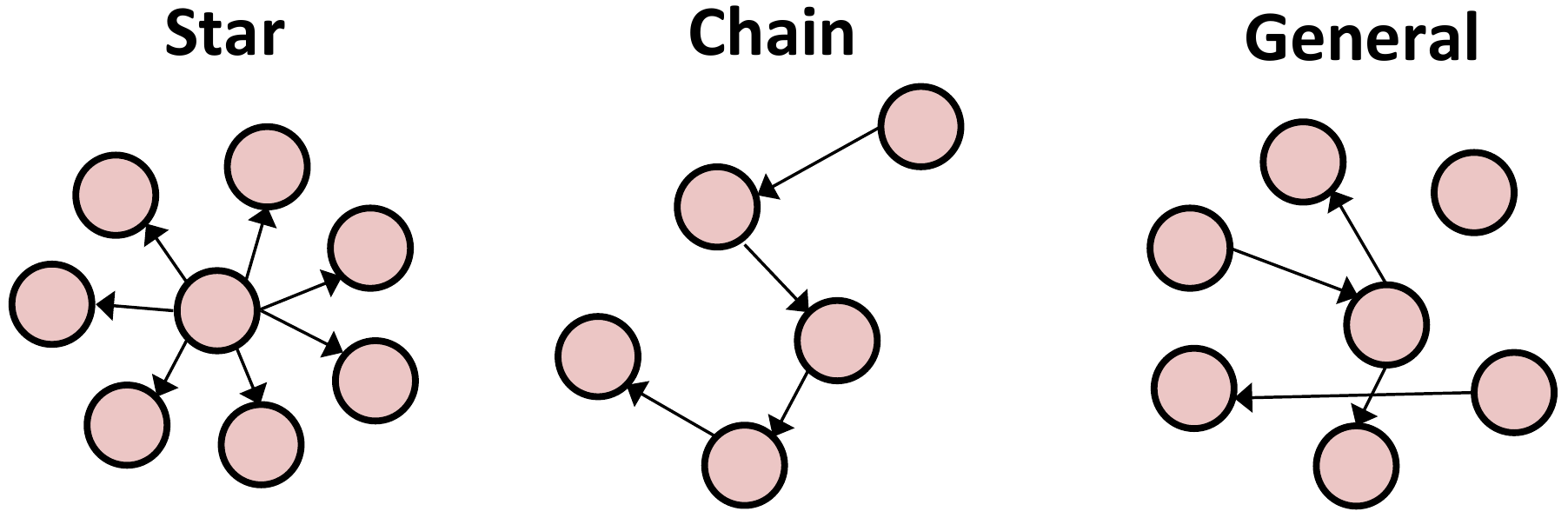} 
\end{center}
\vspace{-2mm}
\caption{Illustration of network topology.}
\label{fig:topology}
\vspace{-2mm}
\end{figure}

\section{Numerical examples}\label{sec:numerical_eg}

In this section, we present some numerical experiments using synthetic data. We focus on comparing EDD of our algorithm with two baseline methods, and demonstrate the accuracy of the analytic threshold. 

\subsection{Comparison of EDD}
\subsubsection{Two baseline algorithms}
We compare our method to two baseline algorithms: 

{\it (i)} \textbf{Baseline 1} is related to the commonly used ``data binning'' approach for processing discrete event data such as \cite{spikeTrain2003}. This approach, however,  ignores the temporal correlation and correlation between nodes. Here, we convert the event data into counts, by discretize time into uniform grid, and count the number of events happening in each interval. Such counting data can be modeled via Poisson distribution. We may derive a likelihood ratio statistic to detect  a change. Suppose $n_1, n_2, \dots, n_c$ are the sequence of counting numbers following Poisson distribution with intensity $\lambda_i$, $i = 1, 2, \dots, c$ is the index of the discrete time step. Assume under the null hypothesis, the intensity function is $\lambda_i = \mu$. Alternatively, there may exist a change-point $\kappa$ such that before the change, $\lambda_i=\mu$, and after the change, $\lambda_i= \mu^*$. It can be shown that 
the log-likelihood ratio statistic as 
\begin{equation*}
\ell_{c, k, \mu^*} = -(c-k) (\mu^*-\mu) +\sum_{i=k+1}^c n_i \mbox{log} \frac{\mu^*}{\mu}. \label{1d_baseratio}
\end{equation*}
We detect a change whenever $\max_{k<c} \max_{\mu^*} \ell_{k,c,\mu^*} > x$ for a pre-determined threshold $x$. 
Assume every dimension follows an independent Poisson process, then the log-likelihood ratio for the multi-dimensional case is just a summation of the log-likelihood ratio for each dimension. Suppose the total dimension is $d$, then
$$
\ell_{k,c, \bm{\mu}^*} =\sum_{j=1}^d \left[ -(c-k) (\mu_j^*-\mu_j) +\sum_{i=k+1}^c n_i^j \mbox{log} \frac{\mu_j^*}{\mu_j} \right].
$$
We detect a change whenever $\max_{k<c} \max_{\bm{\mu}^*} \ell_{k,c,\bm{\mu}^*} > x$.

{\it (ii)} \textbf{Baseline 2} method calculates the one-dimensional change-point detection statistic at each node separately as (\ref{loglikelihood_poi}) and (\ref{ratio_1dhaw}), and then combine the statistics by summation into a {\it global statistic} to perform detection. This approach, however, ignores the correlation between nodes, and can also be viewed as a {\it centralized} approach for change-point detection and it is related to multi-chart change-point detection \cite{changepoint_new_book2014}. 
\subsubsection{Set-up of synthetic experiments}

We consider the following scenarios and compare the EDD of our method to two baseline methods. EDD is defined as the average time (delay) it takes before we can detect the change, and can be understood as the power of the test statistic in the sequential setting. The thresholds of all the three methods are calibrated so that the ARL under the null model is $10^4$ unit time and the corresponding thresholds are obtained via direct Monte Carlo for a fair comparison. The sliding window is set to be $L=10$ unit time.  The exponential kernel $ \varphi (t)= \beta e^{-\beta t}$ is used and $\beta = 1$. The scenarios we considered are described below. The illustrations of the \textit{Case 1} and \textit{Case 2} scenarios are displayed in Fig. \ref{fig:illustration}. The network topology for  \textit{Case 3} to \textit{Case 7} are demonstrated in Fig. \ref{fig:topology}.

\textit{Case 1}. Consider a situation when the events first follow a one-dimensional Poisson process with intensity $\mu = 10$  and then shift to a Hawkes process with influence parameter $ \alpha = 0.5$. This scenario describes the emergence of temporal dependency in the event data. 

\textit{Case 2}. The process shifts from a one-dimensional Hawkes process with parameter $\mu=10$, $\alpha=0.3$ to another Hawkes process with a larger influence parameter $\alpha=0.5$. The scenario represents the change of the temporal dependency in the event data. 

\textit{Case 3}. Consider a star network scenario with one parent and nine children, which is commonly used in modeling how the information broadcasting over the network. 
Before the change-point, each note has a base intensity $\mu = 1$  and the self-excitation $\alpha_{i,i}=0.3$, $ 1 \leq i \leq 10$. The mutual-excitation from the parent to each child is set to be $\alpha_{1,j}=0.3$, $2 \leq j \leq 10$ (if we use the first node to represent the parent). After the change-point, all the self- and mutual- excitation increase to 0.5. 

\textit{Case 4}. The network topology is the same as Case 3. But we consider a more challenging scenario. Before the change, parameters are set to be the same as Case 3. After the change, the self-excitation $\alpha_{i,i}$, $1\leq i \leq 10$ deteriorate to $0.01$, and the influence from the parent to the children increase to $\alpha_{1,j}=0.6$, $j= 2 \leq j \leq10 $. In this case, for each note, the occurring frequency of events would be almost the same before and after the change-points. But the influence structure embedded in the network has actually changed.   

\textit{Case 5}. Consider a network with a chain of ten nodes, which is commonly used to model information propagation over the network. Before the change, each note has a base intensity $\mu = 1$  and the self-excitation $\alpha_{i,i}=0.3$, $1 \leq i \leq10$ and mutual-excitation $\alpha_{i,j}=0.3$, where $j-i=1, 1 \leq  i \leq 9$. After the change-point, all the self- and mutual-excitation parameters increase to 0.5.

\textit{Case 6}. Consider a {\it sparse} network with an arbitrary topology and one hundred nodes. Each note has a base intensity $\mu = 0.1$  and the self-excitation $\alpha_{i,i}=0.3$, $1 \leq i\leq100$. We randomly select twenty directed edges over the network and set the mutual-excitation to be $\alpha_{i,j}=0.3$, where $ i \neq j$, $i, j \text{ are randomly selected}$. After the change-point, all the self- and mutual-excitation increase to 0.5.

\textit{Case 7}. The {\it sparse} network topology and the pre-change parameters are the same with Case 6. The only difference is that after the change-point, only half of the self- and mutual-excitation parameters increase to 0.5.

\subsubsection{EDD results and discussions}
For the above scenarios, we compare the EDD of our method and two baseline algorithms. The results are shown in Table \ref{EDD}. We see our method compares favorably to the two baseline algorithms. In the first five cases, our method has a significant performance gain. Especially for Case 4, which is a challenging setting, only our method succeeds in detecting the spatial structure changes. For Case 6 and Case 7, our method has similar performance as Baseline 2. A possible reason is that in these cases the network topology is a sparse graph so the nodes are ``loosely'' correlated. Hence, the advantage of combining over graph is not significant in these cases.

Moreover, we observe that Baseline 1 algorithm is not stable. In certain cases (Case 6 and Case 7), it completely fails to detect the change. An explanation is that there is a chance that the number of events fall into a given time bin is extremely small or close to zero, and this causes numerical issues when calculating the the likelihood function (since there is a log function of the number of events). On the other hand, our proposed log-likelihood ratio is event-triggered, and hence will avoid such numerical issues.

\begin{table}[H]
\centering
\caption{EDD comparison. Thresholds for all methods are calibrated such that $ARL = 10^4$.}
\label{EDD}
\begin{tabular}{|c|c|c|c|}
\hline
       & \textbf{Baseline 1}      & \textbf{Baseline 2}      & \textbf{Our Method} \\
       \hline \hline
\textit{Case 1} & 22.1           &    $-$            & \textbf{4.8}       \\
   \hline 
\textit{Case 2} & 19.6           &    $-$         & \textbf{18.8}      \\
   \hline 
\textit{Case 3} & 8.2            & 6.9            & \textbf{4.3}       \\
   \hline 
\textit{Case 4} & $\times$  & $\times$ & \textbf{19.8}      \\
   \hline 
\textit{Case 5} & 6.1            & 5.7           & \textbf{4.7}       \\
   \hline 
\textit{Case 6} & $\times$             & \textbf{10.5}           & 10.8      \\
   \hline 
\textit{Case 7} & $\times$  & \textbf{32.5}           & \textbf{32.5}     \\
\hline
\end{tabular}
\vspace{1ex}
 
\raggedright{\textit{Note: `$\times$' means the corresponding method fails to detect the changes; `$-$' means in one-dimensional case Baseline 2 is identical to ours.}}
\vspace{-0.1in}
\end{table}

%From the results, we see (Add discussions here.......)

\subsection{Sensitivity analysis}
We also perform the sensitivity analysis by comparing our method to Baseline 1 algorithm via numerical simulation. The comparison is conducted under various kernel decay parameter $\beta$, and the strength of the post-change signals, which can be controlled by the magnitudes of the changes in $\alpha$ (or $\bm{A}$). For each dataset, we created 500 samples of sequences with half of them containing one true change-point and half of them containing no change-point. We then plot the {\it area under the curve} (AUC)  (defined as the true positive rate versus the false positive rate under various threshold) for comparison, as shown in Fig.~\ref{AUC}. 

\subsubsection{Set-up of synthetic experiments}
Overall, we consider various decay parameter $\beta$ and the magnitudes of the changes in $\alpha$ to compare the approaches. 

\textit{One-dimensional setting.} First, consider that before the change the data is a Poisson process with base intensity $\mu = 1$. For A.1-A.4, the post-change data become one dimensional Hawkes process: for A.1--A.3, $\alpha=0.2$, and $\beta = 1, 10, 100$, respectively; for A.4, $\alpha=0.3$, and $\beta = 10$. By comparing the AUC curves, we see that, our method has a remarkably better performance in distinguishing the true positive changes from the false positive changes compared to the baseline method. The superiority would become more evident under larger $\beta$ and bigger magnitudes of shifts in $\alpha$. For weak changes, the baseline approach is just slightly better than the random guess, whereas our approach consistently performs well. 
Similar results can be found if the pre-change data follow the Hawkes process. For example, in  B.1-B.3, the pre-change data follow Hawkes process with $\mu=1$, $\alpha=0.3$, and $\beta=1$, and the post-change parameters shift to a Hawkes process with $\alpha=0.5$, and $\beta=1, 10, 100$, respectively. We can see the similar trend as before by varying $\beta$ and $\alpha$. 

\textit{Network setting.} We first consider the two-dimensional examples in the following and get the same results. For C.1-C.2, the pre-change data follow two dimensional Poisson processes with $\bm{\mu} = [0.2,0.2]^{\intercal}$, and the post-change data follow two dimensional Hawkes processes with influence parameter $\bm{A}=[0.1,0.1;0.1,0.1]$, with $\beta = 1, 10$, respectively. For D.1--D.3, consider the star network with one parent and nine children. Before the change-point, for each node the base intensity is $\mu=0.1$, $\beta=1$, and the influence from the parent to each child is $\alpha=0.3$. After the change, $\alpha$ changes to 0.4 for D.1, and $\alpha$ changes to 0.5,  $\beta=1,10$ respectively for D.2 and D.3. 
\begin{figure}[]
        \begin{center}
        \begin{tabular}{ccc}
                \includegraphics[
                width = 0.3\linewidth
                            ]{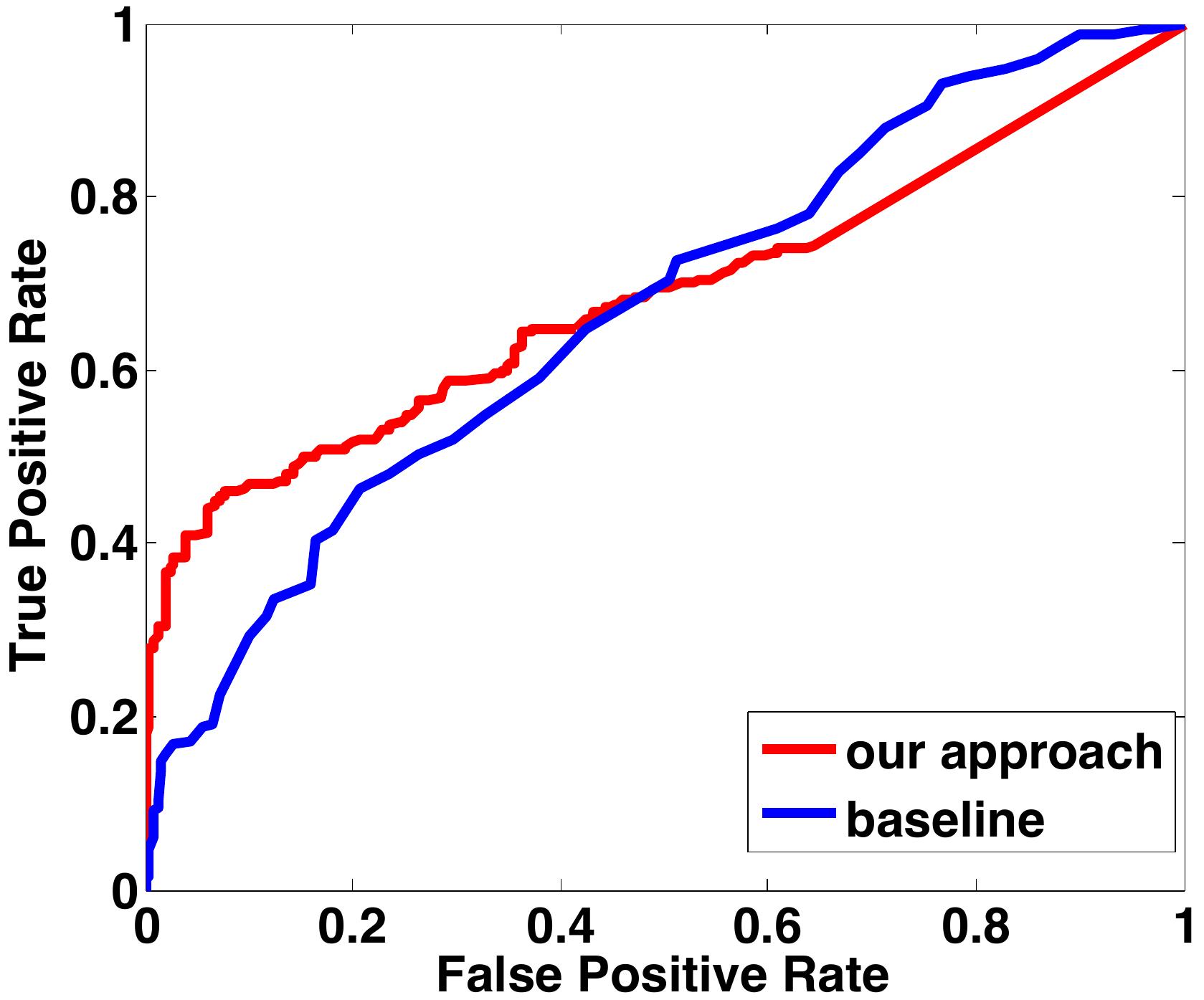} 
                 &
                    \includegraphics[
                    width = 0.28\linewidth
                    ]   {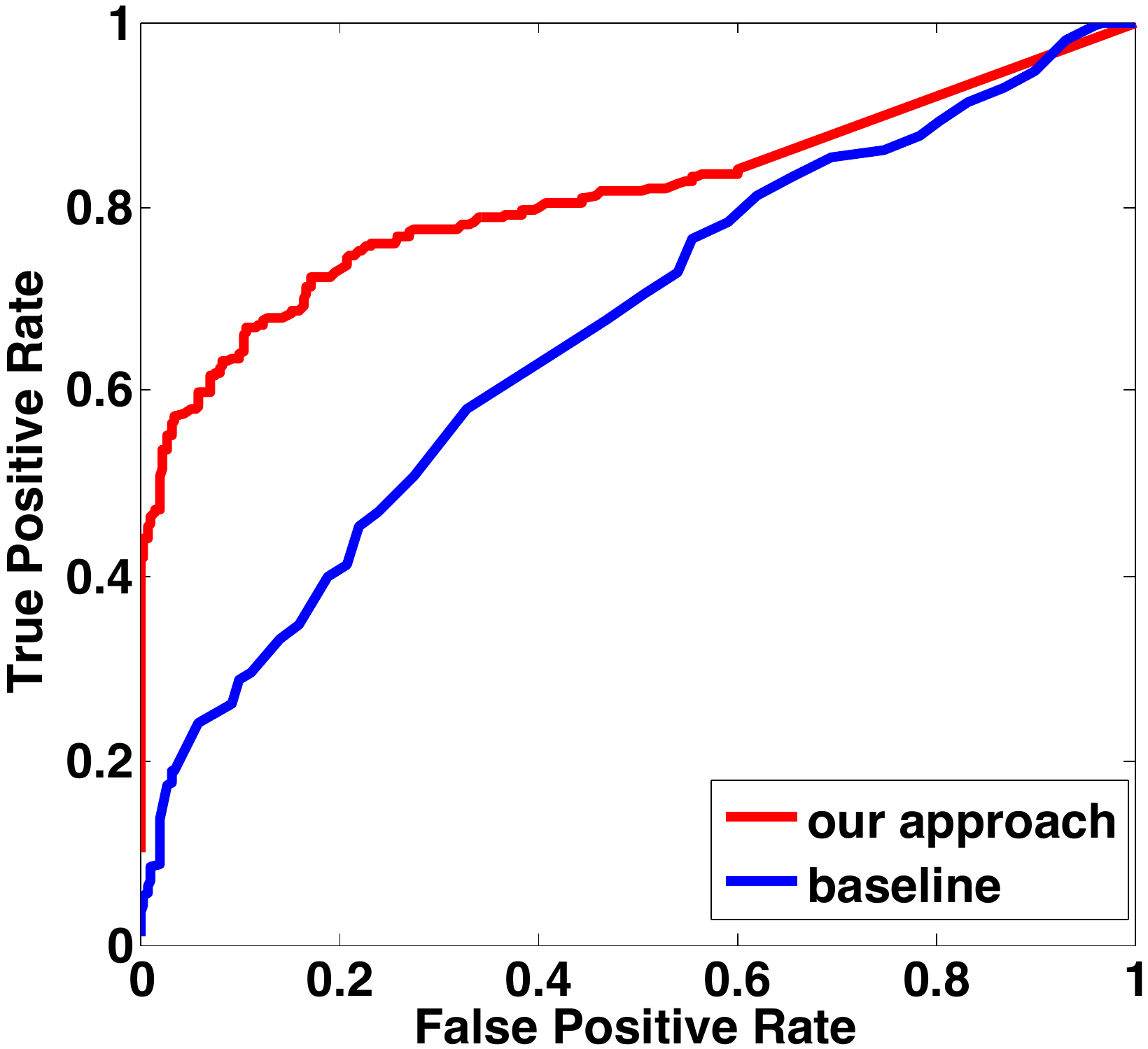}
                  &
                    \includegraphics[
                    width = 0.3\linewidth
                    ]  {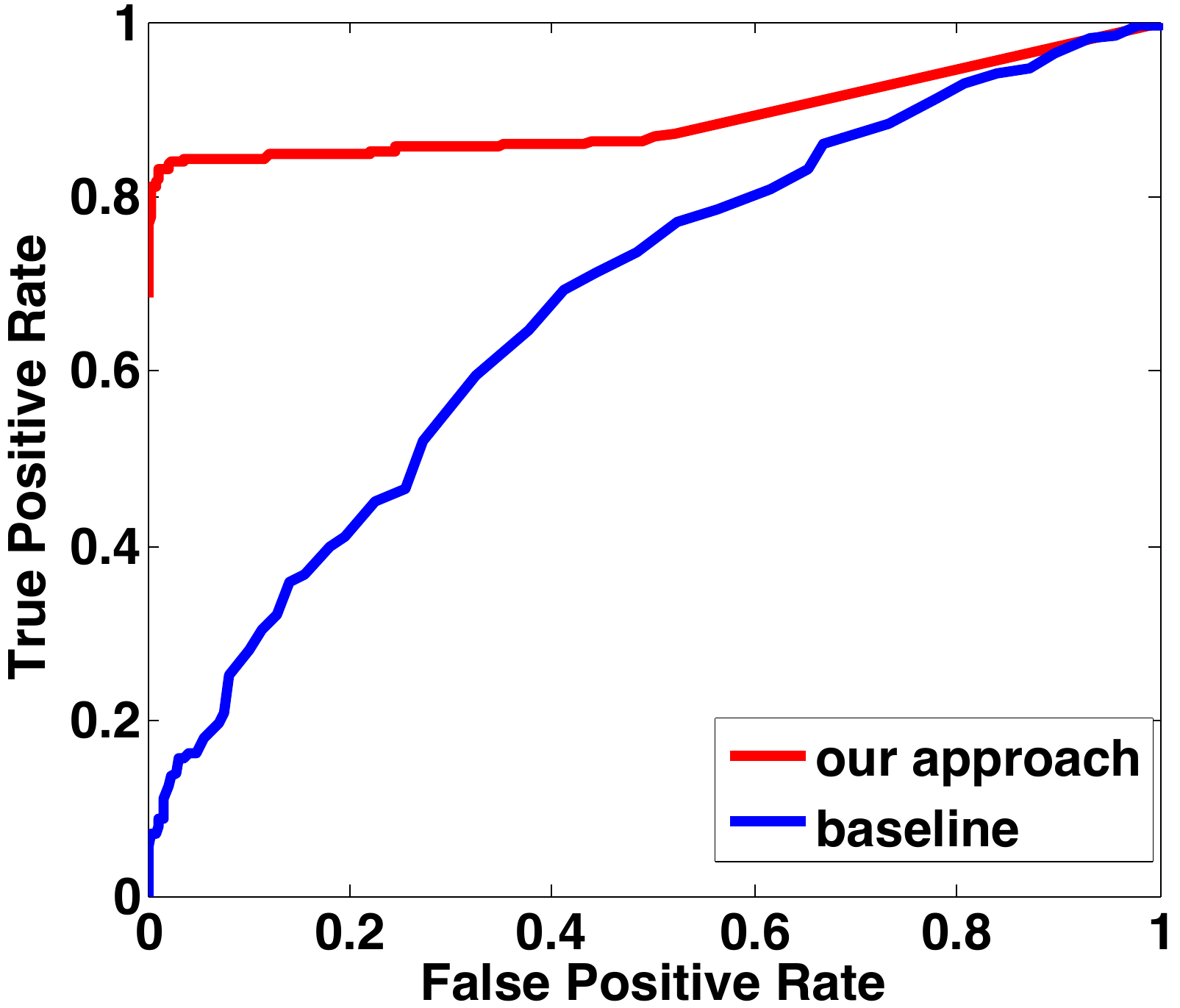}  \\  
                 A.1&A.2&A.3 \\
                    \includegraphics[
                    width = 0.3\linewidth
                    ]  {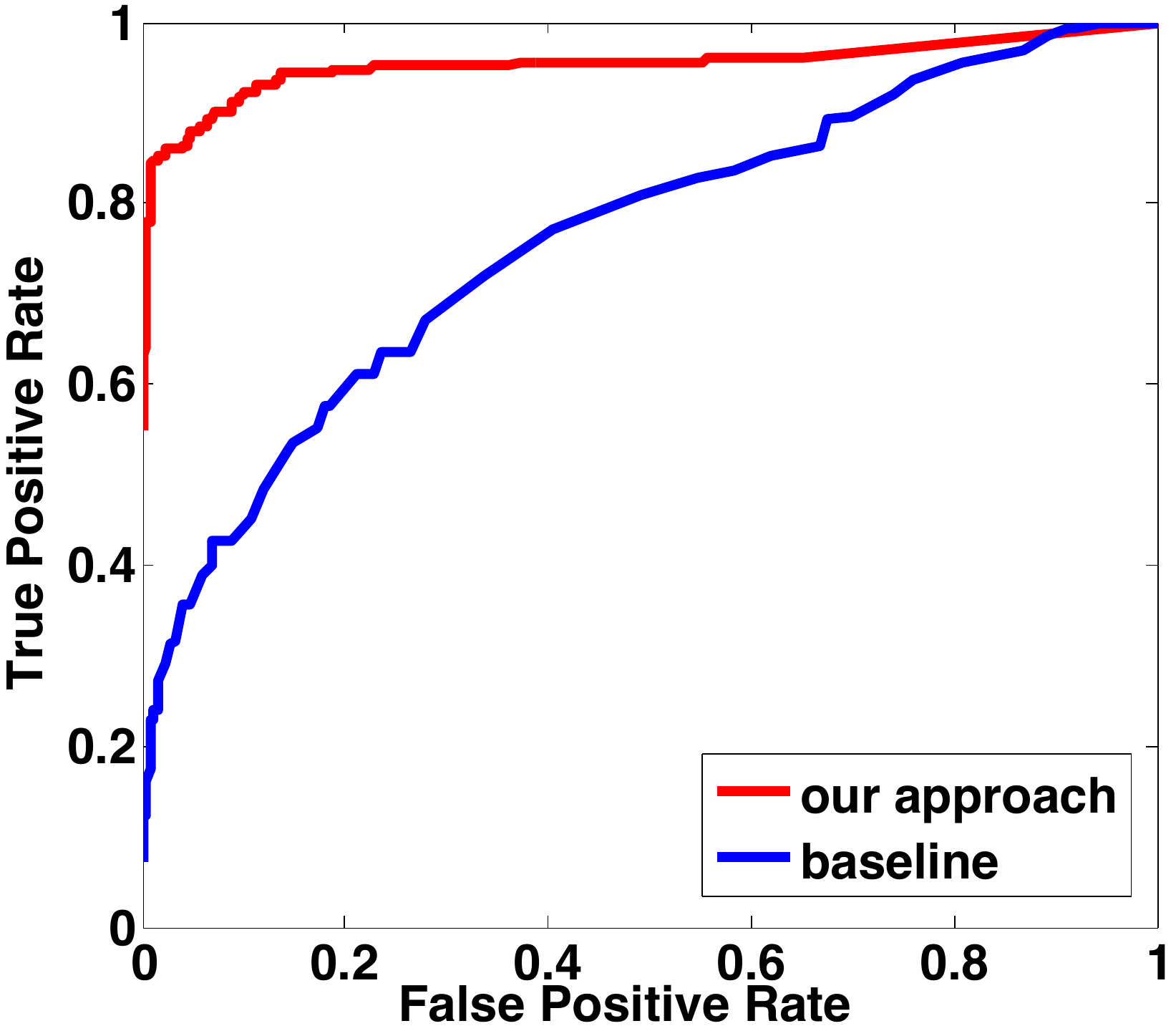}    
               &
               \includegraphics[
               width = 0.3\linewidth
                            ]{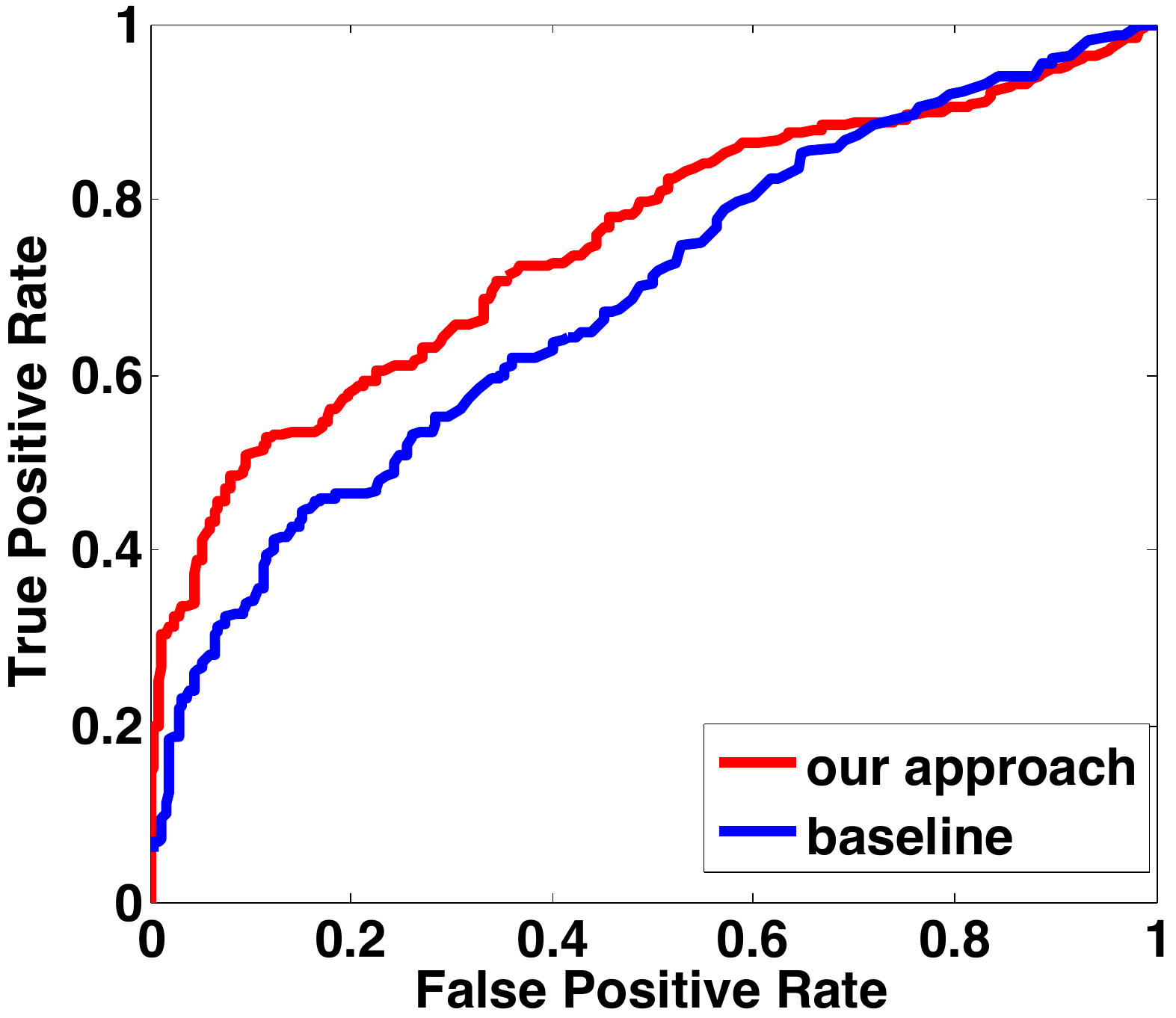} 
                 &
                    \includegraphics[
                    width = 0.3\linewidth
                    ]   {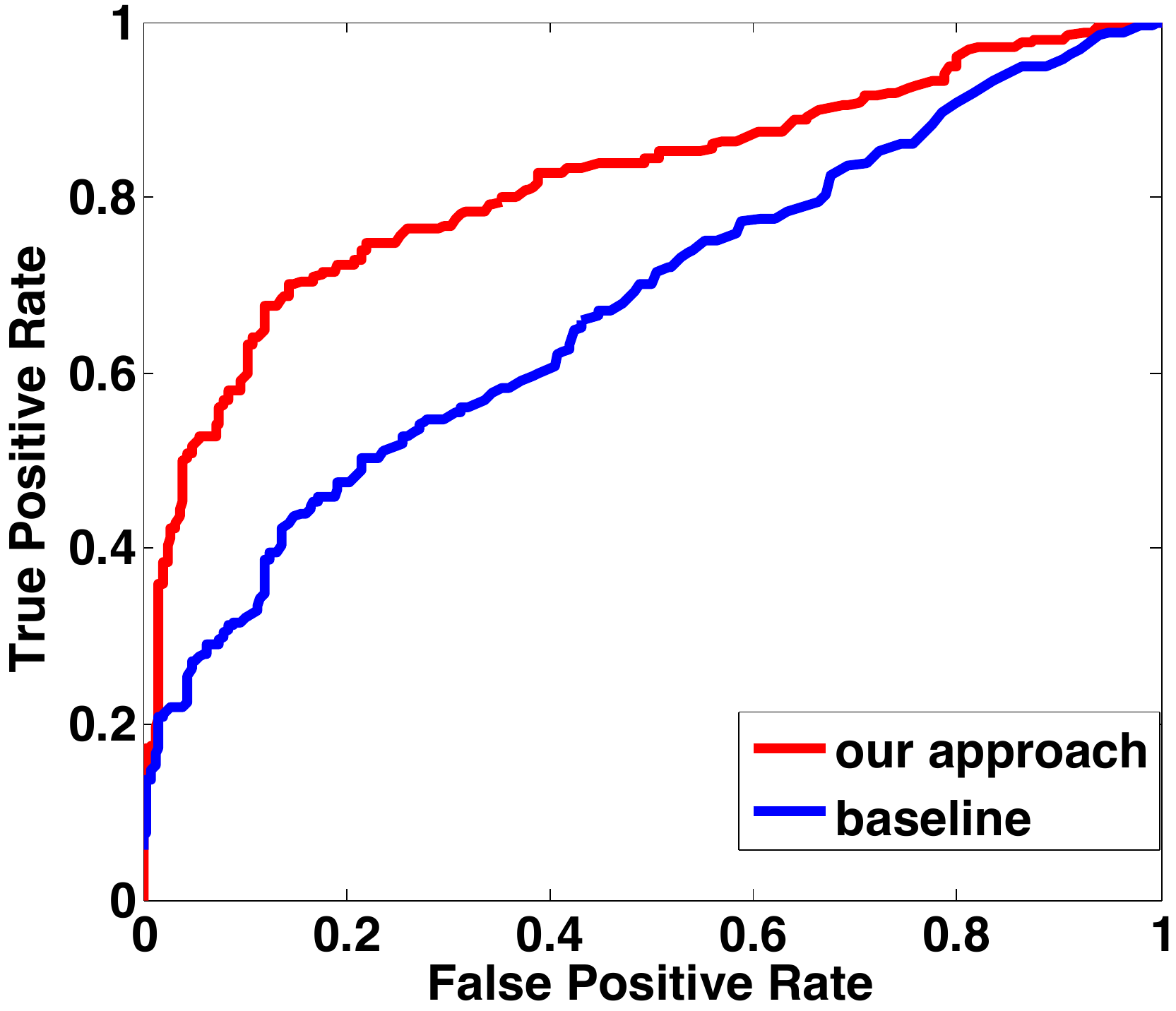} \\
                A.4&B.1&B.2 \\    
                    \includegraphics[
                    width = 0.3\linewidth
                    ]  {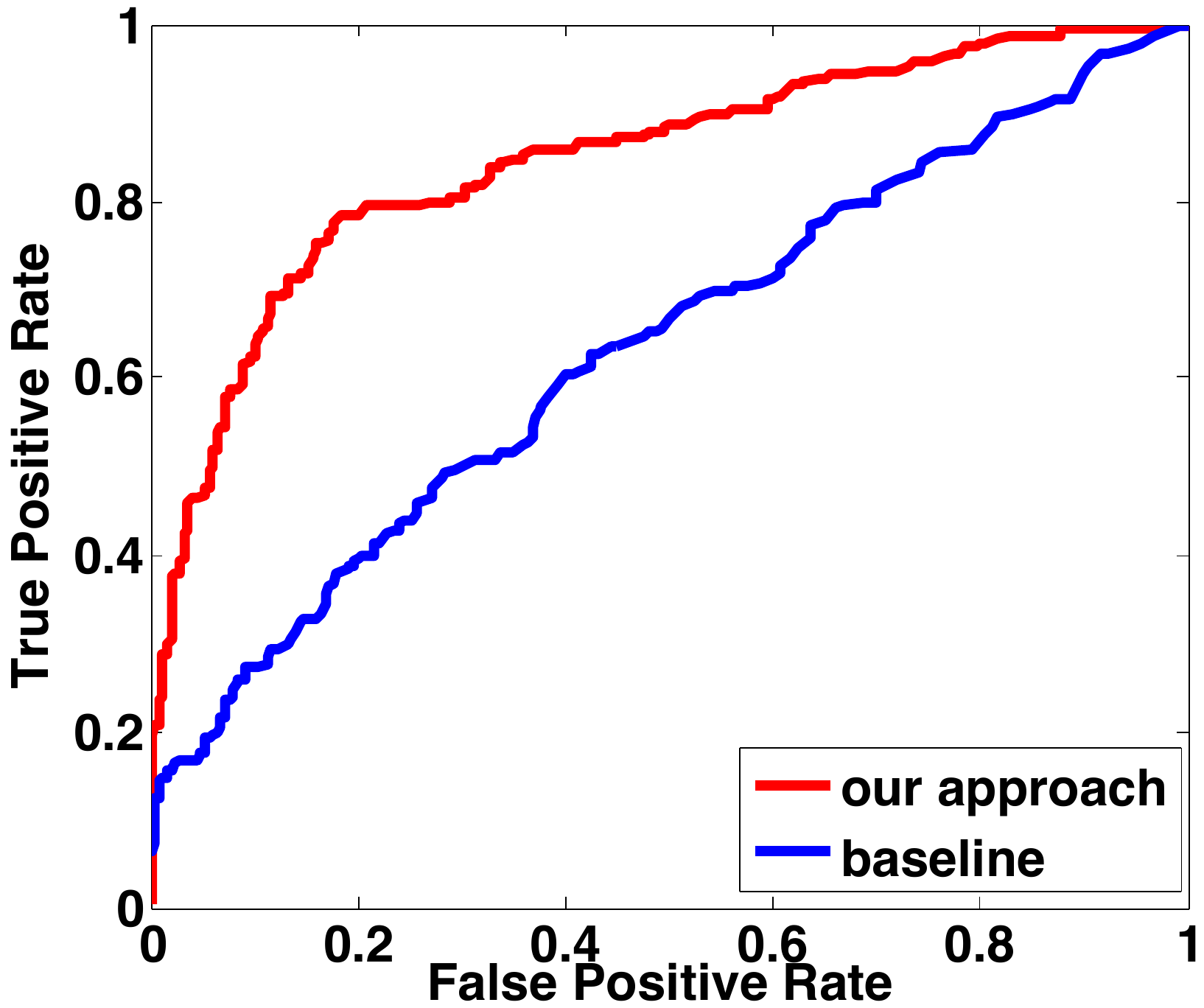}    
                &
                    \includegraphics[
                    width = 0.3\linewidth
                    ]  {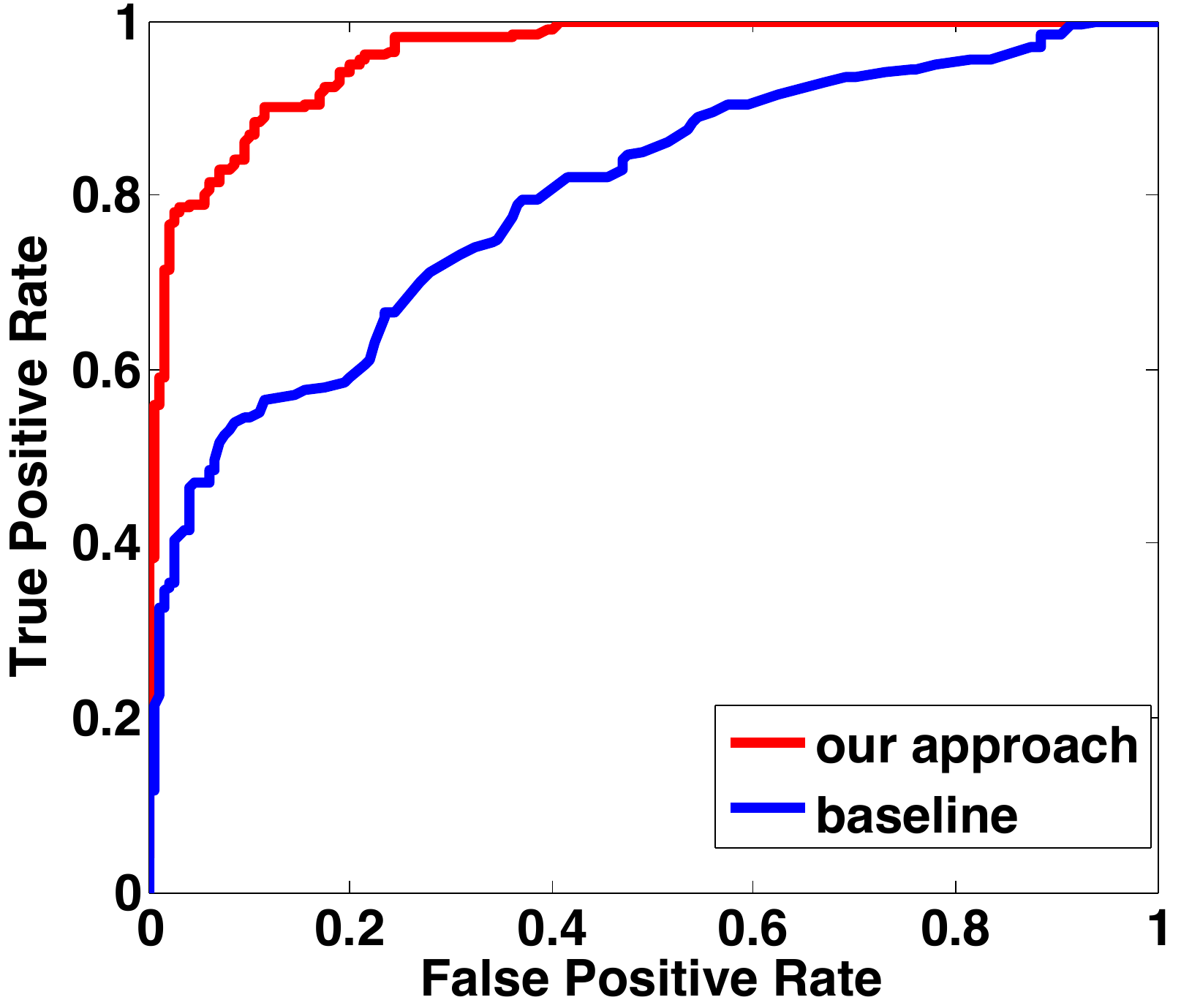}    
               &
             \includegraphics[
                    width = 0.3\linewidth
                            ]{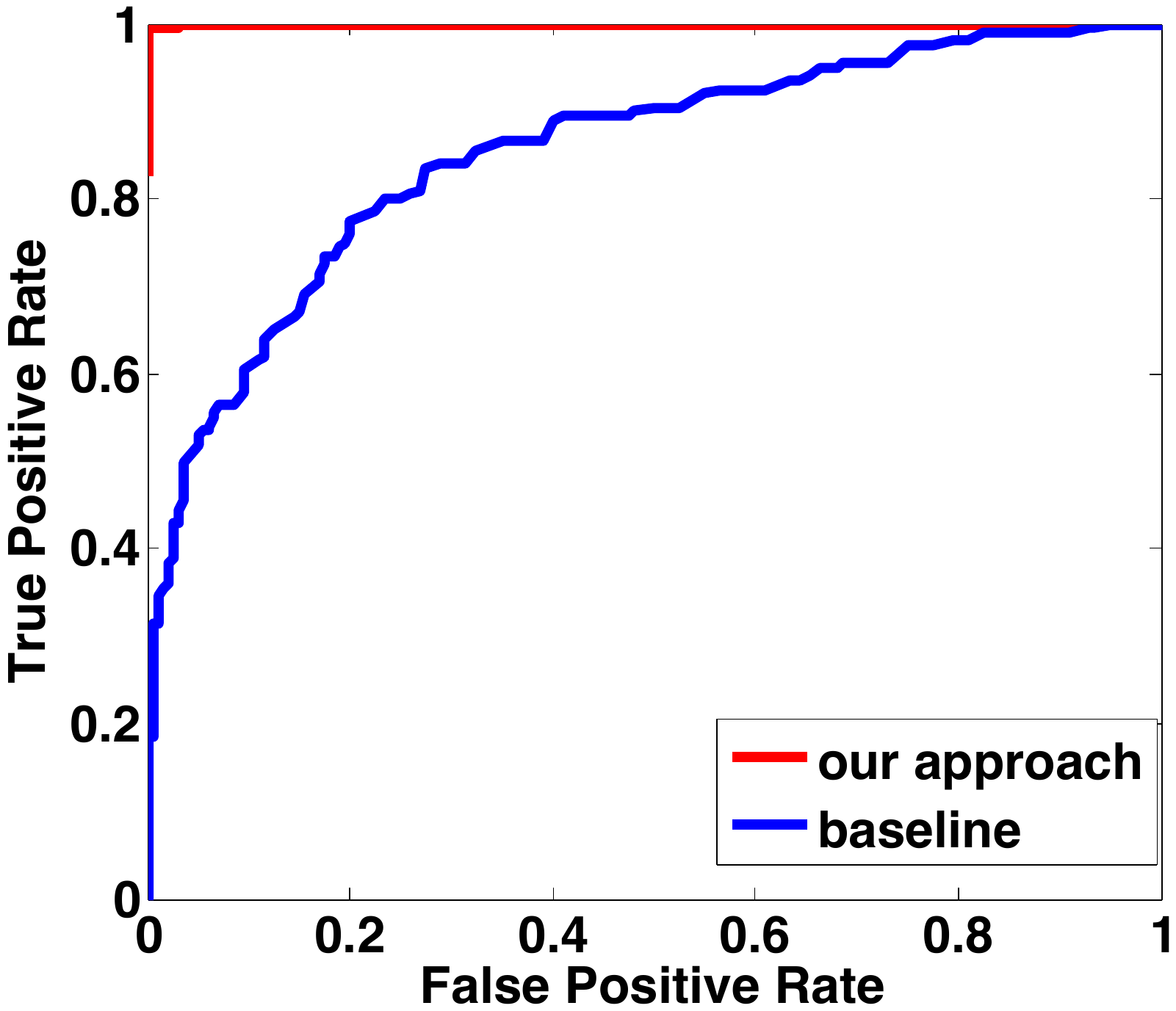}  \\         
                 B.3&C.1&C.2 \\    
                    \includegraphics[
                    width = 0.3\linewidth
                    ]   {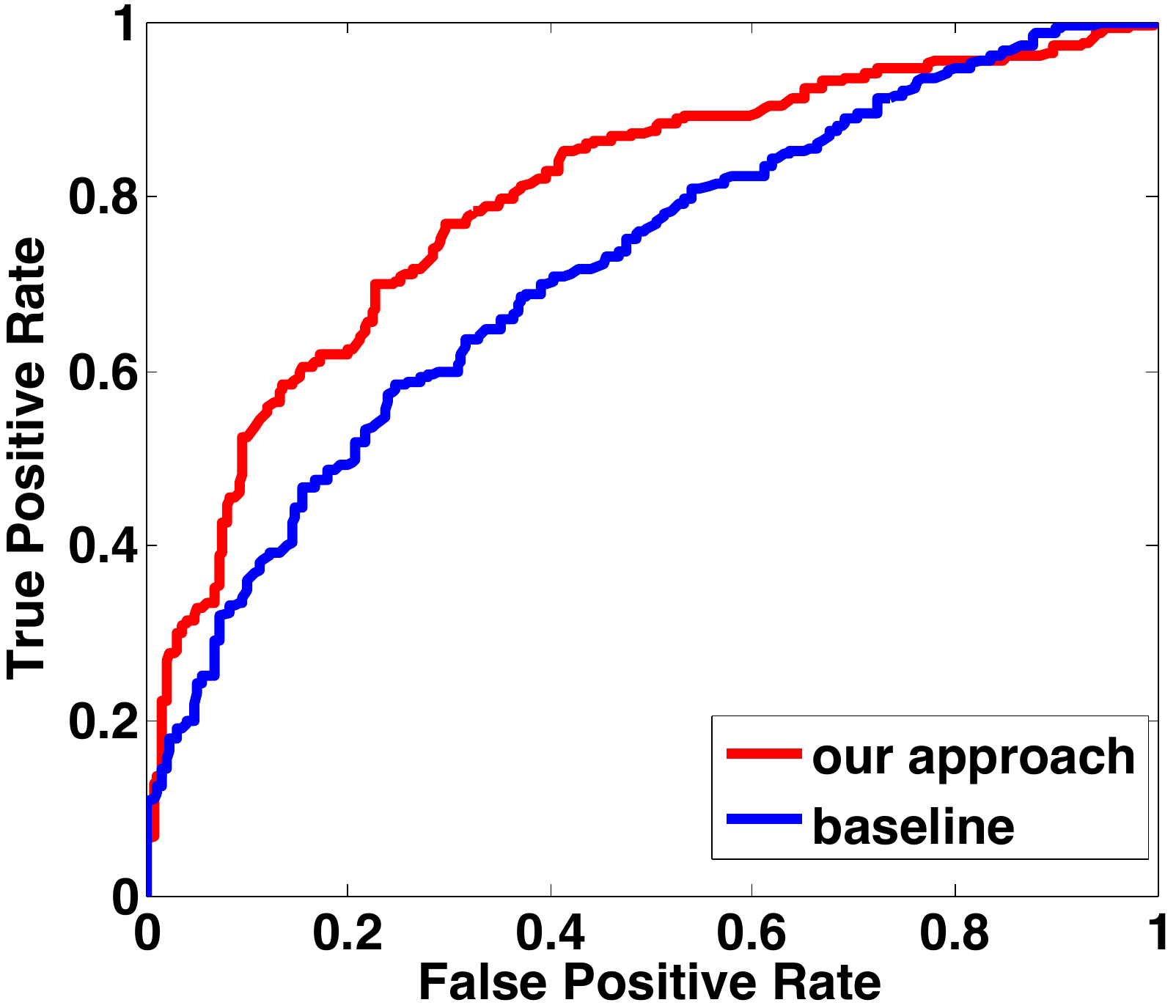}
                  &
                    \includegraphics[
                    width = 0.3\linewidth
                    ]  {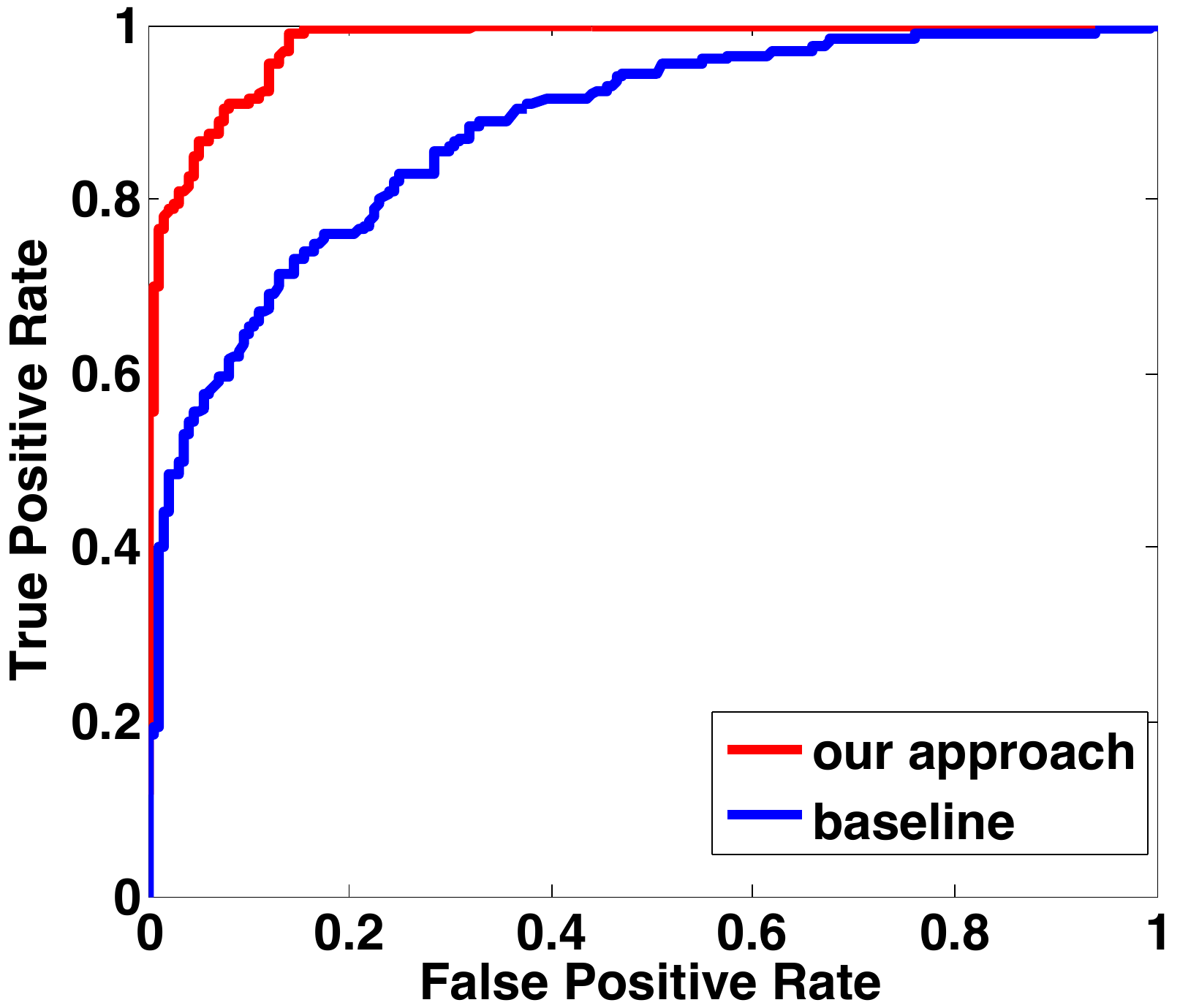}    
                &
                    \includegraphics[
                    width = 0.3\linewidth
                    ]  {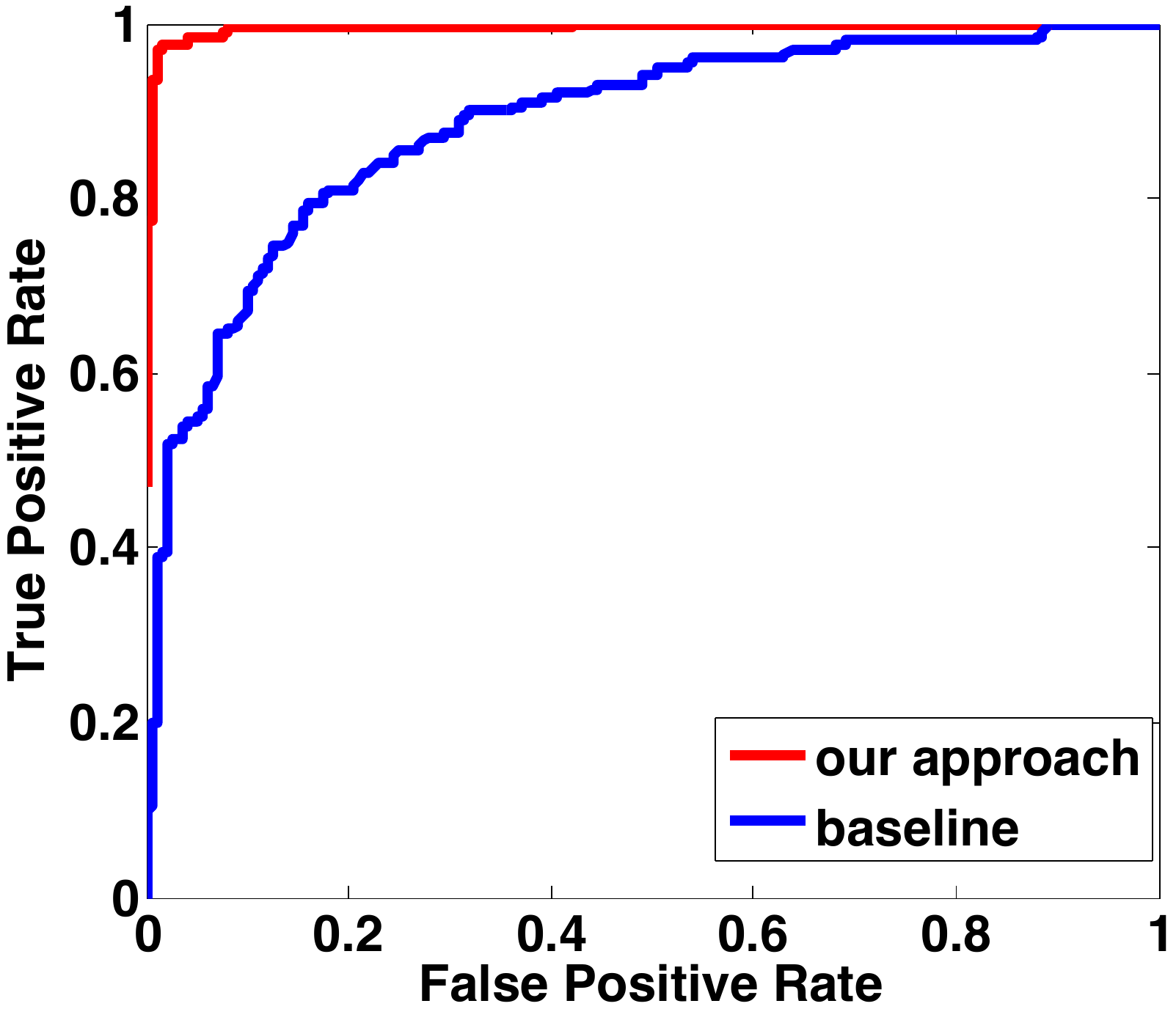}    
               \\
            D.1&D.2&D.3 
                    \end{tabular}
                 \end{center}
              \vspace{-3 mm}
                \caption{ AUC curves: comparison of our method with Baseline 1.}
               \label{AUC}
            \vspace{-3mm}
 \end{figure}
 
 \begin{figure}[t!] 
        \begin{center}
        \begin{tabular}{cc}
                \includegraphics[
                width = 0.38\linewidth
                            ]{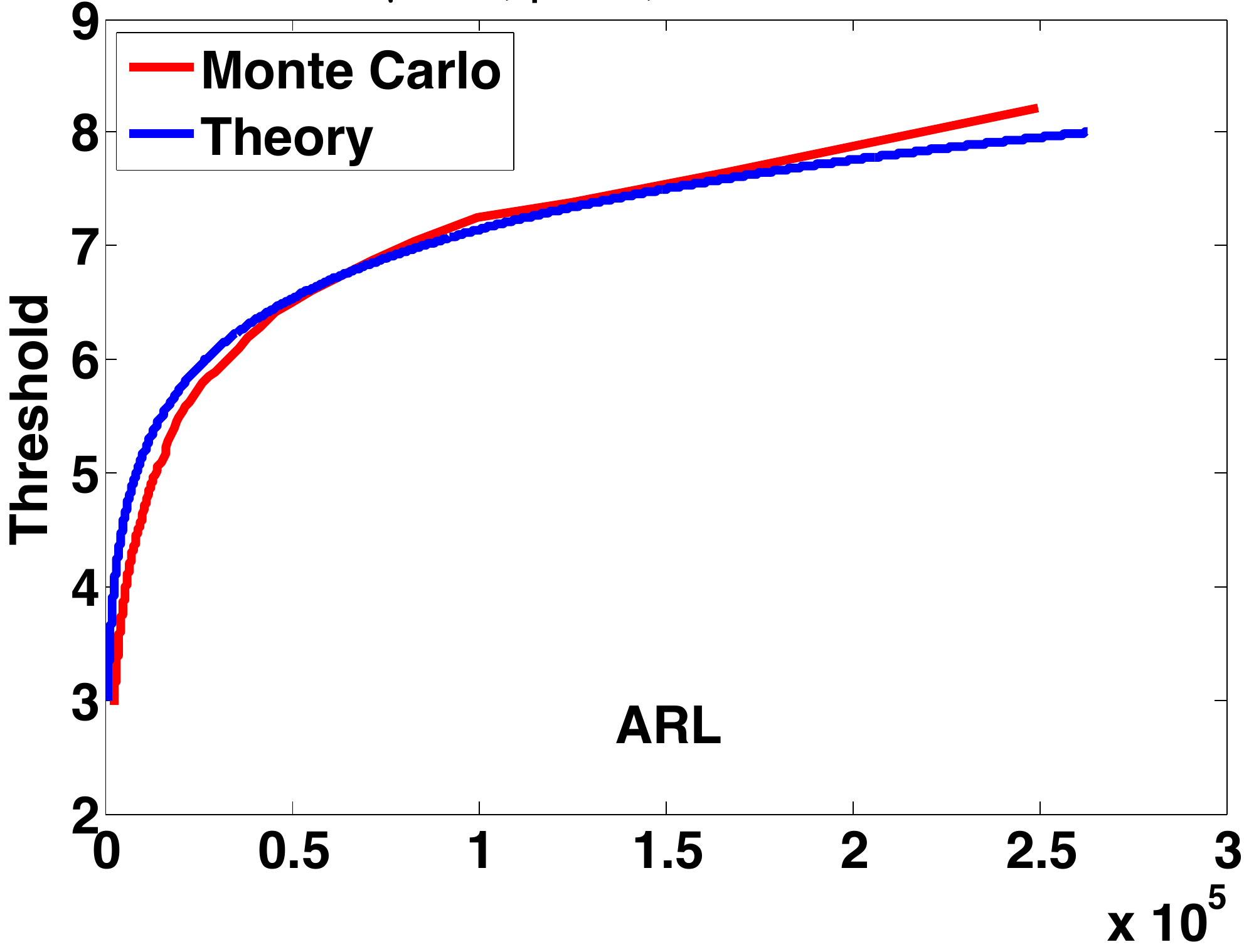} 
                 &
                    \includegraphics[
                     width = 0.4\linewidth
                    ]   {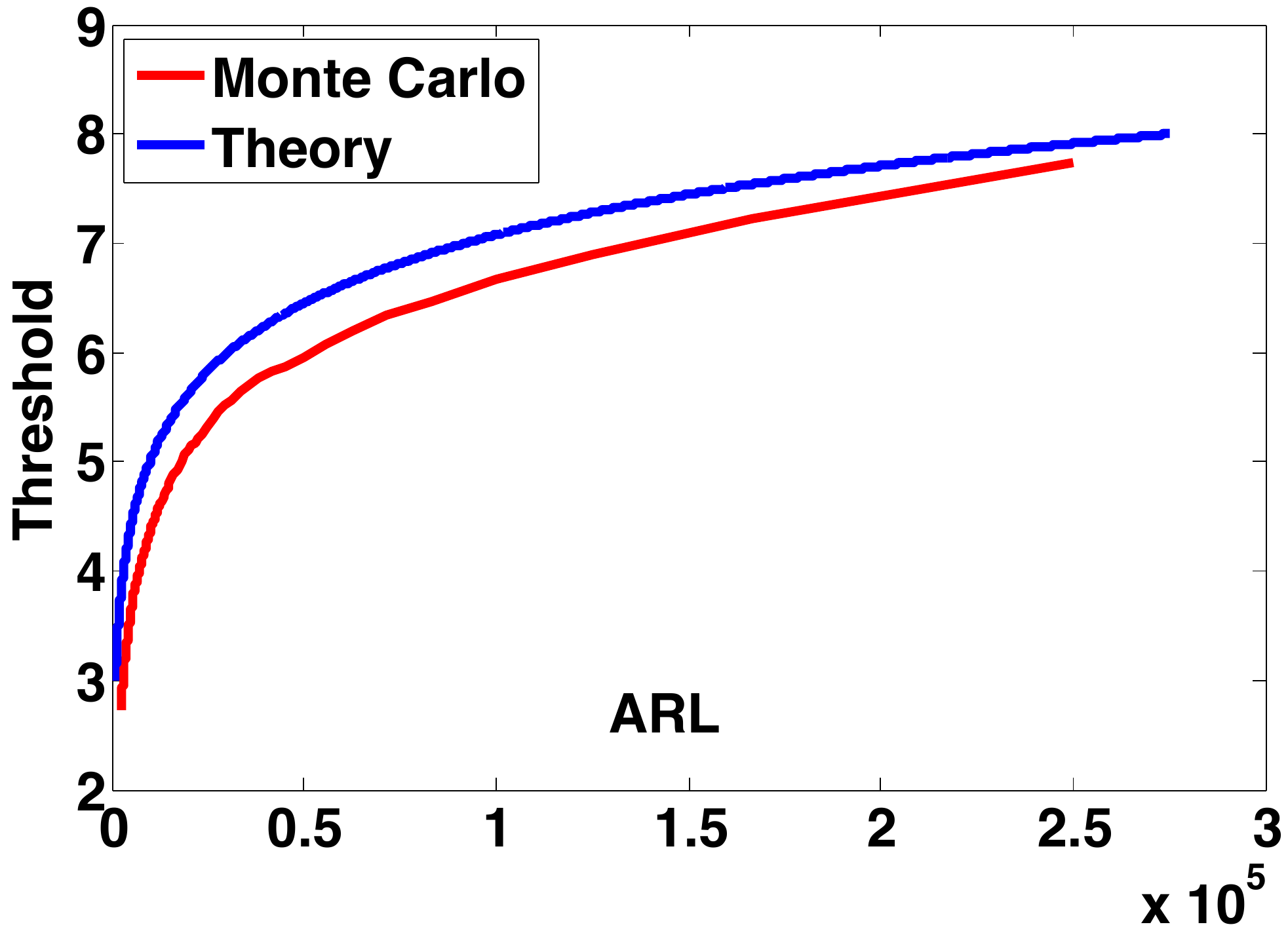}
                  \\
                  (a) & (b) \\
                     \includegraphics[
                      width = 0.38\linewidth
                    ]   {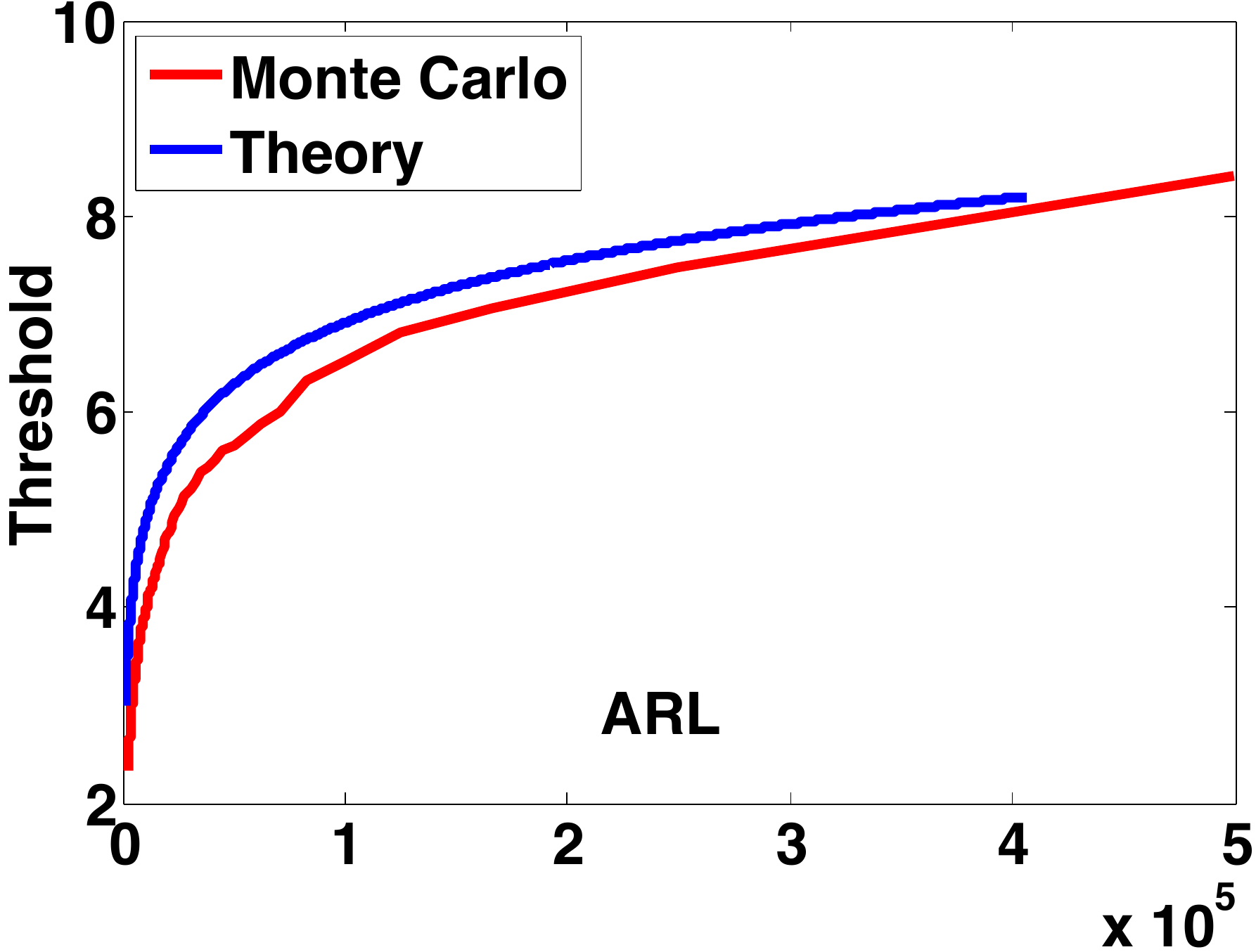} 
     &
                     \includegraphics[
                      width = 0.4\linewidth
                    ]   {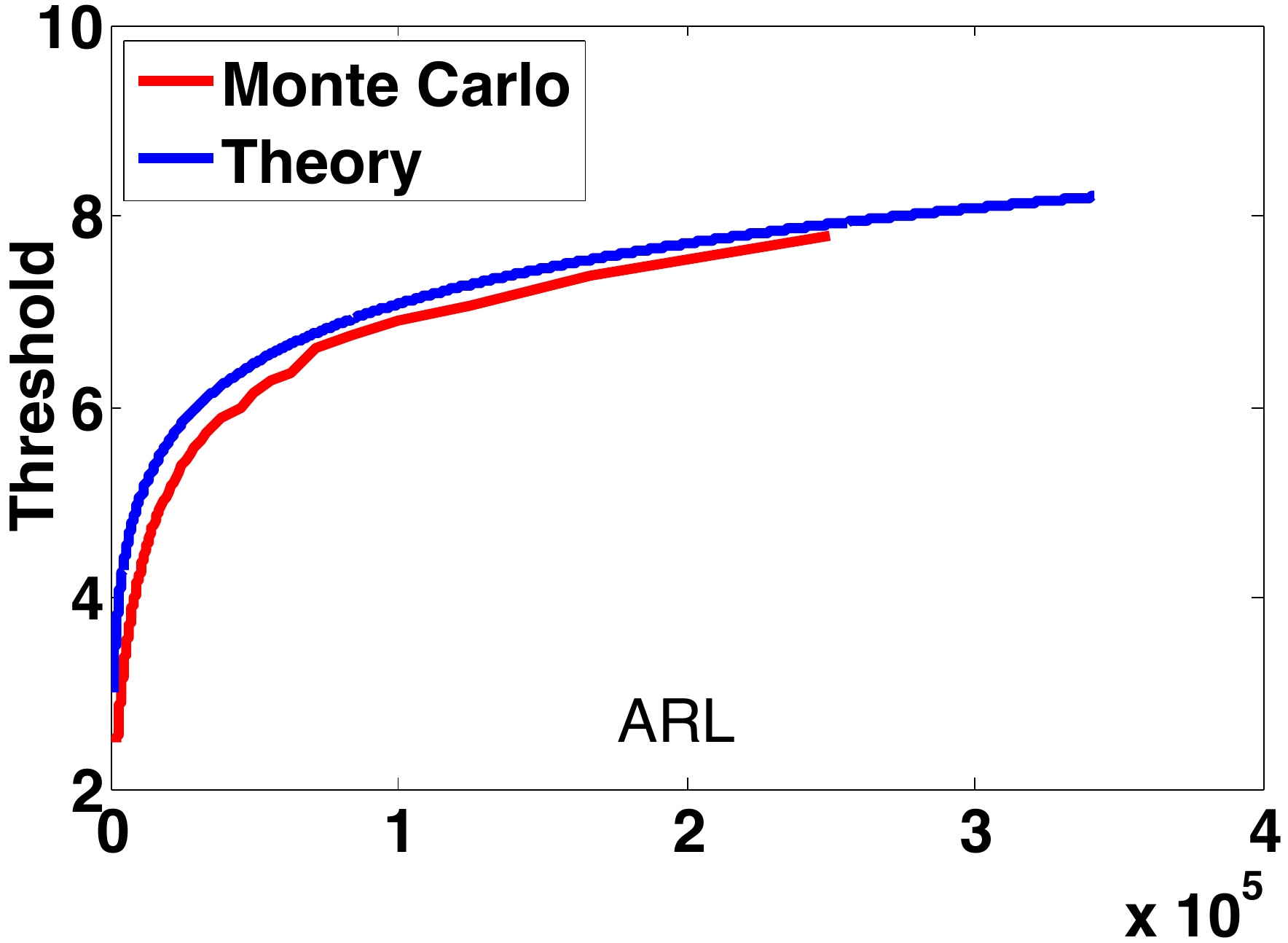}      
                \\
                (c) & (d) \\
                     \includegraphics[
                      width = 0.4\linewidth
                            ]{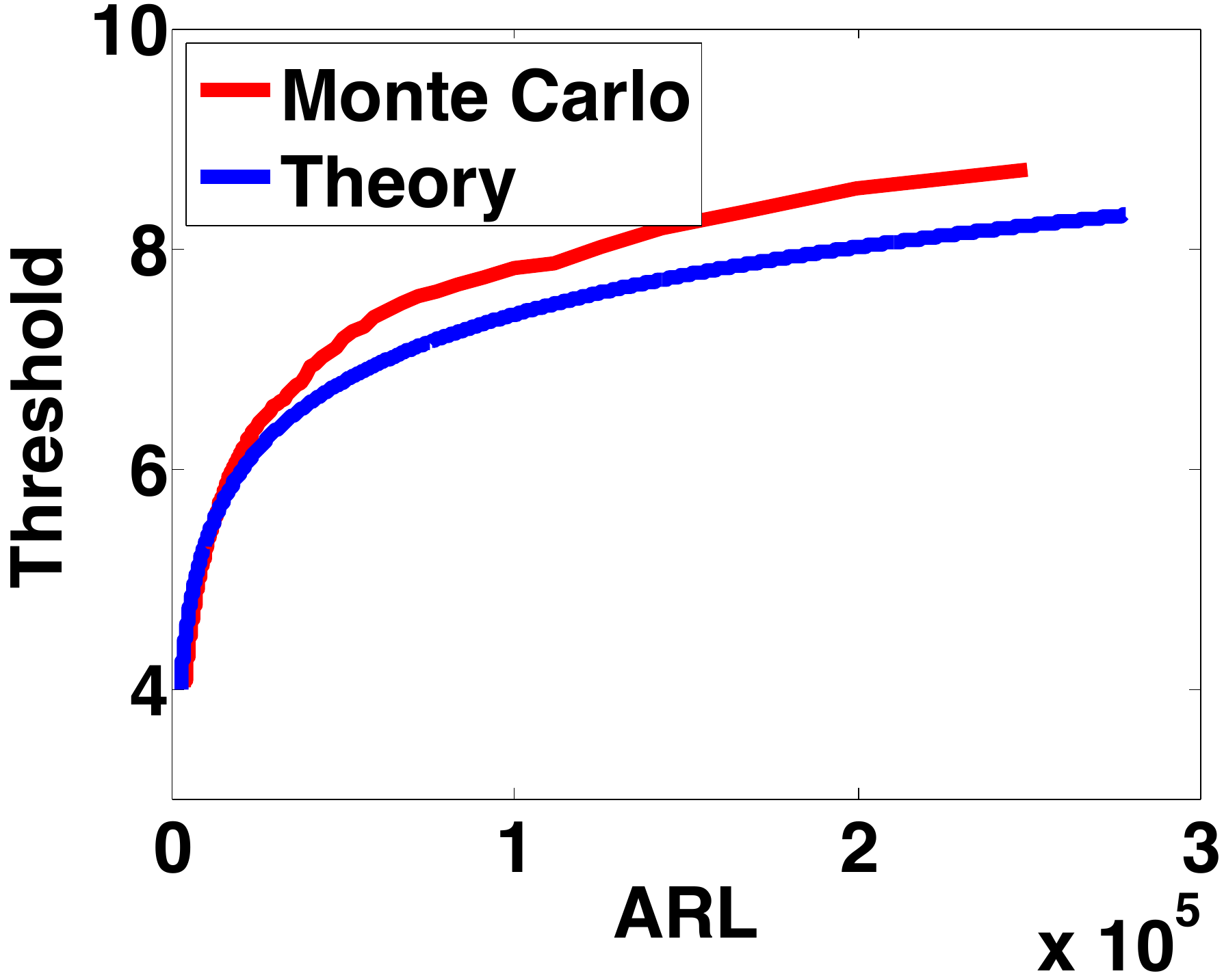} 
                 &
                    \includegraphics[
                     width = 0.4\linewidth
                    ]   {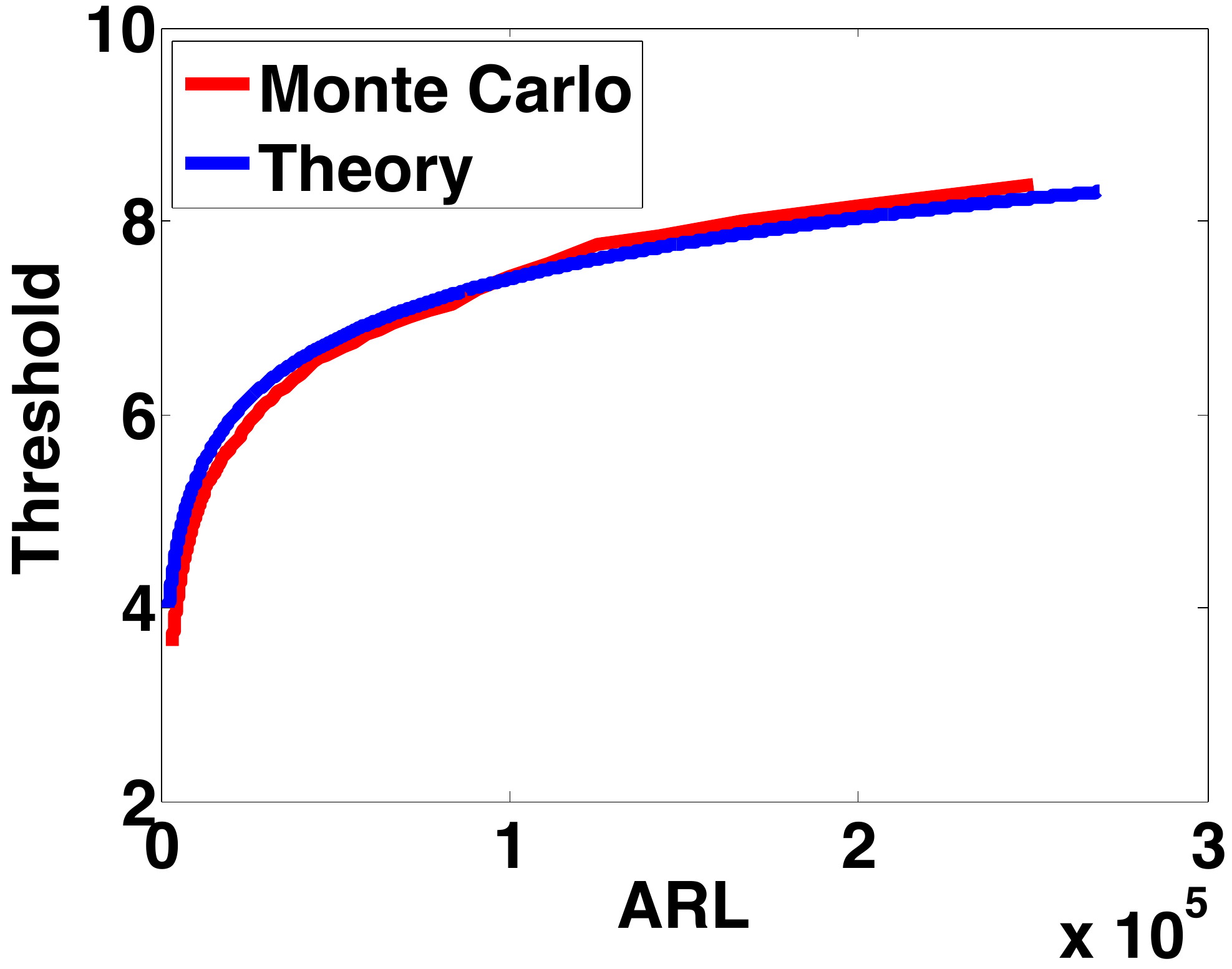}
             \\
              (e) & (f) \\
                     \includegraphics[
                      width = 0.4\linewidth
                    ]   {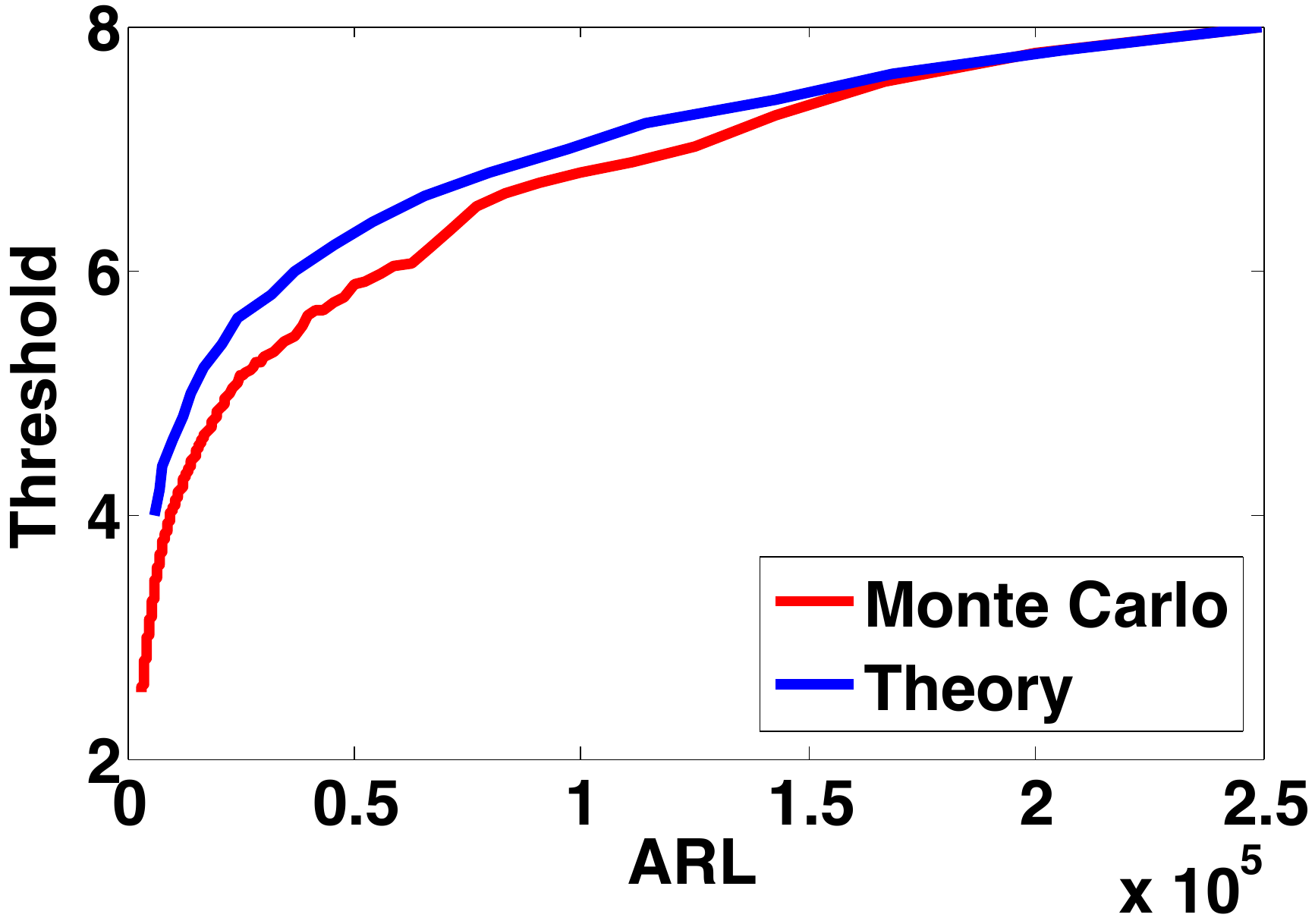} 
                    &
                     \includegraphics[
                      width = 0.4\linewidth
                    ]   {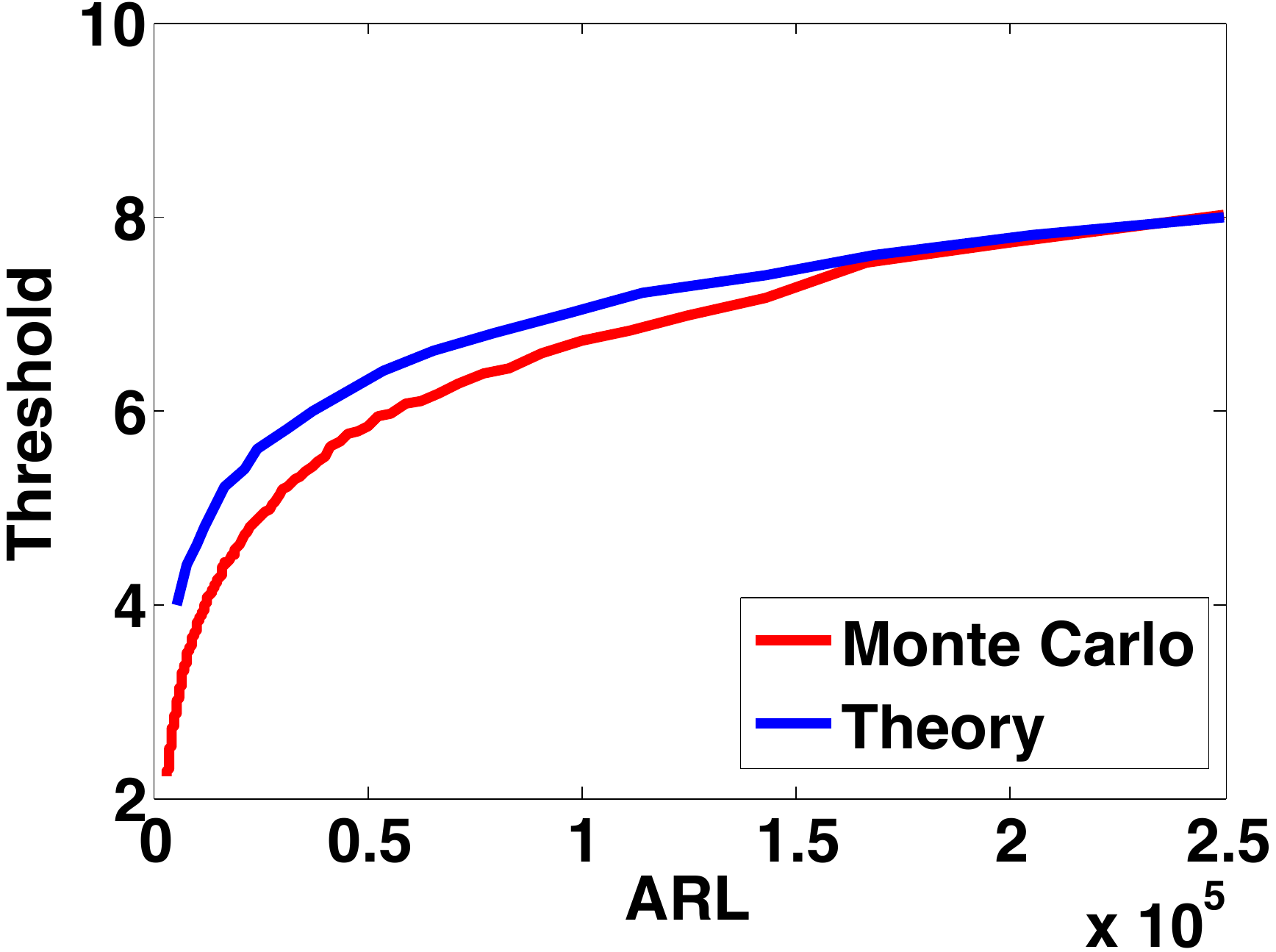}   
                         \\
                 (g) & (h)\\
                    \end{tabular}
                 \end{center}
                \vspace{-3mm}
                \caption{Comparison of theoretical threshold obtained via Theorem \ref{thm:ARL} with simulated threshold.}
                \label{thresh}
               % \label{offlinetail}
               \vspace{-4mm}
 \end{figure}

\subsection{Accuracy of theoretical threshold} \label{eva_ARL}

We evaluate the accuracy of our approximation in Theorem \ref{thm:ARL} by comparing the threshold obtained via Theorem \ref{thm:ARL} with the true threshold obtained by direct Monte Carlo. We consider various scenarios and parameter settings. We demonstrate the results in Fig. \ref{thresh} and list the parameters below. 

For Fig. \ref{thresh}-(a)(b)(c), the null distribution is one-dimensional Poisson process with intensity $\mu=1$. We choose $\beta=1$ as a priori, and vary the length of the sliding time window. We set $L=10, 50, 100$, respectively. For Fig. \ref{thresh}-(d), we select $L=50$ and let $\beta = 10$. By comparing these four examples, we find our approximated threshold is very accurate regardless of $L$ and $\beta$.

For Fig. \ref{thresh}-(e)(f), the null hypothesis is a one-dimensional Hawkes process with base intensity $\mu=1$ and influence parameter $\alpha = 0.3$, $\beta=10$. We vary the sliding window length to be $L=100,150$, respectively. We can see the accurate approximations as before. For Fig. \ref{thresh}-(g)(h), we consider a multi-dimensional case. The null distribution is a two dimensional Poisson processes with base intensity $\bm{\mu}=[0.5,0.5]^{\intercal}$. We set $\beta = 1$ and vary the window length to be $L=300$ and $400$ respectively. The results demonstrate that our analytical threshold is also sharply accurate in the multi-dimensional situation.

\vspace{-0.2in}
\section{Real-data}

We evaluate our online detection algorithm on real Twitter and news websites data. %The twittering behaviors over the network, for example, can be modeled as multi-dimensional point processes. 
By evaluating our log-likelihood ratio statistic on the real twittering events, we see that the statistics would rise up when there is an explanatory major event in actual scenario. By comparing the detected change points to the true major event time, we verify the accuracy and effectiveness of our proposed algorithm.
In all our real experiments, we set the sliding window size to be $L=500$ minutes, and set the kernel bandwidth $\beta$ to be 1. The number of total events for the tested sequences ranges from 3000 to 15000 for every dataset. 
\begin{figure}[h!]
  \centering
  \vspace{-4mm}
  \includegraphics[width=.6\linewidth]%{methodgraph}
    {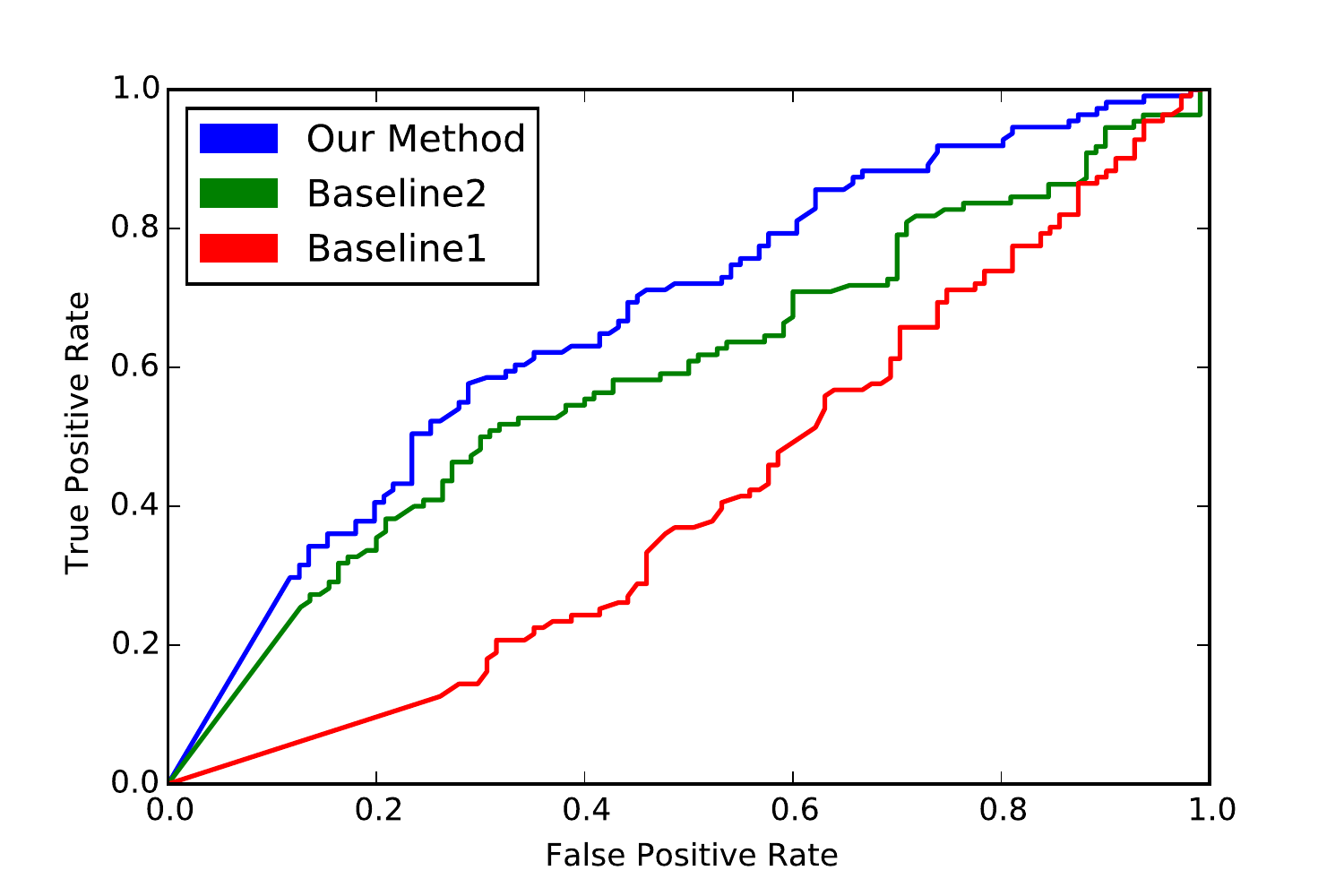}
      \vspace{-3mm}
  \caption{\small AUC for Twitter dataset on 116 important real world events.}
  \label{twitterAUC}                 
  \vspace{-5mm}
\end{figure}

\subsection{Twitter Dataset} 
\vspace{-1mm}
For Twitter dataset we focus on the star network topology. We create a dataset for famous people users and randomly select 30 of their followers among the tens of thousands followers. We assume there is a star-shaped network from the celebrity to the followers, and collect all their re/tweets in late January and early February 2016. Fig. \ref{fig:real-twitter}-(a) demonstrates the statistics computed for the account associated to a TV series named Mr. Robot. We identify that the statistics increase around late January 10-th and early 11-th. This, surprisingly corresponds to the winning of the 2016 Golden Glob Award\footnote{\url{http://www.tvguide.com/news/golden-globe-awards-winners-2016/}}.  Fig. \ref{fig:real-twitter}-(b) shows the statistics computed based on the events of the First lady of the USA  and 30 of her randomly selected followers. The statistics reveal a sudden increase in 13th of January. We find a related event - Michelle Obama stole the show during the president's final State of the Union address by wearing a marigold dress which sold out even before the president finished the speech\footnote{\url{http://www.cnn.com/2016/01/13/living/michelle-obama-dress-marigold-narciso-rodriguez-feat/}}. %This event obviously changed the twittering behaviors of the community consisting of the first lady and a couple her followers. 
Fig. \ref{fig:real-twitter}-(c) is related to Suresh Raina, an Indian professional cricketer. We selecte a small social circle around him as the center of a star-shaped network. We notice that he led his team to win an important game on Jan. 20\footnote{\url{http://www.espncricinfo.com/syed-mushtaq-ali-trophy-2015-16/content/story/963891.html}}, which corresponds to a sharp increase of the statistics. More results for this dataset can be found in Appendix \ref{app:more_eg}.

We further perform sensitivity analysis using the twitter data. We identify 116 important real life events. Some typical examples of such events are release of a movie/album, winning an award, Pulse Nightclub shooting, etc. Next, we identify the twitter handles associated with entities representing these events. %For eg EGYPTAIR which represents the twitter handle of Egyptair airlines. 
We randomly sample 50 followers from each of these accounts and obtain a star topology graph centered around each handle. We collect tweets of all users in all these networks for a window of time before and after the real life event. For each network we compute the statistics. The AUC curves in Fig. \ref{twitterAUC} are obtained by varying the threshold. A threshold value is said to correctly identify the true change-point if the statistic value to the right of the change-point is greater than the threshold. This demonstrates the good performance of our algorithm against two baseline algorithms. %, shown in green and red in the figure. Note that, the lower left portion of the curves is a straight line. This is because after normalization all the peak statistic values are 1 and the rest of the values are significantly lesser than 1. This results in a non-smooth transition of points while sweeping the threshold line across the y-axis.

\begin{figure}[h!]
  \centering
  \vspace{-3mm}
  \includegraphics[width=.6\linewidth]%{methodgraph}
    {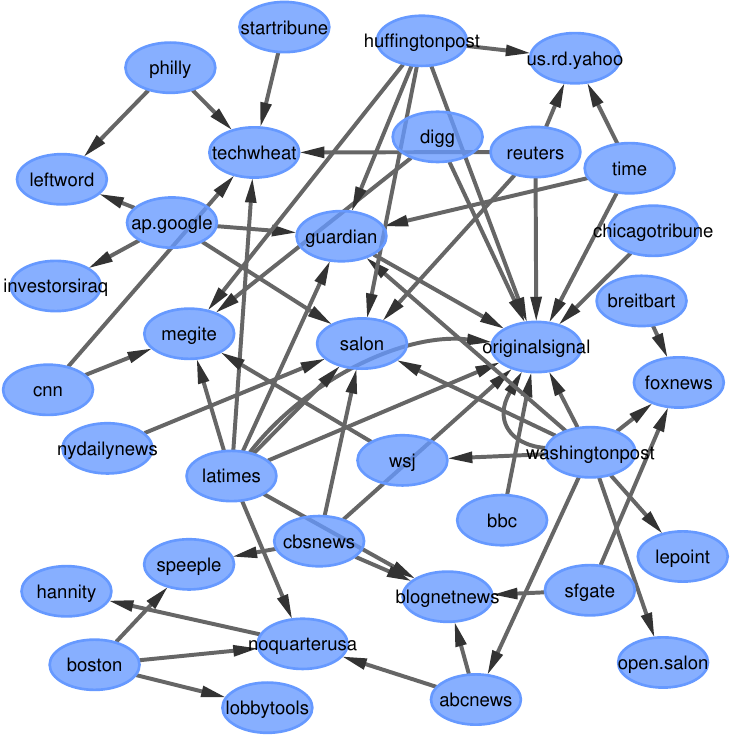}
      \vspace{-3mm}
  \caption{\small Illustration of the network topology for tracking Obama's first presidency announcement.}
  \label{network}                 
  \vspace{-6mm}
\end{figure}

\begin{figure*}[t!]
\begin{center}
\begin{adjustbox}{width=1.0\textwidth}
\begin{tabular}{ccc}
\includegraphics[width=0.3\textwidth]{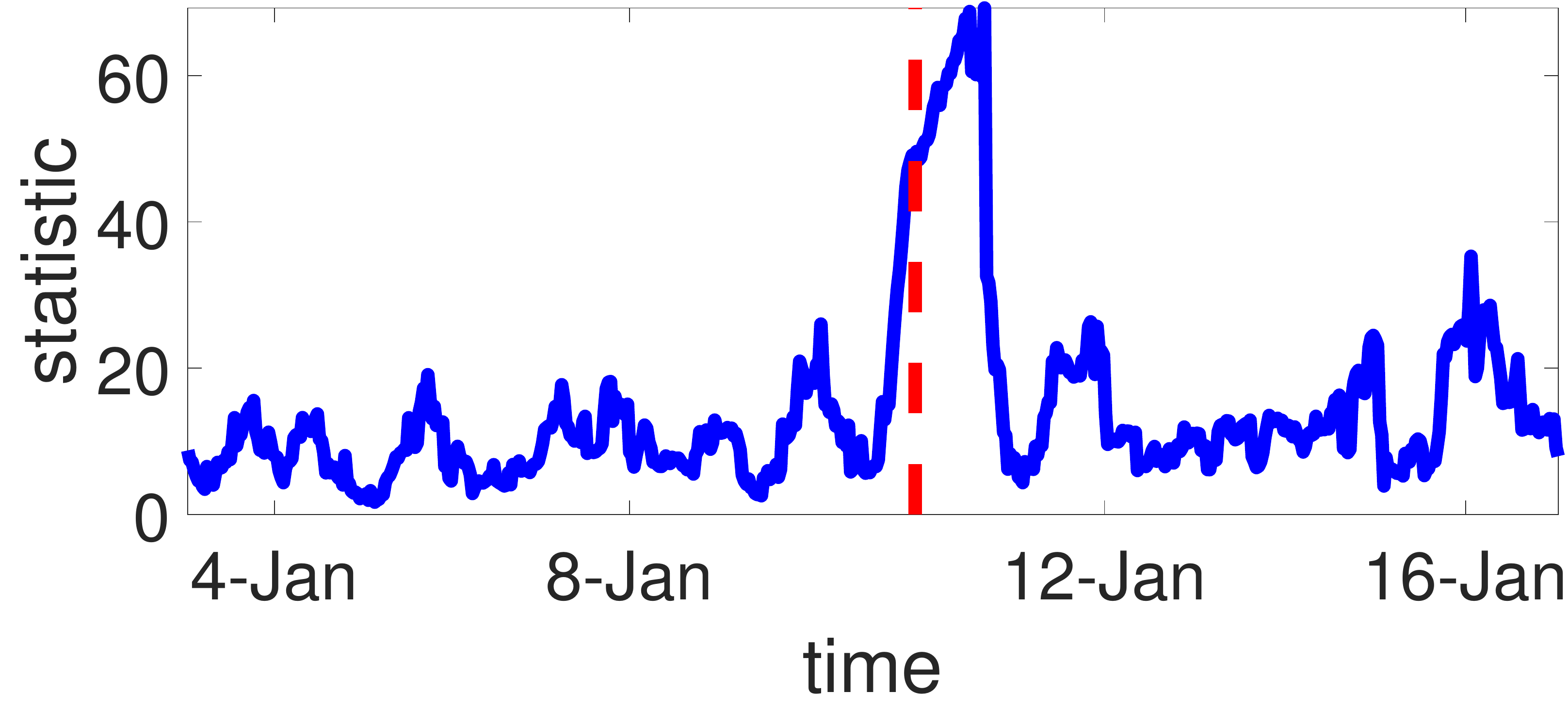} &
\includegraphics[width=0.3\textwidth]{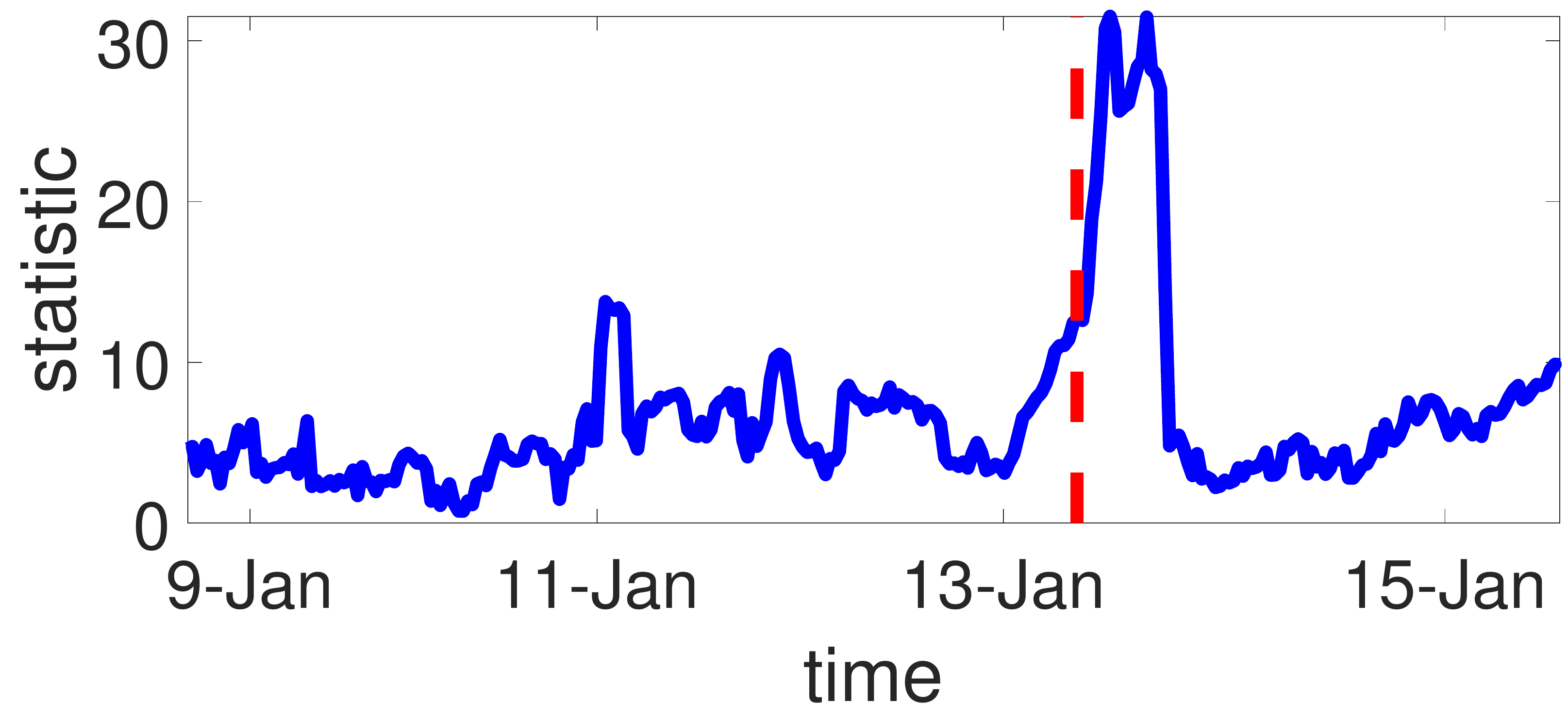} &
\includegraphics[width=0.3\textwidth]{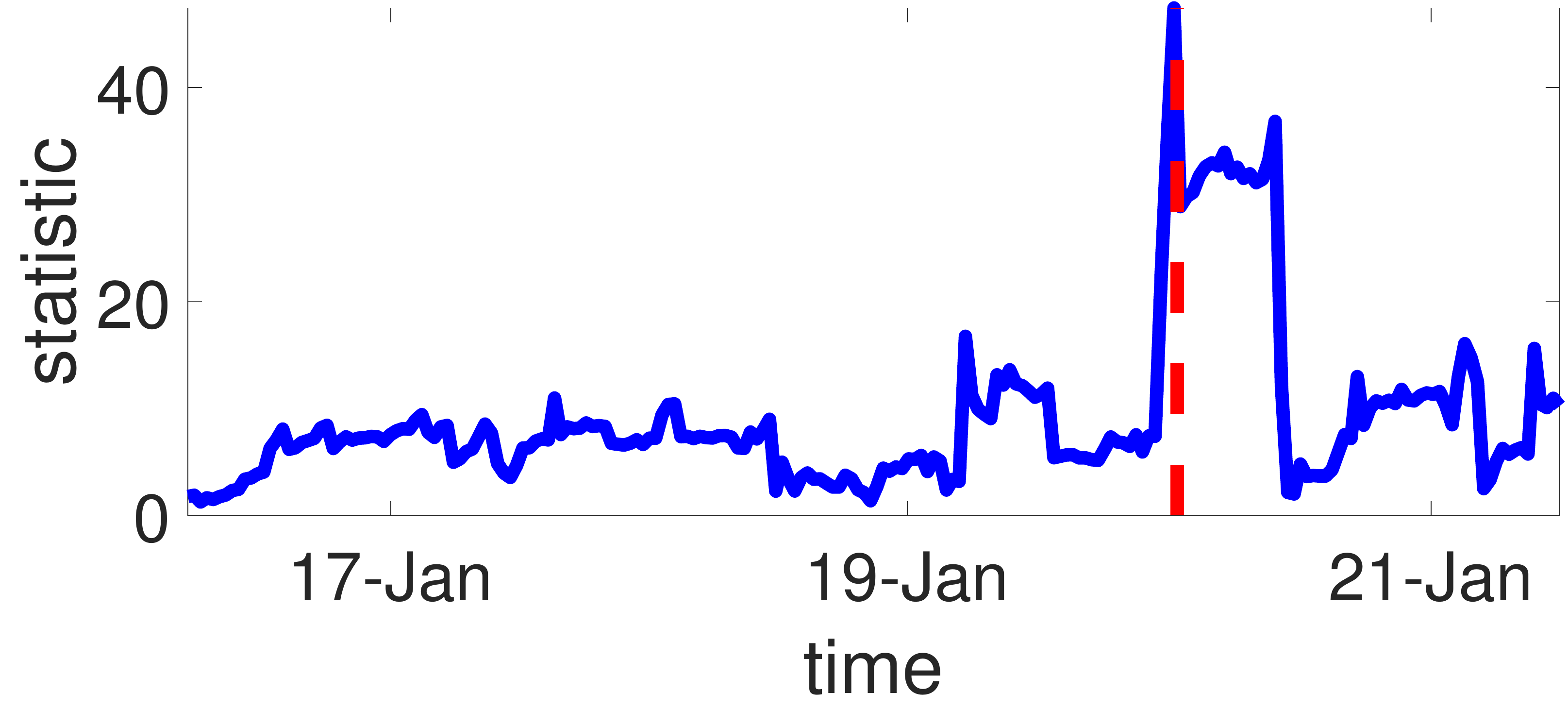} 
\end{tabular}
\end{adjustbox}
\end{center}
\vspace{-3mm}
\caption{\small Exploratory results on Twitter for the detected change points: (left) Mr Robot wins the Golden Globe; (middle) First Lady's dress getting attention; (right) Suresh Raina makes his team won.}
\label{fig:real-twitter}
\vspace{-3mm}
\end{figure*}

\begin{figure*}[h!]
\centering
%\begin{adjustbox}{width=1.0\textwidth}
\begin{tabular}{ccc}
\includegraphics[width=0.31\linewidth]{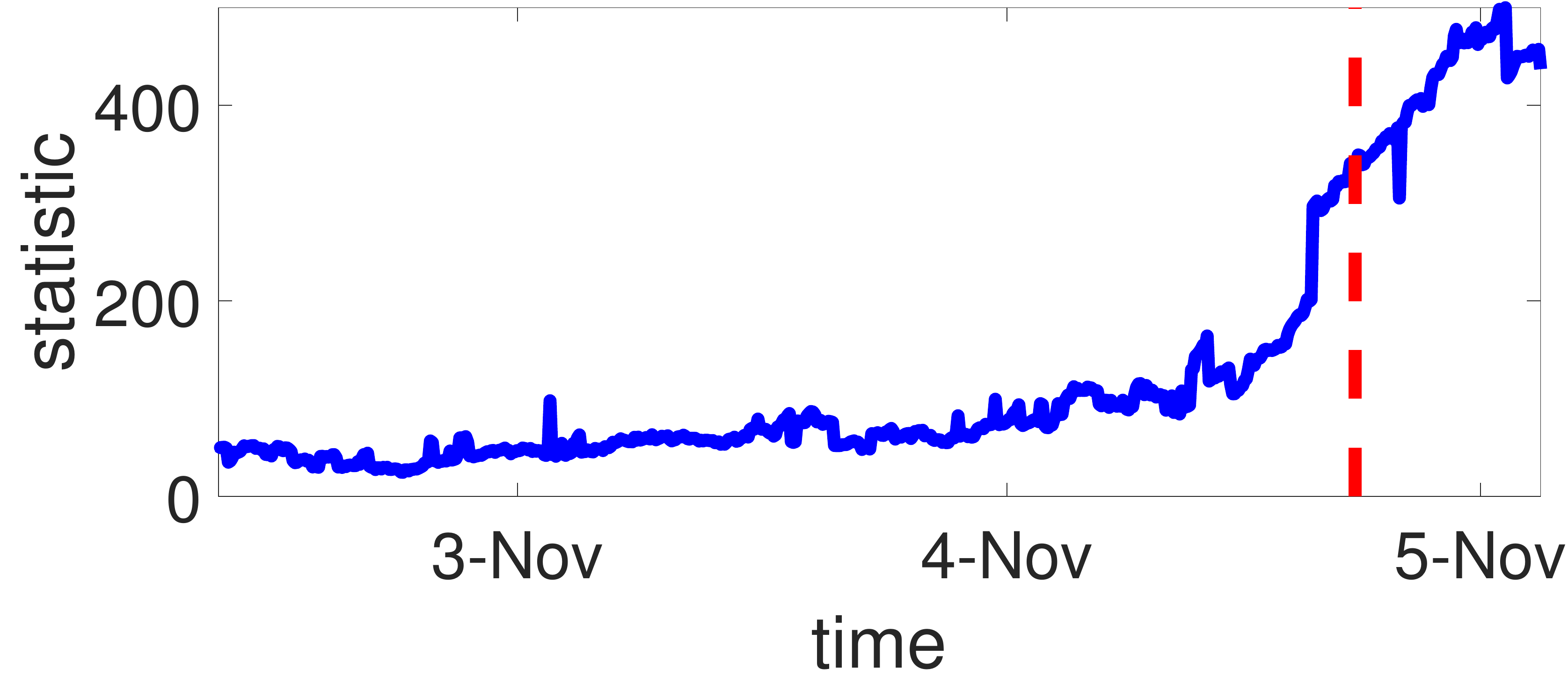} &
 \includegraphics[width=0.31\linewidth]{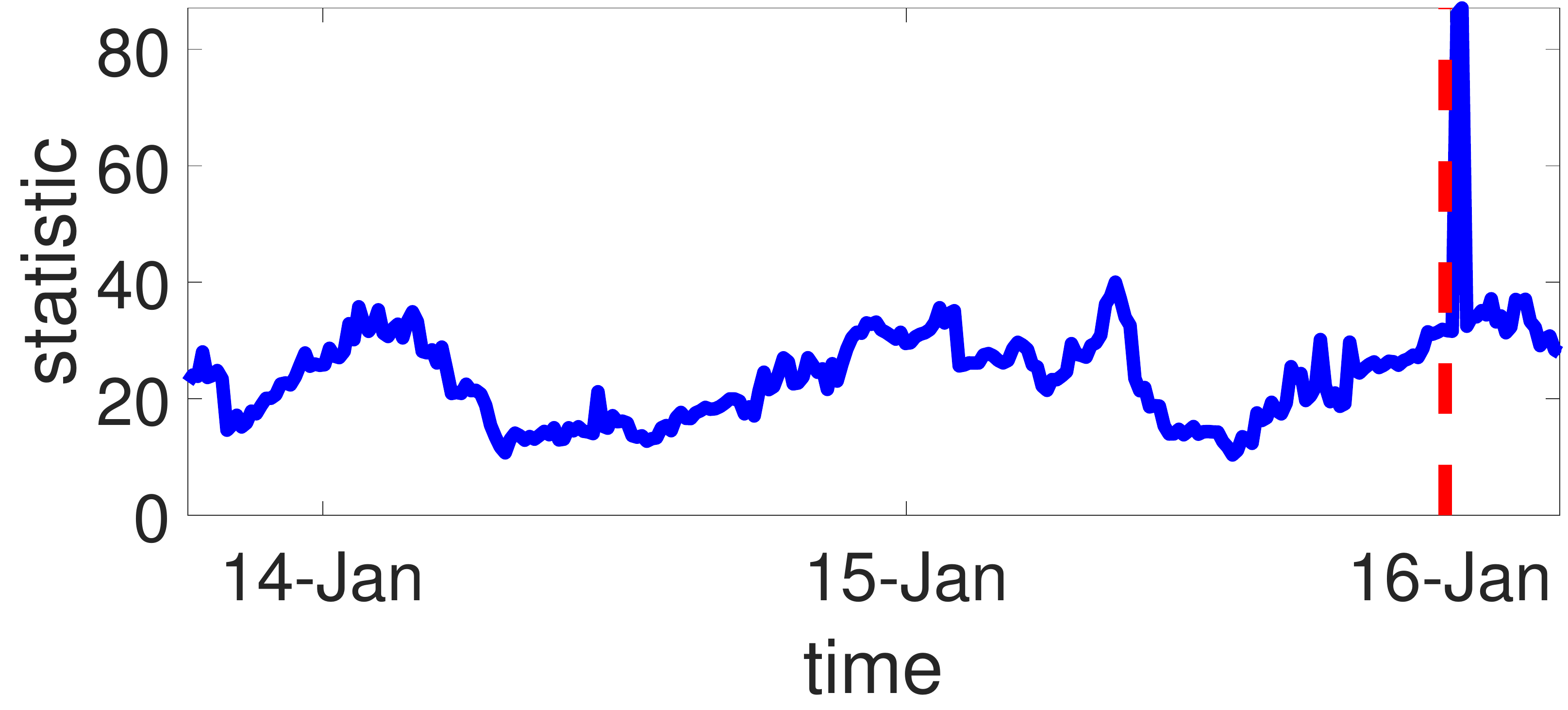}&
 \includegraphics[width=0.31\linewidth]{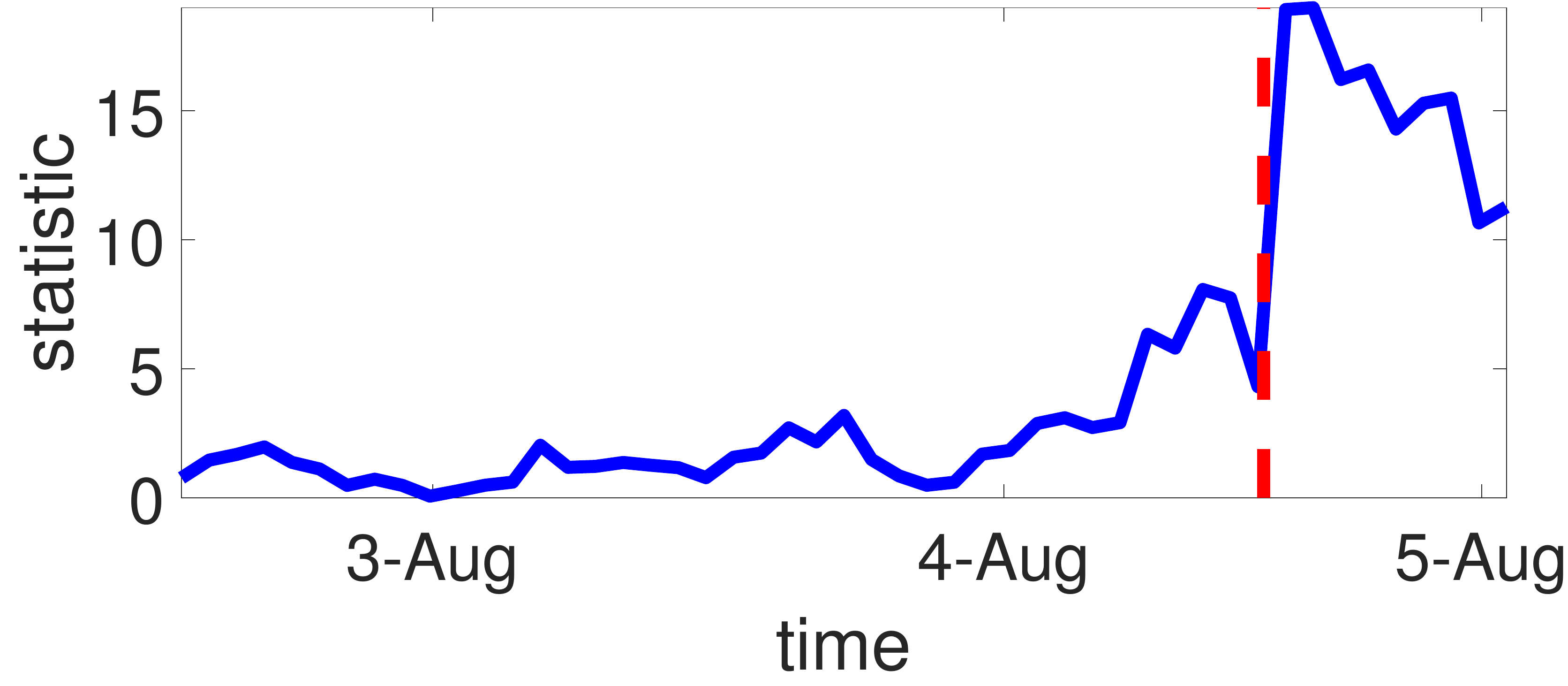}
 \end{tabular}
%\end{adjustbox}
\vspace{-3mm}
\caption{\small Exploratory results on Memetracker for the detected change points: (left) Obama wins the presidential election; (middle) Israel announces ceasefire; (right) Beijing Olympics starts.}
\label{fig:real-memetracker}
\vspace{-3mm}
\end{figure*}

\subsection{Memetracker Dataset}

As a further illustration of our method, we also experiment with the Memetracker\footnote{\url{http://www.memetracker.org/}} dataset to detect changes in new blogs.
The dataset contains the information  flows captured by hyperlinks between different sites with timestamps during nine months. It tracks short units of texts and short phrases, which are called memes and act as signatures of topics and events propagation and diffuse over the web in mainstream media and blogs~\cite{leskovec2009meme}.
The dataset has been previously used in Hawkes process models of social activity~\cite{farajtabar2016multistage,zhou2013learning}. 

We create three instances of change-point detection scenarios from the Memetracker dataset using the following common procedure. First, we identify a key word associated with a piece of news occurred at $\kappa$. Second, we identify the top $n$ websites which have the most mentions of the selected key word in a time window $[t_{\min}, t_{\max}]$ around the news break time $\kappa$ (i.e., $\kappa \in [t_{\min}, t_{\max}]$). Third, we extract all articles with time stamps within $[t_{\min}, t_{\max}]$ containing the keyword, and each article is treated as an event in the point process. Fourth, we construct the directed edges between the websites based on the reported linking structure. These instances correspond to real world news whose occurrences are unexpected or uncertain, and hence can cause abrupt behavior changes of the blogs. The details of these instances are showed in table~\ref{tb:memetracker}. 

\begin{table}[h]
\vspace{-0.1in}
\centering
  \caption{Summary information for the extracted instance for change point detection from Memetracker dataset. The keywords are highlighted in red.} 
  \label{tb:memetracker}
  \setlength{\tabcolsep}{0.03in}
  \vspace{-2mm}
  \begin{tabular}{|c|c|c|c|c|}
      \hline
      real world news & $n$ & $\kappa$ & $t_{\min}$ & $t_{\max}$ \\
      \hline
      {\color{red}Obama} elected president & 80 & 11/04/08 & 11/02/08 & 11/05/08 \\
      Ceasefire in {\color{red}Israel} & 60 & 01/17/09 & 01/13/09 & 01/17/09 \\
      {\color{red}Olympics} in Beijing & 100 & 08/05/08 &  08/02/08 & 08/05/08 \\
      \hline
  \end{tabular}
  \vspace{-2mm}
\end{table}

The first piece of news corresponds to ``Barack Obama was elected as the 44th president of the United States\footnote{\url{https://en.wikipedia.org/wiki/United\_States\_presidential\_election,\_2008}}''. In this example, we also plot the largest connected component of the network as shown in Fig.~\ref{network}. It is notable that this subset includes the credible news agencies such as BBC, CNN, WSJ, Hufftingtonpost, Guardian, etc. As we show in Fig.~\ref{fig:real-memetracker}-(a), our algorithm can successfully pinpoint a change right at the time that Obama was elected. The second piece of news corresponds to ``the ceasefire in Israel-Palestine conflict back in 2009''. Our algorithm detects a sharp change in the data, which is aligned closely to the time right before the peak of the war and one day before the Israel announces a unilateral ceasefire in the Gaza War back in 2009\footnote{\url{http://news.bbc.co.uk/2/hi/middle\_east/7835794.stm}}. The third piece of news corresponds to ``the summer Olympics game in Beijing''. 
Fig.~\ref{fig:real-memetracker}-(c) shows the evolution of our statistics. The change-point detected is 2-3 days before the opening ceremony where all the news websites started to talk about the event\footnote{\url{https://en.wikipedia.org/wiki/2008\_Summer\_Olympics}}.

\vspace{-0.21in}

\section{Summary}\label{sec:summary}

In this paper, we have studied a set of likelihood ratio statistics for detecting change in a sequence of event data over networks. To the best of our knowledge, our work is the first to study change-point detection for network Hawkes process. We adopted the network Hawkes process for the event streams to model self- and mutual- excitation between nodes in the network, and cast the problem in sequential change-point detection frame, and derive the likelihood ratios under several models. We have also presented an EM-like algorithm, which can efficiently compute the likelihood ratio statistic online. The distributed nature of the algorithm enables it to be implemented on larger networks. 
Highly accurate theoretical approximations for the false-alarm-rate, i.e., the average-run-length (ARL) for our algorithms are derived. 
%We have demonstrated numerically that these approximations are highly accurate, which are quite useful to determine a detection threshold analytically to avoid the otherwise onerous numerical simulations.  
%
We demonstrated the performance gain of our algorithms relative to two baselines, which represent the current main approaches to this problem. Finally, we also tested the performance of the proposed method on synthetic and real data.

% Acknowledgements should go at the end, before appendices and references

%\section*{Acknowledgement}
%{\small 
%This research was supported in part by CMMI-1538746 and CCF-1442635 to Y.X.; NSF/NIH BIGDATA
%1R01GM108341, ONR N00014-15-1-2340, NSF IIS-1218749, NSF IIS-1639792, NSF CAREER IIS-1350983, grant from Intel and NVIDIA to L.S..}

\vspace{-4mm}

\bibliographystyle{ieee}
\bibliography{sample,yao_research_statement}

% Manual newpage inserted to improve layout of sample file - not
% needed in general before appendices/bibliography.
% \bibliography{sigproc}  
% \newpage

\clearpage
\appendices

\section{Proofs of Theorem 1}\label{ARL_proof}

We show the one dimensional case as an example. The following informal derivation justifies the theorem. Let $t$ be the current time, and let the window-length be $L$. Recall our notations: $\mathbb{P}$ and $\mathbb{E}$ denote the probability measure and the expectation under the null hypothesis; $\mathbb{P}_{t, \tau, \alpha}$ and $\mathbb{E}_{t, \tau, \alpha}$ denote the probability measure and the expectation under the alternative hypothesis. We also use the notation use $\mathbb{E} [U; A] =  \mathbb{E}[U\mathbb{I}\{A\}]$ to denote conditional expectation.

First, to evaluate ARL, we study the probability that the detection statistic exceeds the threshold before a given time $m$.  We will use the change-of-measure technique \cite{yakir2013extremes}. Under the null hypothesis, the boundary crossing probability can be written as
\begin{align}
& \mathbb{P} \left[ \sup_{t < m, \alpha \in \Theta}  \ell_{t,\tau, \alpha} >x\right] =\mathbb{E}\left[ 1;  \sup_{ t < m, \alpha \in \Theta}  \ell_{t,\tau, \alpha} >x  \right] \nonumber\\
= &~ \mathbb{E}\left[ \underbrace{  \frac{\int_{t}  \int_{\alpha \in \Theta }e^{ \ell_{t,\tau, \alpha}} dt d\alpha}{ \int_{ t'} \int_{\alpha' \in \Theta } e^{\ell_{t',\tau', \alpha'}} dt' d\alpha'}   }_{=1}; \sup_{t <m, \alpha \in \Theta}  \ell_{t,\tau, \alpha} >x \right] \nonumber\\
= & ~ \int_{t} \int_{\alpha \in \Theta}\mathbb{E}\left[ \frac{e^{ \ell_{t,\tau, \alpha }}}{\int_{t'} \int_{\alpha' \in \Theta }e^{\ell_{t',\tau', \alpha'}} dt'd\alpha'  }; \right.\nonumber\\
& \left.  \qquad \sup_{t < m, \alpha \in \Theta}  \ell_{t,\tau, \alpha} >x  \right]  dt d\alpha\label{last_eqn}
\end{align}
where the last equality follows from changing the order of summation and the expectation. Using change-of-measure
$
d \, \mathbb{P}= e^{ -\ell_{t,\tau, \alpha}} d\,\mathbb{P}_{t, \tau, \alpha},
$, the last equation (\ref{last_eqn}) can be written as 
\begin{equation}
\begin{split}
 &  \int_{t}  \int_{\alpha \in \Theta}\mathbb{E}_{t, \tau, \alpha}\left[ \frac{1}{\int_{t'} \int_{\alpha' \in \Theta} e^{\ell_{t',\tau', \alpha'}}   dt'd\alpha' };  \right.\nonumber\\
& \left.  \qquad \sup_{t < m, \alpha}  \ell_{t<m,\tau, \alpha} >x \right] dt d\alpha 
\end{split}
\label{app_eqn1}
\end{equation}
After rearranging each term and introducing additional notations, the last equation above (\ref{app_eqn1}) can be written as
\begin{equation}  \label{changeofmeasure}
\begin{split}
 e^{-x} \int_{t} \int_{\alpha \in \Theta}\mathbb{E}_{t,\tau, \alpha}
\left[ \frac{\mathcal{M}_{t,\tau, \alpha} }{\mathcal{S}_{t,\tau, \alpha}} e^{-[\tilde{l}_{t,\tau, \alpha} +m_{t,\tau, \alpha}]}; \right.\\
\left.\quad ~~~ \tilde{l}_{t,\tau, \alpha} +M_{t,\tau, \alpha} > 0\right]  dt d\alpha
\end{split}
\end{equation}
where 
\begin{align*}
& \mathcal{M}_{t,\tau, \alpha } = \sup_{t'} e^{   \ell_{t', \tau', \alpha} -\ell_{t, \tau, \alpha} }, \\ &\mathcal{S}_{t,\tau, \alpha } = \int_{t'}  e^{\ell_{t', \tau', \alpha} -\ell_{t,\tau, \alpha}} dt',  \\
& \tilde{l}_{t,\tau, \alpha} =  \ell_{t,\tau, \alpha} -x, \quad M_{t,\tau, \alpha} =  \mbox{log} \mathcal{M}_{t,\tau, \alpha}.
\end{align*}
The final expression is also based on the following approximation. When the interval slightly changes from $(\tau', t')$ to $(\tau, t)$, $\alpha'$ changes little under the null hypothesis since $\alpha'$ is estimated from data stored in $(\tau', t')$. Therefore, in the small neighborhood of $(\tau', t')$, we may regard $\alpha$ as a constant. This leads to an approximation:
\begin{align}
\frac{\sup_{t', \alpha'} e^{   \ell_{t', \tau', \alpha'} -\ell_{t, \tau, \alpha} } }{\int_{t'}\int_{\alpha'}  e^{   \ell_{t', \tau', \alpha'} -\ell_{t, \tau, \alpha} } dt' d\alpha' }
\approx \frac{\sup_{t'} e^{   \ell_{t', \tau', \alpha} -\ell_{t, \tau, \alpha} } }{\int_{t'} e^{   \ell_{t', \tau', \alpha} -\ell_{t, \tau, \alpha} } dt' }.
\end{align}

The representation (\ref{changeofmeasure}) consists of a large deviation exponential decay, given by $e^{-x}$, and lower order contribution that reside in the expectation. The random variables in expectation are further dissected into random variables that are influenced mainly by local perturbations and the random variable that captures the main part of the variability. We can show that the random variable $\tilde{l}_{t,\tau, \alpha}$, which is referred to as the ``global term'',  has an expectation  $(t-\tau)I-x$ under the alternative, and a variance $(t-\tau) \sigma^2$. The other random variables are $\mathcal{M}_{t, \tau, \alpha}$ and $\mathcal{S}_{t, \tau, \alpha}$ and its log $m_{t, \tau, \alpha}$, which are determined by the so-called ``local field'' $\{ \ell_{t', \tau', \alpha}-\ell_{t, \tau, \alpha}\}$ are parameterized by $t'$ when we fix $t-\tau$.

Define $\widehat{\mathcal{M}}_{t, \tau, \alpha}$ and $\widehat{\mathcal{S}}_{t, \tau, \alpha}$ by restricting the integral and maximization only to the range of parameter values that are at most $\epsilon$ away from either $\tau$ or $t$. 
By localization theorem (Theorem 5.2 in \cite{yakir2013extremes}), under certain conditions, the local and global components are asymptotically independent, which informs:
\begin{equation}
\begin{split}
&\mathbb{E}_{t,\tau, \alpha}
\left[ \frac{\mathcal{M}_{t,\tau, \alpha} }{\mathcal{S}_{t,\tau, \alpha}} e^{-[\tilde{l}_{t,\tau, \alpha} +m_{t,\tau, \alpha}]} ; \tilde{l}_{t,\tau, \alpha} +M_{t,\tau, \alpha} > 0\right] \\
&\approx \mathbb{E}_{t, \tau, \alpha}
\left[ \frac{\widehat{\mathcal{M}}_{t, \tau, \alpha}}{\widehat{\mathcal{S}}_{t, \tau, \alpha}} \right]  \frac{1}{\sqrt{(t-\tau)\sigma^2}}\phi\left( \frac{(t-\tau)I-x}{\sqrt{(t-\tau)\sigma^2}} \right).
\end{split}
\end{equation}

We can further prove (see Appendix \ref{localfield}) that the expected local rate $\mathbb{E}_{t, \tau, \alpha}
\left[\mathcal{M}_{t, \tau, \alpha}/\mathcal{S}_{t, \tau, \alpha} \right] $ only depends on $\alpha$ and is independent of $t$:
\begin{equation} \label{ratio}
 \mathbb{E}_{t, \tau, \alpha}
\left[ \frac{\widehat{\mathcal{M}}_{t, \tau, \alpha}}{\widehat{\mathcal{S}}_{t, \tau, \alpha}} \right] \approx \nu \left( \frac{2 \xi}{\eta^2}\right),
\end{equation}
for $\xi$ and $\eta^2$ defined in (\ref{def_xi_eta}). The conditions for which these approximations hold are given on Page 56 of \cite{yakir2013extremes}, and in particular, we need to compute the local rate, which is done in Appendix \ref{localfield}.

Hence, the probability in (\ref{app_eqn1}) should be
\begin{equation}\label{tailprob}
\begin{split}
%&\mathbb{P} \left[ \sup_{t<m, \alpha \in (0, 1)}  \ell_{t,\tau, \alpha} >x \right] \\
 & \hspace{-0.1in}   e^{-x} \int_{t}\int_{\alpha \in (0, 1)}   \nu \left( \frac{2 \xi}{\eta^2}\right)  \frac{1}{\sqrt{(t-\tau)\sigma^2}}\phi\left( \frac{(t-\tau)I-x}{\sqrt{(t-\tau)\sigma^2}} \right)     d\alpha dt\\
& \hspace{-0.1in} \approx m   e^{-x} \int_{\alpha \in (0, 1)}   \nu \left( \frac{2 \xi}{\eta^2}\right)  \frac{1}{\sqrt{(t-\tau)\sigma^2}}\phi\left( \frac{(t-\tau)I-x}{\sqrt{(t-\tau)\sigma^2}} \right)   d\alpha. %}_{C}.
\end{split}
\end{equation}
Define $C$ to be the factor that multiplies $m$ in the equation above.

Next, since
\[
\mathbb{P} \left[ \sup_{t<m, \alpha \in (0, 1)}  \ell_{t,\tau, \alpha} >x \right] =
\mathbb{P} \left[ T<m \right],
\]
we can relate (\ref{tailprob}) to the ARL $\mathbb{E}[T]$. Note that we can write the tail probability (\ref{tailprob}) in a form $\mathbb{P} \left[ T<m \right]= mC[1 + o(1)]$. When $x\rightarrow \infty$, from the arguments in \cite{SiegmundVenkatraman1995,SiegmundYakir2008b}, we see that the stopping time $T$ is asymptotically exponentially distributed and $\mathbb{P}[T <m] \to 1- \mbox{exp}(-Cm)$. As a result, $\mathbb{E}[T] \sim C^{-1}$, which is equivalent to (\ref{ARL_expr}). Derivations for $I$, $\sigma^2$, $\xi$ and $\eta^2$ will be talked about in Appendix \ref{KL}.

\section{First- and second-order statistics of Hawkes processes}

We first to characterize the first- and second-order statistics for Hawkes processes, which are useful for  evaluating $I$, $\sigma^2$, $\xi$ and $\eta^2$. For the defined one-dimensional Hawkes processes and multi-dimensional Hawkes processes, if we choose kernel function $\varphi(t)$ with $\int \varphi(t) dt =1$, we will have the following two lemmas that are derived from the results in \cite{hawkes1971spectra}. \cite{bacry2014second}:
\begin{lemma}[First-order statistics for Hawkes processes]\label{firstorder}
If the influence parameters satisfy $\alpha \in (0,1)$ (one-dimension) or the spectral norm $\rho( \bm{A} )<1$ (high-dimension), then the Hawkes processes are asymptotically stationary and with stationary intensity $m_t= \mathbb{E}_{\mathcal{H}_{t} } [\bm{\lambda}_t]$. We further have that for the one-dimensional case
\[\bar{\lambda}:=\lim_{t \to \infty} m_t = \frac{\mu}{1-\alpha} \] and for the multi-dimensional case
\[\bar{\bm{\lambda}} :=\lim_{t \to \infty} \bm{m}_t = (\bm{I}-\bm{A})^{-1} \bm{\mu}.\]\end{lemma}
\begin{lemma}[Second-order statistics for Hawkes processes]\label{secondorder}
For stationary Hawkes processes, the covariance intensity, which is defined as: 
 \begin{align}
\bm{c} (t'-t) = \mbox{Cov} \left[  \bm{ \lambda}_t, \bm{\lambda}_{t'}  \right]=   \frac{\mbox{Cov} \left[ d\bm{N}_t, d\bm{N}_{t'} \right] }{dtdt'}
\end{align}
will only depend on $t'-t$. Then for one-dimensional Hawkes processes, we have:
\begin{align}
c(\tau) = \begin{cases}
\frac{\alpha \beta (2-\alpha) \mu}{2(1-\alpha)^2} e^{-\beta(1-\alpha)\tau}, &  \tau >0;\\
\frac{\mu}{1-\alpha} \delta (\tau), & \tau=0; \\
c(-\tau), &  \tau< 0.\\
\end{cases}
\end{align}
for the multi-dimensional Hawkes processes
\begin{align*}
\bm{c}(\bm{\tau}) = \begin{cases}
 \beta e^{-\beta (\bm{I}- \bm{A}) \tau} \bm{A} \left( \bm{I} +\frac{1}{2} (\bm{I}-\bm{A})^{-1}
\bm{A}  \right) \\
\quad \cdot \mbox{diag} \left(  (\bm{I}-\bm{A})^{-1} \bm{\mu}\right), &  \bm{\tau }>0;\\
\mbox{diag} \left(  (\bm{I}-\bm{A})^{-1} \bm{\mu}\right) \delta(\bm{\tau}), & \bm{\tau}=0; \\
\bm{c}(-\bm{\tau})^{\intercal}, & \bm{ \tau} < 0.\\
\end{cases}
\end{align*}

\end{lemma}

%The proof for Lemma \ref{firstorder} and \ref{secondorder} are as follows. 
\begin{proof}[Proof of Lemma \ref{firstorder}]
For multi-dimensional Hawkes processes, by mean field approximation and define $\bm{m}_t = \mathbb{E}_{\mathcal{H}_{t}} [\bm{\lambda}_t]$, we have:
\begin{align}
\bm{m}_t =  \bm{\mu} + \bm{A} \int_{-\infty}^{t}\varphi (t-s) \bm{m}_s ds
\end{align}
which can be written as
\begin{align} \label{meanintensity}
\bm{m}_t = \left( \bm{I} + \sum_{n=1}^{\infty} \bm{A}^n \int_{-\infty}^t \varphi^{(\star n)}(s) ds \right) \bm{\mu}.
\end{align}
where $\star$ denotes the convolution operation, and $\varphi^{(\star n)}$ denote the $n$-fold convolution. 
Let
$\Psi (t) = \bm{A} \varphi(t)+ \bm{A} ^2 \varphi(t)\star\varphi(t)+\bm{A} ^3\varphi(t)\star\varphi(t)\star\varphi(t)+\cdots=\sum_{n=1}^{\infty} \bm{A}^{n}\varphi^{(\star n)}(t).
$
And we can write (\ref{meanintensity}) as:
\[
\bm{m}_t = \left( \bm{I} + \int_{-\infty}^t \Psi (s)  ds \right) \bm{\mu}.
\]
Given a function $f(t)$, we denote its Laplace transform $\mathcal{L}(\cdot)$ as:
\[
\widehat{f}(z) =\mathcal{L} (f(t))=\int_{-\infty}^{\infty} f(t)e^{-zt}dt.
\]
Next, apply Laplace transform to both sides of equation (\ref{meanintensity}). Clearly
\[
\widehat{\bm{m}}(z)=\frac{1}{z}(\bm{I}-\frac{ \beta}{z+\beta}\bm{A})^{-1} \bm{\lambda}_0,
\]
where
\[
\widehat{\Psi}(z) =\sum_{n=1}^{\infty} \left( \frac{\beta}{z+\beta} \right)^n\cdot \bm{A}^n
=(\bm{I}-\frac{ \beta}{z+\beta}\bm{A})^{-1}-\bm{I}.
\]
By the property of Laplace transformation,
\begin{align}
\bar{ \bm{\lambda} }:=\lim_{t \to \infty} \bm {m}_t= \lim_{z\to 0} z\widehat{\bm{m}}(z)=\left( \bm{I}-\bm{A}\right)^{-1} \bm{\mu}.
\end{align}
For a special case where $d=1$, we have $\bar{ \lambda} = \mu/(1-\alpha)$.
\end{proof}
\begin{proof}[Proof for Lemma \ref{secondorder}]
For $\bm{\tau}>0$, we have:
\begin{equation*}
\begin{split}
\bm{c}(\bm{\tau})
&=\frac{\mathbb{E} \left[ d\bm{N}_{t+\tau} d\bm{N}_{t}^{\intercal} \right]}{(dt)^2}-\bar{ \bm{\lambda}} \bar{ \bm{\lambda}}^{ \intercal} = \mathbb{E} \left[ \bm{\lambda}_{t+\tau}  \frac{d \bm{N}_{t}^{\intercal} }{dt}\right]-\bar{ \bm{\lambda}} \bar{ \bm{\lambda}}^{\intercal} \\
& = \mathbb{E} \left[ \left( \bm{\mu}+\bm{A}\int_{-\infty}^{t+\tau} \varphi (t+\tau-s) d\bm{N}_s \right) \frac{ d\bm{N}_{t}^{\intercal} }{dt}\right]-\bar{ \bm{\lambda}} \bar{ \bm{\lambda}}^{\intercal} \\
& = \bm{A} \int_{-\infty}^{\tau} \varphi(\tau-s)\bm{c}(s)ds \\
& = \bm{A}\varphi(\tau)\mbox{diag}(\bar{ \bm{ \lambda}}) +\bm{A}\int_{-\infty}^{\tau} \varphi(\tau-s) \bm{c}(s) ds \\
& = \bm{A}\varphi(\tau)\mbox{diag}(\bar{ \bm{\lambda}}) +\bm{A}\int_{0}^{\infty} \varphi(\tau+s) \bm{c}(s) ds \\
&~~+\bm{A} \int_0^{\tau} \varphi(\tau-s) \bm{c}(s) ds.
\end{split}
\end{equation*}
For the last two equalities, we are using the relation, $\bm{c}(-\tau)=\bm{c}(\tau)^{\intercal}$ and the fact that when $\tau = 0$ 
$
\bm{c} (\tau)=\mbox{diag}(\bar{ \bm{\lambda}}) \delta(\tau).
$
Note that for Poisson processes, we have $\bm{c}(\tau)=\mbox{diag}(\bm{\lambda}) \delta(\tau)$.
Now substituting $\varphi(\tau)=\beta e^{-\beta \tau}$ into the above, we have:
\begin{equation} \label{cov}
\begin{split}
&\bm{c}(\tau) = \bm{A}\beta e^{-\beta \tau}\mbox{diag}(\bar{ \bm{\lambda}}) +\bm{A}\int_{0}^{\infty} \beta e^{-\beta ( \tau+s) }\bm{c}(s) ds \\
&~~+\bm{A} \int_0^{\tau} \beta e^{-\beta (\tau-s)} \bm{c}(s) ds.
\end{split}
\end{equation}
Applying Laplace transform to both sides of (\ref{cov}), we obtain
\[
\widehat{\bm{c} }(z)=\frac{\beta}{z+\beta} \bm{A} \mbox{diag}(\bar{ \bm{\lambda}}) + \frac{\beta}{z+\beta} \bm{A}  \widehat{\bm{c} }(\beta) +\frac{\beta}{z+\beta} \bm{A}  \widehat{\bm{c} }(z),
\]
where 
\begin{equation}
\begin{split}
 &\mathcal{L} \left(  \int_{0}^{\infty} \beta e^{-\beta ( \tau+s) }\bm{c}(s) ds \right) 
=  \mathcal{L} \left(  \beta e^{-\beta \tau} \int_{0}^{\infty}  e^{-\beta s}\bm{c}(s) ds \right) \\
& = \mathcal{L} \left( \beta e^{-\beta \tau}  \widehat{\bm{c} }(\beta) \right)
 =  \frac{\beta}{z+\beta} \widehat{\bm{c} }(\beta).
\end{split}
\end{equation}
Replacing $z$ with $\beta$, we obtain 
\[
\widehat{\bm{c} }(\beta)=\frac{1}{2} (\bm{I}-\bm{A})^{-1} \bm{A}\mbox{diag}(\bar{ \bm{\lambda}}).
\]
Therefore, 
\begin{equation*}
\begin{split}
& \widehat{\bm{c}}(z) =\left( (z+\beta) \bm{I}-\beta \bm{A} \right)^{-1} \beta \bm{A} \left( \bm{I} +\frac{1}{2} (\bm{I}-\bm{A})^{-1}
\bm{A}  \right) \\
&~~\cdot\mbox{diag} \left(  (\bm{I}-\bm{A})^{-1} \bm{\mu}\right)
\end{split}
\end{equation*}
Using inverse Laplace transform for $\widehat{\bm{c}}(z)$, we obtain
\begin{equation*}
\begin{split}
&\bm{c}(\tau)=\mathcal{L}^{-1} \left(  \widehat{\bm{c}}(z)  \right) = \beta e^{-\beta (\bm{I}- \bm{A}) \tau} \bm{A} \left( \bm{I} +\frac{1}{2} (\bm{I}-\bm{A})^{-1}
\bm{A}  \right) \\
&~~\cdot \mbox{diag} \left(  (\bm{I}-\bm{A})^{-1} \bm{\mu}\right), \quad \tau>0.
\end{split}
\end{equation*}
For a special case $d=1$, we obtain:
\[
c(\tau)= \frac{\alpha \beta (2-\alpha) \mu}{2(1-\alpha)^2} e^{-\beta(1-\alpha)\tau}, \quad \tau>0.
\]
\end{proof}

\vspace{-0.3in}
\section{Approximate local rate} \label{localfield}
To show (\ref{ratio}), we need to evaluate the mean and variance of the local field $\{\ell_{t+\epsilon, \tau+\epsilon, \alpha} -\ell_{t,\tau, \alpha}\}$ after change-of-measures.
From (\ref{loglikelihood_poi}) we see the the log-likelihood ratio $\ell_{t, \tau, \alpha}$ is an integration from time $\tau$ to $t$. Thus, we can rewrite $\ell_{t+\epsilon, \tau+\epsilon, \alpha}$ into several parts by dissecting the integration region: 
\begin{align}
\int_{\tau+\epsilon}^{t+\epsilon} = \int_{\tau+\epsilon}^{\tau+\epsilon^{+}} +\int_{\tau+\epsilon^{+}}^{t+\epsilon^{-}} +\int_{t+\epsilon^{-}}^{t+\epsilon}.
\end{align}
From this we the only overlap of data between $\ell_{t+\epsilon, \tau+\epsilon, \alpha}$ and $\ell_{t, \tau, \alpha}$ is the integration over the interval $(\tau+\epsilon^{+}, t+\epsilon^{-})$. Therefore, we have
\begin{equation*}
\begin{split}
&\mathbb{E}_{t,\tau, \alpha} [\ell_{t+\epsilon, \tau+\epsilon, \alpha}] = \mathbb{E}\left[\ell_{t+\epsilon, \tau+\epsilon, \alpha} e^{\ell_{t, \tau, \alpha}}\right] \\
=&~ \mathbb{E}\left[ \ell_{\tau+\epsilon^{+}, \tau+\epsilon, \alpha} e^{\ell_{t, \tau, \alpha}}\right]
+\mathbb{E}\left[ \ell_{t+\epsilon^{-}, \tau+\epsilon^{+}, \alpha} e^{\ell_{t, \tau, \alpha}}\right] \\
 & ~~ +\mathbb{E}\left[ \ell_{t+\epsilon,t+\epsilon^{-}, \alpha} e^{\ell_{t, \tau, \alpha}}\right] \\
= &~\mathbb{E}\left[ \ell_{\tau+\epsilon^{+}, \tau+\epsilon, \alpha}\right]\mathbb{E}\left[ e^{\ell_{t, \tau, \alpha}}\right]
+\mathbb{E}_{t, \tau, \alpha} \left[  \ell_{t+\epsilon^{-}, \tau+\epsilon^{+}} \right]
\\
&~~+\mathbb{E}\left[ \ell_{t+\epsilon,t+\epsilon^{-}, \alpha} \right] \mathbb{E}\left[e^{\ell_{t, \tau, \alpha}}\right].
\end{split}
\end{equation*}
Due to the property of the likelihood ratio, $\mathbb{E}\left[ e^{\ell_{t, \tau, \alpha}}\right] = 1$. 
%Note that if we define the null distribution density is $f(\mathbf{X})$ and the alternative distribution density is $g(\mathbf{X}) $, obviously
%\begin{equation*}
%\begin{split}
%&\mathbb{E}\left[ e^{\ell_{t, \tau, \alpha}}\right]=\int \mbox{exp} \left\{ \mbox{log} \frac{g(\mathbf{X})}{f( \mathbf{X})}  \right\}  f(\mathbf{X}) d\mathbf{X} \\
%&=\int  \frac{g(\mathbf{X})}{f( \mathbf{X})}  f(\mathbf{X}) d\mathbf{X} = \int g(\mathbf{X}) d\mathbf{X} =1.
%\end{split}
%\end{equation*}
Thus, we have:
\begin{equation*}
\begin{split}
&~\mathbb{E}_{t,\tau, \alpha} [\ell_{t+\epsilon, \tau+\epsilon, \alpha}] \\
&= \mathbb{E}\left[ \ell_{\tau+\epsilon^{+}, \tau+\epsilon, \alpha}\right]
+\mathbb{E}_{t, \tau, \alpha} \left[  \ell_{t+\epsilon^{-}, \tau+\epsilon^{+}} \right]
+\mathbb{E}\left[ \ell_{t+\epsilon,t+\epsilon^{-}, \alpha} \right] \\
& = -\epsilon^{-} \frac{\mathbb{E}[\ell_{t, \tau, \alpha}]}{t-\tau} +(t+\epsilon^{-}-\tau-\epsilon^{+}) 
\frac{\mathbb{E}_{t, \tau, \alpha}[\ell_{t, \tau, \alpha}]}{t-\tau} \\
& ~~+\epsilon^{+} \frac{\mathbb{E}[\ell_{t, \tau, \alpha}]}{t-\tau}.
\end{split}
\end{equation*}
For the last equality, we use the fact the both $\mathbb{E}[\ell_{t, \tau, \alpha}]$ and $\mathbb{E}_{t, \tau, \alpha}[\ell_{t, \tau, \alpha}]$ are linear with time interval $(t-\tau)$, which will be proven in Appendix \ref{KL}.
Finally we have:
\begin{equation*}
\begin{split}
&\mathbb{E}_{t, \tau, \alpha}[\ell_{t+\epsilon, \tau+\epsilon, \alpha}-\ell_{t, \tau, \alpha}]\\
&=(-\epsilon^{-}+\epsilon^{+})\frac{\mathbb{E}[\ell_{t, \tau, \alpha}]}{t-\tau}  -(\epsilon^{+}-\epsilon^{-}) \frac{\mathbb{E}_{t, \tau, \alpha}[\ell_{t, \tau, \alpha}]}{t-\tau} \\
&=\underbrace{ \frac{\mathbb{E}[\ell_{t, \tau, \alpha}]-\mathbb{E}_{t, \tau,\alpha}[\ell_{t, \tau, \alpha}]}{t-\tau}}_{-\xi<0} |\epsilon|.
\end{split}
\end{equation*}
By Jensen's inequality, we can prove that $\mathbb{E}[\ell_{t, \tau, \alpha}]<0$ and $\mathbb{E}_{t, \tau, \alpha}[\ell_{t, \tau, \alpha}]>0$. 

\vspace{0.1in}
Similarly, we derive the variance of the local field:
\begin{equation*}
\begin{split}
&\quad \rm{Var}_{t, \tau, \alpha}[\ell_{t+\epsilon, \tau+\epsilon, \alpha}-\ell_{t, \tau, \alpha}] \\
&= \rm{Var}_{t, \tau, \alpha}\left[ \left( \ell_{\tau+\epsilon^{+}, \tau+\epsilon, \alpha} 
+ \ell_{t+\epsilon^{-}, \tau+\epsilon^{+}, \alpha} 
+\ell_{t+\epsilon,t+\epsilon^{-}, \alpha}  \right) -\ell_{t,\tau, \alpha}\right] \\
& = \rm{Var}_{t, \tau, \alpha}\left[ \ell_{\tau+\epsilon^{+}, \tau+\epsilon, \alpha} 
-\left( \ell_{\tau, \tau+\epsilon^{+},\alpha} +\ell_{t+\epsilon^{-}, t, \alpha}  \right)
~~+\ell_{t+\epsilon,t+\epsilon^{-}, \alpha}  \right] \\
&=\rm{Var}_{t, \tau, \alpha}\left[ \ell_{\tau+\epsilon^{+}, \tau+\epsilon, \alpha} \right] \\
&~~+\rm{Var}_{t,\tau, \alpha} \left[ \ell_{\tau, \tau+\epsilon^{+},\alpha} +\ell_{t+\epsilon^{-}, t, \alpha}  \right]
+\rm{Var}_{t,\tau, \alpha} \left[ \ell_{t+\epsilon,t+\epsilon^{-}, \alpha}  \right] \\
&=\rm{Var}\left[ \ell_{\tau+\epsilon^{+}, \tau+\epsilon, \alpha} \right]
+\rm{Var}_{t,\tau, \alpha} \left[ \ell_{\tau, \tau+\epsilon^{+},\alpha}\right. \\
&~~ \left.+\ell_{t+\epsilon^{-}, t, \alpha}  \right]
+\rm{Var} \left[ \ell_{t+\epsilon,t+\epsilon^{-}, \alpha}  \right] \\
&=(\epsilon^{+} -\epsilon^{-}) \frac{\rm{Var}[\ell_{t, \tau, \alpha}]}{t-\tau} +(\epsilon^{+}-\epsilon^{-}) 
\frac{\rm{Var}_{t, \tau, \alpha}[\ell_{t, \tau, \alpha}]}{t-\tau} \\
& =\underbrace{ \frac{\rm{Var}[\ell_{t, \tau, \alpha}]
+\rm{Var}_{t, \tau, \alpha}[\ell_{t, \tau, \alpha}] }  {t-\tau} }_{\eta^2}|\epsilon|.
\end{split}
\end{equation*}
Above, we use the fact that both $\mbox{Var}[\ell_{t, \tau, \alpha}]$ and $\mbox{Var}_{t, \tau, \alpha}[\ell_{t, \tau, \alpha}]$ are approximately linear with time interval $(t-\tau)$, which will be proven in Appendix \ref{KL}.

The above derivations show that the asymptotic distribution of $\{\ell_{t+\epsilon, \tau+\epsilon, \alpha}-\ell_{t,\tau, \alpha} \}$, for small $|\epsilon|$ is a two-sided Brownian motion with a negative drift $-\xi$. The variance of an increment of this Brownian motion is $\eta^2$.
That is, the re-centered process:
\begin{align}
\ell_{t+\epsilon, \tau+\epsilon, \alpha}-\ell_{t, \tau, \alpha} = B(\eta^2 |\epsilon|)-\xi|\epsilon|
\end{align}
with the equality meaning equality in distribution, where $B$ is a two-sided random walk with negative drift. According to Chapter 3 in \cite{yakir2013extremes}, we obtain (\ref{ratio}).
% can
%approximate the local rate
%\begin{equation*}
%\mathbb{E}_{t, \tau, \alpha}
%\left[ \frac{\mathcal{M}_{t, \tau, \alpha}}{\mathcal{S}_{t, \tau, \alpha}} \right]\\
%= 
% \mathbb{E}_{t, \tau, \alpha}
%\left[ \frac{\sup_{t'} e^{   \ell_{t', \tau', \alpha} -\ell_{t, \tau, \alpha} } }{\int_{t'}  e^{\ell_{t', \tau', \alpha} -\ell_{t,\tau, \alpha}} dt'} \right] =
%\nu \left( \frac{2 \xi}{\eta^2}\right).
%\end{equation*}

\section{Expectation and variance of log-likelihood ratio under null and alternative distributions} \label{KL}
The calculations $I$, $\sigma^2$, $\xi$ and $\eta^2$ boil down to evaluating $\mathbb{E}_{t, \tau, \alpha} [\ell_{t, \tau, \alpha}]$, $\mbox{Var}_{t, \tau, \alpha} [\ell_{t, \tau, \alpha}]$, $\mathbb{E}[\ell_{t, \tau, \alpha}]$ and $\mbox{Var}[\ell_{t, \tau, \alpha}]$, i.e., the expectation and variance of log-likelihood ratio under null and alternative distributions. Below, we will perform the calculation for all likelihoods considered in our paper. The main techniques used are mean-field approximation, Delta method, and Lemma \ref{firstorder} and \ref{secondorder}. Below, let $\mathbb{E}_{\mathcal{H}_{t-}}[\cdot]$ denote the conditional expectation for the Hawkes process given the past history. 

\subsubsection{One-dimension: Poisson to Hawkes.}

%The hypothesis test problem is described as (\ref{test_poi_haw_1d}), and the log-likelihood ratio is constructed as (\ref{likelihoodratio}).

Assuming stationary and $(t-\tau)$ is large, we can approximate the stationary intensity for the Hawkes process to be $\bar{\lambda}^*$, which is defined as \[\bar{\lambda}^*=\lim_{t\to \infty} m_t^*= \lim_{t\to \infty}  \mathbb{E}_{\mathcal{H}_{t-} } [\lambda_t^*].\] We use mean field approximation, which assumes each stochastic process $\lambda_t^*$ has small fluctuations around its mean $\bar{ \lambda}^*$: $|\lambda_t^*-\bar{ \lambda}^* |/\bar{ \lambda}^* \ll 1$. Then we compute the expectation of log-likelihood ratio under alternative hypothesis 
\begin{align}
&\mathbb{E}_{t, \tau, \alpha}[\ell_{t, \tau, \alpha}] \nonumber \\
& = \mathbb{E}_{t,\tau, \alpha} \left[ \int_{\tau}^t \mbox{log} \left( \lambda^*_s \right) dN_s -\int_{\tau}^t \mbox{log} \left( \lambda_s \right) dN_s   \right.  \nonumber  \\
&\left. ~~ - \int_{\tau}^t \left( \lambda^*_s-\lambda_s\right) ds  \right]  \label{e1}\\
& \approx \mathbb{E}_{\mathcal{H}_{t-}}\left[ \int_{\tau}^t \lambda^*_s \mbox{log} \left( \lambda^*_s\right) ds \right.\nonumber \\
&~~ \left.-\int_{\tau}^t \lambda^*_s  \mbox{log} \left( \lambda_s\right) ds      \right]    - \int_{\tau}^t \left( m^*_s-\lambda_s \right) ds  \label{e2} .
\end{align}
From (\ref{e1}) to (\ref{e2}), we use the fact that under $\mathbb{P}_{t, \tau, \alpha}$, $Ns$ is a Hawkes random field with conditional intensity $\lambda^*_s$. From (\ref{e1}) to (\ref{e2}), more justifications can be found in \cite{daley2004scoring,vere1998probabilities,daley2007introduction}. %In their work, they derived the average information gain per unit of time where the log-likelihood ratio is between a point process and a Poisson process.

Next, when $(t-\tau)$ is large, we can approximate the stationary intensity for Hawkes process to be $\bar{\lambda}^*$. To approximate $\mathbb{E}_{\mathcal{H}_{t-}}\left[ \int_{\tau}^t \lambda^*_s \mbox{log} \left( \lambda^*_s\right) ds\right]$, we perform the first order taylor expansion for a new defined function $ f(\lambda^*_s)=\lambda^*_s \mbox{log} \left( \lambda^*_s \right)$ around $\mathbb{E}_{\mathcal{H}_{t-}}[ \lambda^*_s]=\bar{ \lambda}^* $ (this is based on the Delta method):
\begin{align}
\lambda^*_s\mbox{log} \left( \lambda^*_s \right) \approx \bar{\lambda}^* \mbox{log} \left(\bar{  \lambda} ^* \right) +  \left[   \mbox{log}(\bar{ \lambda} ^*)+1\right] (\lambda^*_s-\bar{ \lambda} ^*).
\end{align}
Taking expectation on both sides of the equation and using $\mathbb{E}_{\mathcal{H}_{t-}}[ \lambda^*_s]=\bar{ \lambda}^*$, we have
\[
\mathbb{E}_{\mathcal{H}_{t-}}\left[   \int_{\tau}^t \lambda^*(s ) \mbox{log} \left( \lambda^*(s) \right) ds    \right] 
\approx  \,  \int_{\tau}^t \bar{\lambda}^* \mbox{log} (\bar{  \lambda} ^*) ds.
\]
Finally, we have:
\begin{equation*}
\begin{split}
&\mathbb{E}_{t, \tau, \alpha}[\ell_{t, \tau, \alpha}] \approx (t-\tau) \left[  \bar{\lambda}^*\mbox{log} \left( \frac{\bar{ \lambda}^*}{\mu}\right) -(\bar{ \lambda} ^*-\mu)\right]\\
&=(t-\tau) \underbrace{ \left[   \frac{\mu}{1-\alpha} \mbox{log} \left( \frac{1}{1-\alpha}\right) -\frac{\alpha}{1-\alpha}\mu\right] }_{I}.
\end{split}
\end{equation*}
%where $I$ is the Kullback-Leibler information per time under the stationary assumption.

On the other hand, under the null distribution and given stationary assumption, we have:
\begin{align*}
&\mathbb{E}[\ell_{t,\tau, \alpha}]\\ & = \mathbb{E}\left[     \int_{\tau}^t \mbox{log} \left( \frac{\lambda^*_s}{\lambda_s} \right) dN_s - \int_{\tau}^t \left( \lambda^*_s-\lambda_s \right) ds  \right] \\
& \approx  \mathbb{E}_{\mathcal{H}_{t-}} \left[  \int_{\tau}^t \lambda_s\mbox{log} \left( \frac{\lambda^*_s}{\lambda_s} \right) ds - \int_{\tau}^t \left( \lambda^*_s-\lambda_s \right) ds  \right] \\
& \approx  (t-\tau) \underbrace{  \left[   \mu \mbox{log} \left( \frac{1}{1-\alpha}\right) -\frac{\alpha}{1-\alpha}\mu\right] }_{I_0}.
\end{align*}
For the second equality we use the fact that under $\mathbb{P}$, $Ns$ is a Poisson random field with intensity $\lambda_s$. For the last equality, we use mean-field approximation.

Next, we compute the variance of log-likelihood ratio under null distribution and alternative distribution, respectively. Under the alternative distribution,
\begin{align*}
 &\int_{\tau}^t \mbox{log} \left( \frac{\lambda^*_s}{\lambda_s} \right) dN_s - \int_{\tau}^t \left( \lambda^*_s-\lambda_s \right) ds  \\
 &\approx \,   \int_{\tau}^t \left[  \lambda^*_s \mbox{log} \left( \frac{\lambda^*_s}{\lambda_s} \right) - \lambda^*_s \right]  ds + \lambda_s (t-\tau).
\end{align*}
Then the only random part is $\int_{\tau}^t \left[  \lambda^*_s \mbox{log} \left( \frac{\lambda^*_s}{\lambda_s} \right) - \lambda^*_s \right]  ds  $. Therefore, 
\begin{align}
\mbox{Var}_{t, \tau, \alpha}[\ell_{t, \tau, \alpha}] \approx \mbox{Var}_{\mathcal{H}_{t-}} \left[    \int_{\tau}^t \left[  \lambda^*_s\mbox{log} \left( \frac{\lambda^*_s}{\lambda_s} \right) - \lambda^*_s\right]  ds   \right].
\end{align}
Again, to use Delta method, we consider a function with respect to $\lambda^*_s$:
\[
f(\lambda^*_s) =  \lambda^*_s \mbox{log} \left( \frac{\lambda^*_s}{\lambda_s} \right) - \lambda^*_s,
\]
and apply the first order taylor expansion around $\mathbb{E}_{\mathcal{H}_{t-}}[\lambda^*_s]=\bar{ \lambda} ^*$:
\begin{equation} \label{e3}
f(\lambda^*_s) \approx f(\bar{ \lambda} ^*) + \mbox{log} \left( \frac{\bar{ \lambda}^*}{\lambda_s} \right) \left( \lambda^*_s-\bar{ \lambda}^* \right) .
\end{equation}
From (\ref{e3}), we obtain
\begin{align*}
&\mbox{Var}_{\mathcal{H}_{t-}} \left[   f(\lambda^*_s)   \right]  \approx \mathbb{E}_{\mathcal{H}_{t-}} \left[   \left(   f(\lambda^*_s)    -   f(\lambda^*) \right)^2    \right] \\
&\approx \left[ \mbox{log} \left(  \frac{\bar{ \lambda}^*}{\lambda_s} \right) \right]^2 \mathbb{E}_{\mathcal{H}_{t}} \left[  (\lambda^*_s -\bar{\lambda}^*)^2  \right],
\end{align*}
where $\mathbb{E}_{\mathcal{H}_{t-}} \left[  (\lambda^*_s  -\bar{\lambda}^*)^2  \right] = \mbox{Var}_{\mathcal{H}_{t}} [\lambda^*_s]$.
Note that the log-likelihood ratio is  an integration from $\tau$ to $t$. When computing the variance, we need to consider $\mbox{Cov}[ \lambda^*_s, \lambda^*_{s+\tau}  ]$. Under the stationary assumption, from Lemma \ref{secondorder}, we obtain an expression for $c(\tau):=\rm{Cov}_{\mathcal{H}_{t-}} [\lambda^*_s, \lambda^*_{s+\tau}] $, which only depends on $\tau$.  
Therefore,
\begin{align*}
& \mbox{Var}_{t, \tau, \alpha}[\ell_{t, \tau, \alpha}] \\
& \approx  \left[ \mbox{log} \left( \frac{\bar{\lambda}^*}{\lambda_s} \right) \right]^2 
\int_{\tau}^t \int_{\tau}^t c(s'-s)dsds'\\
& =\left[ \mbox{log} \left(\frac{\bar{\lambda}^*}{\lambda_s} \right) \right]^2 \left[    \int_{0}^{t-\tau} \lambda^*  ds + 2 \int_{0}^{t-\tau} \int_{0}^{s}c(v)dvds  \right] \\
& =  (t-\tau)\left[ \mbox{log} \left( \frac{1}{1-\alpha} \right) \right]^2 
\left[ \frac{\mu}{1-\alpha}+ \frac{\alpha (2-\alpha)\mu}{(1-\alpha)^3} \right.\\
&\left.~~ +
  \frac{\alpha (2-\alpha)\mu e^{-\beta(1-\alpha) (t-\tau)} }{\beta (1-\alpha)^4(t-\tau)}  -     \frac{\alpha (2-\alpha)\mu }{\beta (1-\alpha)^4(t-\tau)}     \right].
\end{align*}
Moreover, since $\alpha$ is usually a small number, when $(t-\tau)$ is a large number, we may ignore the small terms and further approximate:
\begin{equation*}
\begin{split}
&\mbox{Var}_{t, \tau, \alpha}[\ell_{t, \tau, \alpha}] \\
&\approx (t-\tau)  \underbrace{  \left[ \mbox{log} \left( \frac{1}{1-\alpha} \right) \right]^2 
\left[ \frac{\mu}{1-\alpha}+ \frac{\alpha (2-\alpha)\mu}{(1-\alpha)^3} \right] }_{\sigma^2}.
\end{split}
\end{equation*}
On the other hand, under the null distribution, we have the variance of the log-likelihood ratio
\begin{equation*}
\begin{split}
\mbox{Var}[\ell_{t, \tau, \alpha}] &\approx \left[ \mbox{log} \left( \frac{\bar{ \lambda}^*}{\lambda_s} \right) \right]^2  \int_{\tau}^t  \lambda_s ds \\
&= (t-\tau) \underbrace{  \mu \left[ \mbox{log}\left(  \frac{1}{1-\alpha} \right) \right]^2}_{\sigma^2_0}.
\end{split}
\end{equation*}

\vspace{-0.6in}

\subsubsection{Multi-dimension: Poisson to Hawkes}

The derivations for the multi-dimensional case would follow the same strategy as the one-dimensional case. So we just put the key results here. For the expectation of the log-likelihood ratio under alternative distribution, we have:
\begin{align*}
&\mathbb{E}_{t, \tau, \bm{A} }[ \ell_{t, \tau, \alpha }] \\
& \approx (t-\tau)  \left[\bar{ \bm{ \lambda}} ^{*\intercal} \left( \mbox{log} (\bar{\bm{\lambda}} ^* )-   \mbox{log} (\bm{\mu}) \right) -\bm{e}^{\intercal}(\bar{\bm{\lambda}} ^*-\bm{\mu}) \right]   \\
& = (t-\tau) \left[ (\bm{I}-\bm{A})^{-1}\bm{\mu} \left( \mbox{log} ((\bm{I}-\bm{A})^{-1}\bm{\mu} )-   \mbox{log} (\bm{\mu}) \right)\right.\\ 
&~~\left.-\bm{e}^{\intercal}((\bm{I}-\bm{A})^{-1}\bm{\mu}-\bm{\mu}) \right] .
\end{align*}
where the quantity inside $[\cdot]$ above corresponds to $I$ in this case.
Under null, we have
\begin{align*}
\mathbb{E}[ \ell_{t, \tau, \alpha }] & \approx 
(t-\tau) \left[  \bm{ \mu}^{\intercal} \left( \mbox{log} (\bm{\lambda}^*)-   \mbox{log} (\bm{\mu}) \right) -\bm{e}^{\intercal}(\bm{\lambda}^*-\bm{\mu}) \right] \\
& = (t-\tau) \left[   \bm{ \mu}^{\intercal} \left( \mbox{log} ( (\bm{I-A})^{-1}\bm{\mu})-   \mbox{log} (\bm{\mu}) \right)\right. \\
& ~~\left.-\bm{e}^{\intercal}(( \bm{I-A})^{-1}-\bm{I} )\bm{\mu} \right],
\end{align*}
where the quantity inside $[\cdot]$ above corresponds to $I_0$ in this case.
For the variance of the log-likelihood ratio under alternative, we have
\begin{align}
& \mbox{Var}_{t, \tau, \bm{A}}[ \ell_{t, \tau, \bm{A} }] \nonumber\\
 =&~\mbox{Var}_{t, \tau, \bm{A}}  \left[\sum_{i=1}^d  \int_{\tau}^t \mbox{log} \left(  \frac{\lambda_i (s)}{\mu_i} \right) dN_s^i \right] \nonumber \\
=& ~\sum_{i=1}^d   \mbox{Var}_{t, \tau, \bm{A}}  \left[ \int_{\tau}^t \mbox{log} \left(  \frac{\lambda_i (s)}{\mu_i} \right) dN_s^i \right] \nonumber \\
&\hspace{-0.3in}+ 2 \sum_{i<j} \mbox{Cov}_{t, \tau, \bm{A}} \left[  \int_{\tau}^t \mbox{log} \left(  \frac{\lambda_i (s)}{\mu_i} \right) dN_s^i, \int_{\tau}^t \mbox{log} \left(  \frac{\lambda_j (s)}{\mu_j} \right) dN_s^j      \right]  .\label{var} 
\end{align}
From Lemma \ref{secondorder}, for $s>0$
\begin{equation*}
\begin{split}
\bm{c}(s)
= &~\beta e^{-\beta (\bm{I}- \bm{A}) s} \bm{A} \left( \bm{I} +\frac{1}{2} (\bm{I}-\bm{A})^{-1}
\bm{A}  \right) \\
&~~\cdot \mbox{diag} \left(  (\bm{I}-\bm{A})^{-1} \bm{\mu}\right).
\end{split}
\end{equation*}
To compute (\ref{var}), we also need 
\begin{align*}
& \int_{\tau}^t \int_{\tau}^t \bm{c}(s'-s) dsds' 
=   2 \int_0^{t-\tau} \int_0^s \bm{c}(v) dv ds \\
= & \, 2 \beta \int_0^{t-\tau} \int_0^s e^{-\beta (\bm{I}- \bm{A}) v}   dv ds  \\
&~~\bm{A} \left( \bm{I} +\frac{1}{2} (\bm{I}-\bm{A})^{-1}
\bm{A}  \right) \mbox{diag} \left(  (\bm{I}-\bm{A})^{-1} \bm{\mu}\right) \\
= &\, 2 \beta \int_0^{t-\tau} \left( -\frac{1}{\beta} (\bm{I}-\bm{A})^{-1}  \left( e^{-\beta (\bm{I}-\bm{A}) s} -\bm{I}  \right) \right) ds  \\
&~~\bm{A} \left( \bm{I} +\frac{1}{2} (\bm{I}-\bm{A})^{-1}
\bm{A}  \right) \mbox{diag} \left(  (\bm{I}-\bm{A})^{-1} \bm{\mu}\right) \\
= &\, 2  (\bm{I}-\bm{A})^{-1} \int_0^{t-\tau}   \left( \bm{I} - e^{-\beta (\bm{I}-\bm{A}) s}  \right) ds \\
&~~ \bm{A} \left( \bm{I} +\frac{1}{2} (\bm{I}-\bm{A})^{-1}
\bm{A}  \right) \mbox{diag} \left(  (\bm{I}-\bm{A})^{-1} \bm{\mu}\right) \\
\approx  & \, (t-\tau)   (\bm{I}-\bm{A})^{-1}  \bm{A} \left( 2 \bm{I} +(\bm{I}-\bm{A})^{-1}
\bm{A}  \right)\\
&~~\cdot \mbox{diag} \left(  (\bm{I}-\bm{A})^{-1} \bm{\mu}\right) .
\end{align*}
Note that when computing $\mbox{Cov} [dN_s^i, dN_{s'}^i]$, we need to consider an extra term:
\begin{align}
 \int_{\tau}^t \int_{\tau}^t \bar{ \bm{\lambda}} ^* \delta(s'-s) dsds' = \int_0^{t-\tau} \bar{ \bm{\lambda} }^* ds = (t-\tau)  \bar{ \bm{\lambda} }^*.
\end{align}
After rearranging  terms, using the mean-field approximation and Delta method, we  obtain 
\begin{align}
 \mbox{Var}_{t, \tau, \bm{A}}[ \ell_{t, \tau, \bm{A} }] 
\approx  (t-\tau)  \underbrace{  \bm{e}^{\intercal} \left( \bm{H} \circ \bm{C} \right) \bm{e} }_{\sigma^2},
\end{align}
where $\bm{H}$ and $\bm{C}$ are defined in Table \ref{table1}.

We compute the variance of the log-likelihood under null distribution. Note that when the data follow Poisson processes, we have $\mbox{Cov} [N^i_t, N^j_{t'}]_{ t \neq t'}=0$. Therefore, 
\begin{align}
\mbox{Var}[ \ell_{t, \tau, \bm{A} }] &  \approx (t-\tau) \left[  \bm{ \mu}^{\intercal} \left( \mbox{log} (\bar{ \bm{\lambda}} ^*)-   \mbox{log} (\bm{\mu} )\right)^{(2)}  \right]  \\
& \approx (t-\tau) \underbrace{ \left[  \bm{ \mu}^{\intercal} \left( \mbox{log} ((\bm{I-A})^{-1}\bm{\mu})-   \mbox{log} (\bm{\mu} )\right)^{(2)}  \right]  }_{\sigma^2_0}. 
\end{align}

\subsubsection{One-dimension: Hawkes to Hawkes.}

Similarly, we compute the expectation of the log-likelihood ratio under alternative distribution
\begin{align*}
& \mathbb{E}_{t, \tau, \alpha}[\ell_{t, \tau, \alpha}] \\
& = \mathbb{E}_{t,\tau, \alpha} \left[ \int_{\tau}^t \mbox{log} \left( \lambda^*_s\right) dN_s -\int_{\tau}^t \mbox{log} \left( \lambda_s \right) dN_s        - \int_{\tau}^t \left( \lambda^*_s-\lambda_s\right) ds  \right]  \\
& \approx \mathbb{E}_{\mathcal{H}_{t-}} \left[ \int_{\tau}^t \lambda^*_s \mbox{log} \left( \lambda^*_s \right) ds -\int_{\tau}^t \lambda^*_s\mbox{log} \left( \lambda_s \right) ds        - \int_{\tau}^t \left( \lambda^*_s-\lambda_s\right) ds  \right]  \\
& \approx (t-\tau) \left[ \bar{ \lambda} ^* \mbox{log}(\bar{\lambda}^*) -\bar{\lambda}^* \mbox{log} (\bar{\lambda}) -(\bar{\lambda}^*-\bar{\lambda})                    \right] \\
& \approx (t-\tau) \underbrace{  \left[  \frac{\mu}{1-\alpha^*} \mbox{log} \left( \frac{1-\alpha}{1-\alpha^*} \right)  -\frac{\mu}{1-\alpha^*} +\frac{\mu}{1-\alpha}     \right] }_{I},
\end{align*}
where the first approximation is due to that under $\mathbb{P}_{t, \tau, \alpha}$, $N(ds)$ is a Hawkes random field with intensity $\lambda^*_s$, and for the latter approximation, we are using mean field approximation and (multivariate) Delta Method given $\mathbb{E}_{\mathcal{H}_{t-}}[\lambda^*(s)]=\bar{\lambda}^*$ and $\mathbb{E}_{\mathcal{H}_{t-}}[\lambda_s]=\bar{\lambda}$. And for the stationary intensity, we have $\bar{ \lambda} = \mu/(1-\alpha)$ and $\bar{ \lambda} ^*= \mu/(1-\alpha^*)$. 

Next, the expectation of the log-likelihood ratio under null distribution is given by
\begin{align*}
&\mathbb{E} [\ell_{t, \tau, \alpha}] \\
& = \mathbb{E} \left[ \int_{\tau}^t \mbox{log} \left( \lambda^*_s \right) dN_s -\int_{\tau}^t \mbox{log} \left( \lambda_s\right) dN_s        \right.\\
&\left.~~- \int_{\tau}^t \left( \lambda^*_s-\lambda_s \right) ds  \right]  \\
& \approx \mathbb{E}_{\mathcal{H}_{t-}}\left[ \int_{\tau}^t \lambda_s \mbox{log} \left( \lambda^*_s  \right) ds -\int_{\tau}^t \lambda_s\mbox{log} \left( \lambda_s \right) ds        \right.\\
&~~\left.- \int_{\tau}^t \left( \lambda^*_s-\lambda_s \right) ds  \right]  \\
& \approx (t-\tau) \left[ \bar{ \lambda} \mbox{log}(\bar{\lambda}^*) -\bar{ \lambda} \mbox{log} ( \bar{ \lambda}) -(\bar{ \lambda}^*-\bar{ \lambda} )                    \right] \\
& = (t-\tau) \underbrace{ \left[  \frac{\mu}{1-\alpha} \mbox{log} \left(  \frac{1-\alpha}{1-\alpha^*} \right) -\frac{\mu}{1-\alpha^*} +\frac{\mu}{1-\alpha}       \right]}_{I_0},
\end{align*}
and the variance of the log-likelihood ratio under alternative distribution is given by\begin{align*}
 \ell_{t, \tau, \alpha} 
&=  \int_{\tau}^t \mbox{log} \left( \lambda^*_s \right) dN_s -\int_{\tau}^t \mbox{log} \left( \lambda_s \right) dN_s    \\
&~~    - \int_{\tau}^t \left( \lambda^*_s-\lambda_s \right) ds  \\
&\approx     \int_{\tau}^t  \underbrace{   \left[   \lambda^*_s \mbox{log} \left( \lambda^*_s  \right) - \lambda^*_s  \mbox{log}\left( \lambda_s \right)     -  \lambda^*_s+\lambda_s  \right] }_{f(   \lambda^*_s, \lambda_s )}ds.
\end{align*}
Next, we perform the first order taylor expansion to the newly defined multivariate function with respect to $\lambda^*_s$ and $\lambda_s$:
\begin{align*}
&f (\lambda^*_s, \lambda_s )\\
&\approx f (\bar{\lambda} ^*, \bar{ \lambda}) + \mbox{log} \left( \frac{\bar{\lambda}^*}{\bar{\lambda} } 
    \right)\left(  \lambda^*_s- \bar{\lambda^*} \right)  +\left(  1-\frac{\bar{ \lambda}^*}{\bar{\lambda}}\right) (\lambda_s -\bar{ \lambda} ).
\end{align*}
Based on this, we have
\begin{align*}
&\mbox{Var} \left[  f (\lambda^*_s, \lambda_s )   \right]  = \mathbb{E}\left[   \left(    f (\lambda^*_s, \lambda_s )-   f (\bar{\lambda}^*, \bar{\lambda}  ) \right)^2 \right] \\
& \approx  \left[   \mbox{log} \left( \frac{\bar{\lambda}^*}{\bar{ \lambda} } 
    \right)  \right]^2 \mbox{Var}[\lambda^*_s] +\left(  1-\frac{\bar{\lambda}^*}{\bar{ \lambda} }\right)^2 \mbox{Var}[\lambda_s].
\end{align*}
Note that the null intensity $\lambda_s $ is independent of the alternative intensity $\lambda^*_s$. Finally, we have:
\begin{align*}
& \mbox{Var}_{t, \tau, \alpha} [\ell_{t, \tau, \alpha}] \\
 \approx &~ \mbox{Var}_{t, \tau, \alpha}  \left[        \int_{\tau}^t \lambda^*_s \mbox{log} \left( \lambda^*_s \right) ds -\int_{\tau}^t \lambda^*_s \mbox{log} \left( \lambda_s \right) ds  \right.     \\
 &\left.~~ - \int_{\tau}^t \left( \lambda^*_s-\lambda_s \right) ds         \right] \\
 \approx & \left[   \mbox{log} \left( \frac{\bar{\lambda}^*}{\bar{\lambda}} 
    \right)  \right]^2 \int_{\tau}^t \int_{\tau}^t c ^* (s'-s)dsds' \\
    &~~+ \left(  1-\frac{\bar{\lambda}^*}{\bar{\lambda} }\right)^2 \int_{\tau}^t \int_{\tau}^t c(s'-s)dsds'  \\
  \approx &~  (t-\tau) \left(     \left[ \mbox{log} \left( \frac{1-\alpha}{1-\alpha^*} \right) \right]^2 
\left[ \frac{\mu}{1-\alpha^*}+ \frac{\alpha^* (2-\alpha^*)\mu}{(1-\alpha^*)^3}    \right] \right.\\
&~~\left.+\left(  1-\frac{1-\alpha}{1-\alpha^*}\right)^2 \left[ \frac{\mu}{1-\alpha}+ \frac{\alpha(2-\alpha)\mu}{(1-\alpha)^3}    \right]   \right).
\end{align*}
The factor in the last equation that multiplies $(t-\tau)$ corresponds to $\sigma^2$ in this setting. 
Again, we've ignored some small terms.

Similarly, we can compute the variance of the log-likelihood ratio under null distribution. Under null distribution, 
\begin{align}
  \ell_{t, \tau, \alpha} \approx \int_{\tau}^t  \underbrace{ \lambda_s  \mbox{log} ( \lambda^*_s ) - \lambda_s  \mbox{log} ( \lambda_s) -\lambda^*_s+\lambda_s  }_{f(\lambda^*_s, \lambda_s)}ds.
\end{align}
Still perform the first order taylor expansion to the new defined function:
\[
f( \lambda^*_s, \lambda_s ) \approx  f(\bar{\lambda} ^*, \bar{\lambda} )+\left( \frac{\bar{ \lambda} }{\bar{ \lambda} ^*}-1 \right)
(\lambda^*_s-\bar{\lambda}^*) +\mbox{log} \left( \frac{\bar{\lambda}^*}{\bar{\lambda} } \right) (\lambda_s-\bar{\lambda} ).
\]
Therefore, using multivariate Delta method
\begin{align*}
&\mbox{Var}[  f( \lambda^*_s, \lambda_s )   ] \\
& =\mathbb{E} \left[   \left(    f( \lambda^*_s, \lambda_s ) -   f( \bar{\lambda}^*, \bar{\lambda}  )\right)^2 \right] \\
& \approx \left( \frac{\bar{ \lambda} }{\bar{\lambda}^*}-1 \right)^2 \mbox{Var}[\lambda^*_s] + \left[ \mbox{log} \left( \frac{\bar{\lambda}^*}{\bar{\lambda} } \right)    \right]^2 \mbox{Var} [\lambda _s].
\end{align*}
Finally we obtain
\begin{align*}
& \mbox{Var} [\ell_{t, \tau, \alpha}] 
 \approx  \left( \frac{\bar{\lambda} }{\bar{\lambda}^*}-1 \right)^2  \int_{\tau}^t \int_{\tau}^t c^{*}(s'-s)dsds' \\
 &~~+ \left[ \mbox{log} \left( \frac{\bar{\lambda}^*}{\bar{\lambda}} \right)    \right]^2 \int_{\tau}^t \int_{\tau}^t c(s'-s)dsds'  \\
  \approx &  (t-\tau)   \left(   \left[  1-\frac{1-\alpha^*}{1-\alpha}\right]^2  
\left[ \frac{\mu}{1-\alpha^*}+ \frac{\alpha^* (2-\alpha^*)\mu}{(1-\alpha^*)^3}    \right] \right.
\\&~~\left.+\left[ \mbox{log} \left( \frac{1-\alpha}{1-\alpha^*} \right) \right]^2  \left[ \frac{\mu}{1-\alpha}+ \frac{\alpha(2-\alpha)\mu}{(1-\alpha)^3}    \right]   \right) .
\end{align*}
The factor in the last equation that multiplies $(t-\tau)$ corresponds to $\sigma_0^2$ in this setting. 

The proof for multi-dimensional case with a transition from the Hawkes process to a Hawkes process is similar and omitted here. 

\section{More real-data examples}\label{app:more_eg}

The scenario for Fig.~\ref{fig:real-twitter}(d) is also interesting as it reflects the activity on the network surrounding Mr. Shkreli, the former chief executive of Turing Pharmaceuticals, who is facing federal securities fraud charges. At Feb. 4th he was invited to congress for a hearing to be questioned about drug price hikes\footnote{\url{http://www.nytimes.com/2016/02/05/business/drug-prices-valeant-martin-shkreli-congress.html}}. 

The fifth example, Fig.~\ref{fig:real-twitter}(e) is about Rihanna who announced the release of her new album in a tweet on Jan. 25th. That post was retweeted 170K times and received 280K likes and creates a sudden change in network of her followers.\footnote{\url{http://jawbreaker.nyc/2016/01/is-rihannas-anti-album-finally-done/}} 

The last example, in Fig.~\ref{fig:real-twitter}(f), demonstrates an increase in the statistic related to the network of Daughter around 25th of January who is attributed to releasing his new album at Jan. 25th.\footnote{\url{http://www.nme.com/news/daughter/79540}}

\begin{figure}[h!]
\begin{center}
\begin{tabular}{c}
\includegraphics[width=0.32\textwidth]{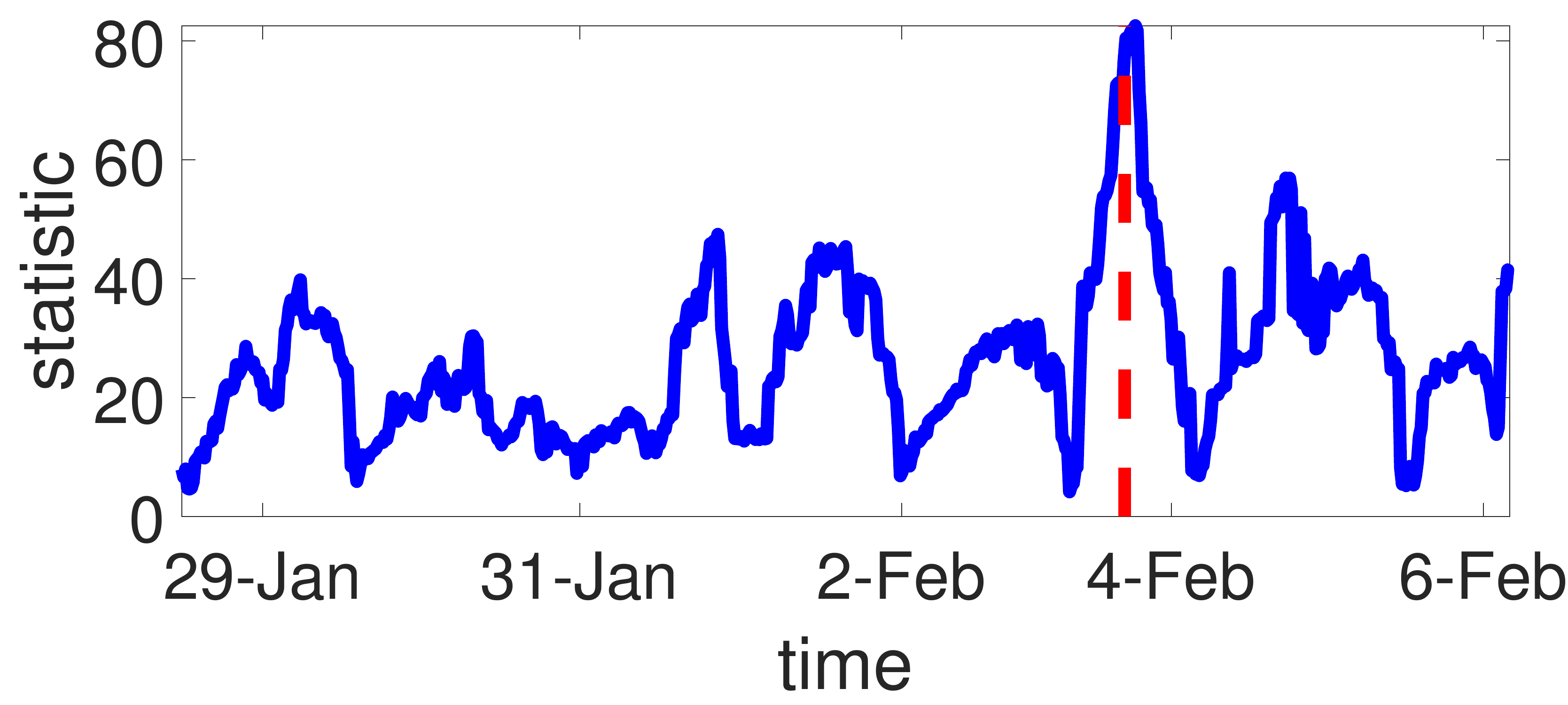} \\(a)\\
\includegraphics[width=0.32\textwidth]{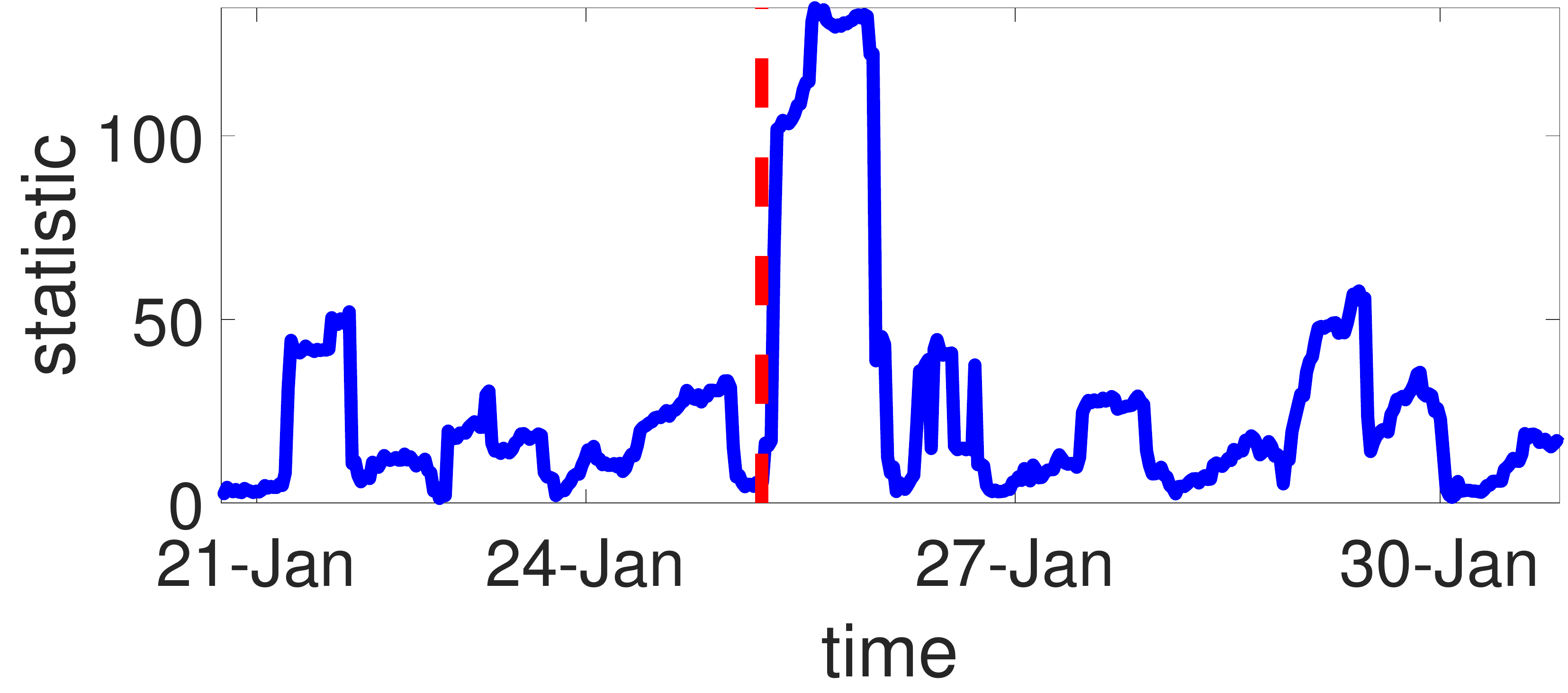} \\(b)\\
\includegraphics[width=0.32\textwidth]{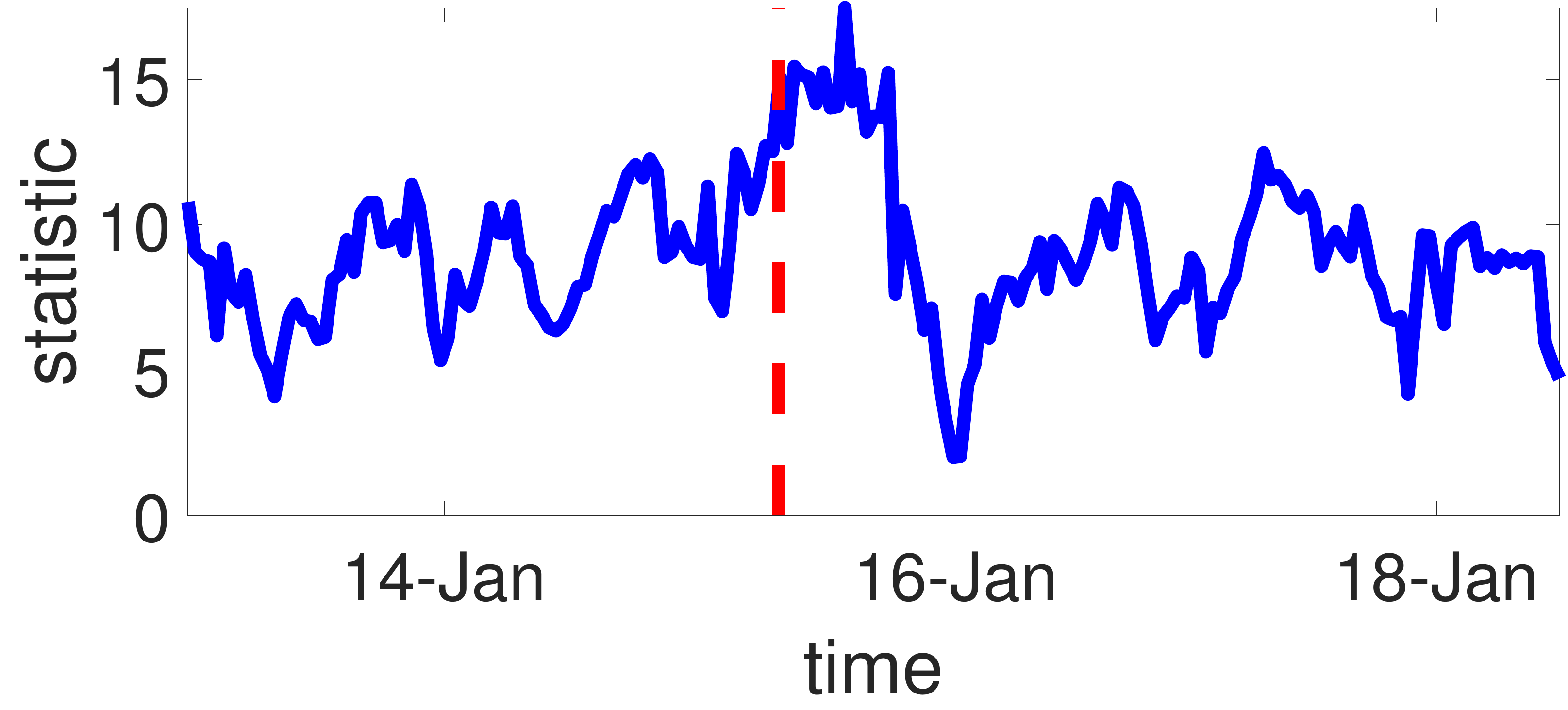} \\(c)
\end{tabular}
\caption{\small Exploratory results on Twitter for the detected change points: (a) Court hearing  on Martin Shkreli; (b) Rihanna listens to ANTI; (c) Daughter releases his new album.}
\end{center}
\end{figure}

\end{document}